\documentclass{article}

\usepackage{cite}
\usepackage{amsmath,amssymb,amsfonts}
\usepackage{amsthm,thmtools}
\usepackage{algorithmic}
\usepackage{graphicx}
\usepackage{textcomp}
\def\BibTeX{{\rm B\kern-.05em{\sc i\kern-.025em b}\kern-.08em
    T\kern-.1667em\lower.7ex\hbox{E}\kern-.125emX}}

\usepackage{cleveref}
\usepackage{nicefrac}       
\usepackage{microtype}      
\usepackage{xcolor}  
\usepackage{mathtools}
 \usepackage{algorithm}
\usepackage{caption}
\usepackage{enumitem}
\newtheorem{prop}{Proposition}
\newtheorem{theorem}{Theorem}
\newtheorem{lemma}{Lemma}
\newtheorem{coro}{Corollary}
\newtheorem{rmk}{Remark}
\newtheorem{defi}{Definition}

\newtheorem{assump}{Assumption}

\DeclareMathOperator*{\argmax}{arg\,max}

 \newcommand{\bR}{\mathbb{R}}
 
   \newcommand{\cD}{\mathcal{D}}

 \newcommand{\bE}{\mathbb{{E}}}
 \newcommand{\cA}{\mathcal{A}}
  \newcommand{\cF}{\mathcal{F}}
  
 \newcommand{\cS}{\mathcal{S}}
 \newcommand{\cX}{\mathcal{X}}
 \newcommand{\cI}{\mathcal{I}}

  \newcommand{\Prtheta}{\textup{Pr}^{\theta}}
   \newcommand{\Dalpha}{\Delta^\alpha(|\cA_i|)}
    \newcommand{\NEgap}{\textup{\texttt{NE-gap}}}
\newcommand{\stay}{\textup{Stay}}
\newcommand{\lleft}{\textup{Left}}
\newcommand{\rright}{\textup{Right}}
\newcommand{\up}{\textup{Up}}
\newcommand{\down}{\textup{Down}}

\newenvironment{talign*}
 {\csname align*\endcsname}
 {\endalign}
\allowdisplaybreaks
\title{Gradient play in stochastic games: stationary points, convergence, and sample complexity}
\author{Runyu (Cathy) Zhang, Zhaolin Ren, Na Li \thanks{R. Zhang, Z. Ren, and N. Li are affiliated with Harvard School of Engineering and Applied Sciences, (e-mail: runyuzhang@fas.harvard.edu, zhaolinren@g.harvard.edu, nali@seas.harvard.edu) } \thanks{This research is funded by NSF CAREER ECCS-1553407, NSF AI institute 2112085, NSF CNS: 2003111, ONR YIP N00014-19-1-2217.}}
\date{January 2022}

\begin{document}

\maketitle
\begin{abstract}
We study the performance of the gradient play algorithm for stochastic games (SGs), where each agent tries to maximize its own total discounted reward by making decisions \textit{independently} based on current state information which is shared between agents. Policies are directly parameterized by the probability of choosing a certain action at a given state. We show that Nash equilibria (NEs) and first-order stationary policies are equivalent in this setting, and give a local convergence rate around strict NEs. Further, for a subclass of SGs called Markov potential games (which includes the setting with identical rewards as an important special case), we design a sample-based reinforcement learning algorithm and give a non-asymptotic global convergence rate analysis for both exact gradient play and our sample-based learning algorithm. Our result shows that the number of iterations to reach an $\epsilon$-NE scales linearly, instead of exponentially, with the number of agents. Local geometry and local stability are also considered, where we prove that strict NEs are local maxima of the total potential function and fully-mixed NEs are saddle points.
\end{abstract}

\section{Introduction}
 Multi-agent systems find applications in a wide range of societal systems, e.g. electric grids, traffic networks, smart buildings and smart cities etc. Given the complexity of these systems, multi-agent reinforcement learning (MARL) has gained increasing attention in recent years (e.g. \cite{daneshfar2010, ShalevShwartz2016}) 
 Among MARL algorithms, policy gradient-type methods are highly popular because of their flexibility and capability to incorporate structured state and action spaces. However, while many recent works \cite{Zhang18,Hoi-To18, chen18,li19,qu2020} have studied the performance of multi-agent policy gradient algorithms, due to a lack of understanding of the optimization landscape in these multi-agent learning problems, most works can only show convergence to a first-order stationary point. Deeper understanding of the quality of these stationary points is missing even in the simple identical-reward multi-agent RL setting.


In this paper, we investigate this problem from a game-theoretic perspective. We model the multi-agent system as a stochastic game (SG) where agents take \textit{independent} stochastic policies and can have different reward functions. The study of SGs dates back to as early as the 1950s by \cite{shapley53} with a series of follow-up works on developing NE-seeking algorithms, especially in the RL setting (e.g. \cite{Littman94,Bowling00,Shoham03,Bucsoniu10,Lanctot17, Zhang19} and citations therein). While well-known classical algorithms for solving SGs are mostly value-based, such as Nash-Q learning \cite{Hu03}, Hyper-Q learning \cite{Tesauro03}, and WoLF-PHC \cite{Bowling01}, gradient-based algorithms have also started to gain popularity in recent years due to their advantages as mentioned earlier (e.g. \cite{Abdallah08,Zhang10,Foerster17}). 

In this work, we aim to gain a deeper understanding of the structure of first-order stationary points and the dynamical behavior for these gradient-based methods, with a particular focus on answering the following questions: 1) How do the first-order stationary points relate to the NEs of the underlying game?, 2) What is the stability of the individual NEs?, 3) How do agents learn from samples in this environment?

These questions have already been widely discussed in other settings, e.g.,  {one-shot (stateless) finite-action games \cite{Shapley64,Crawford85,Jordan93,Krishna98,Shamma05,Kohlberg86,Van91, foster2006regret, germano2007global, marden2009payoff}}, one-shot continuous games \cite{mazumdar20}, zero-sum linear quadratic (LQ) games \cite{Zhang19-LQgame}, etc. There are both negative and positive results depending on the settings. For one-shot continuous games, \cite{mazumdar20} proved a negative result suggesting that gradient flow has stationary points (even local maxima) that are not necessarily NEs. Conversely, \cite{Zhang19-LQgame} designed projected nested-gradient methods that provably converge to NEs in zero-sum LQ games. However, much less is known in the tabular setting of SGs with finite state-action spaces.

\noindent\textbf{Contributions.} We consider the \emph{gradient play} algorithm for the infinite time horizon, discounted reward SGs with independent, directly-parameterized agents' policies.  Through generalizing the gradient domination property in \cite{agarwal2020} to the SG setting, we
first establish the equivalence of first-order stationary policies and Nash equilibria (Theorem \ref{thm:equivalence-stationary-NE}). This result suggests that even if agents have an identical reward,  the first-order stationary points are only equivalent to Nash equilibria, which are usually non-unique and have different reward values. This is fundamentally different from the centralized learning case \cite{agarwal2020} where it can be shown that the first order stationary point is the unique global optimal solution.

Then we study the convergence of gradient play for SGs. For general games, it is known that gradient play may fail to obtain global convergence \cite{Shapley64,Crawford85,Jordan93,Krishna98}. Thus we firstly focus on characterizing some local properties for the general cases. In particular, we characterize the structure of strict NEs and show that gradient play locally converges to strict NEs within finite steps (Theorem~\ref{theorem:local-convergence-general-game}).

Next we study a special class of SGs called Markov potential games (MPGs) \cite{Gonzalez13,Macua18,leonardos21}, {which includes identical reward multi-agent RL \cite{Tan93,Claus98,Panait05,matignon2012independent, zhang2019collaborative,oroojlooy2023review} as an important special case.} Concurrently, this work and \cite{leonardos21} have established the global convergence rate to a NE for gradient play under MPGs (Theorem \ref{theorem:potential-game-convergence}). However, the result does not specify which NE the policies converge to. 
Given the fact that there are many NEs that would have poor global value, global convergence results have a limited implication on the algorithm performance. This motivates us to study the local geometry around some specific types of NEs. We show that strict NEs are local maxima of the total potential function, thus stable points under gradient play, and that fully mixed NEs are saddle points, thus unstable points under gradient play (Theorem \ref{theorem:MPG-local-maximum-saddle-points}). 

Then, we design a fully-decentralized sample-based gradient play algorithm and prove that it can find an $\epsilon$-Nash equilibrium with high probability using $\widetilde{O}\left(\frac{n}{\epsilon^6} \textup{poly}\left(\frac{1}{1-\gamma}, |\cS|, \max_i|\cA_i|\right)\right)$ samples (Theorem \ref{thm:main}, $|\cS|,|\cA_i|$ denote the size of the state space and action space of agent $i$ respectively). The key enabler of our algorithm is the existence of an underlying \textit{averaged MDP} for each agent when other agents' policies are fixed. Our learning method can be viewed as a \textit{model-based} \textit{policy evaluation} method with respect to agents' averaged MDPs. This averaged MDP concept could be applied to design many other MARL algorithms, especially policy-evaluation-based methods.

\noindent\textbf{Comparison to other works on NE learning for SGs:} 
There are some recent studies on general SGs with finite state-action spaces. However, either the structure of SGs or the methods they consider are different from our setting. For example, \cite{song2021,jin2021v} consider learning correlated equilibria (CCE) rather than NEs for finite time horizon general-sum games; 
{\cite{Arslan16} and \cite{Yongacoglu22} propose decentralized learning algorithms for the weakly acyclic games, which include the identical interest game as a special case, but only asymptotic convergence is considered}; \cite{daskalakis2021, ozdaglar2021independent} considers convergence to NE for two-player zero-sum games. In addition, \cite{Zhang18, li19, qu2020} consider slightly different MARL settings, where agents collaboratively maximize the summation of agents' reward with either full or partial state observation. They also require communication between neighboring agents for a better global coordination.

For the MPG subclass,\cite{zhang22MTNS,leonardos21,song2021,zhang2022,fox2022independent} study the convergence to a NE.  \cite{song2021} designs the Nash-CA (Nash Coordinate Ascent) algorithm which requires agents to update sequentially and does not belong to the gradient-based algorithm class. While \cite{zhang2022, fox2022independent} consider gradient-based algorithms, they study softmax policies which are different from directly-parameterized policies. {\cite{ding2022independent} considers policy gradient with function approximation}. \cite{leonardos21} is the most related work to this paper, which studies the performance of gradient-based algorithms under direct parameterization. It studies the global convergence rate and develop sample complexity results for gradient play, but do not study the local geometry for general SGs.  Additionally, the sample-based algorithm considered in \cite{leonardos21} is based on Monte-Carlo gradient estimation, which might suffer from high variance in real implementation and is very different from our algorithm that estimates the gradient via estimating the ``model'' with respect to agents' averaged MDPs. Moreover, our concept of ``averaged'' MDPs could also serve as a useful tool for the design and analysis of other MARL algorithms. 




\section{Problem setting and preliminaries}\label{sec:Problem-setup}
We consider a stochastic game (SG, \cite{shapley53})  $\mathcal{M} = (N, \cS, ~\cA = \cA_1\times\dots\times \cA_n, ~P, ~r = (r_1, \dots, r_n), ~\gamma, ~\rho)$ with $n$ agents   which is specified by an agent set $N=\left\{1,2,\dots,n\right\}$, a finite state space $\cS$, a finite action space $\mathcal{A}_i$ for each agent $i\in N$,  a transition model $P$ where $P(s'|s,a) = P(s'|s,a_1, \dots, a_n)$ is the probability of transitioning into state $s'$ upon taking action $a:=(a_1,\ldots,a_n)$ in state $s$ where $a_i\in\cA_i$ is the action of agent $i$, agent $i$'s reward function $r_i: \cS\times\cA \rightarrow [0,1]$, a discount factor $\gamma \in [0,1)$, and an initial state distribution $\rho$ over $\cS$.

A stochastic policy $\pi: \cS\rightarrow\Delta(\cA)$ (where $\Delta(\cA)$ is the probability simplex over $\cA$) specifies a strategy in which agents choose their actions \textit{jointly} based on the current state in a stochastic fashion, i.e. $\Pr(a_t|s_t) = \pi(a_t|s_t)$. A decentralized stochastic policy is a special subclass of stochastic policies, with $\pi = \pi_1\times\ldots\times\pi_n$, where $\pi_i: \cS \rightarrow \Delta(\cA_i)$. For decentralized stochastic policies, each agent takes its action based on the current state $s$ \textit{independently of} other agents' choices of actions, i.e.:\vspace{-3pt}
\begin{equation*}
 \textstyle  \Pr(a_t|s_t) = \pi(a_t|s_t) = \prod_{i=1}^n \pi_i(a_{i,t}|s_t),  a_t \!=\! (a_{1,t},\! \dots\!, a_{n,t}).\vspace{-3pt}
\end{equation*}
For notational simplicity, we define: ~$\pi_{I}(a_I|s):= \prod_{i\in I} \pi_i(a_i|s)$, where $I\subseteq N$ is an index set. Further, we use the notation $-i$ to denote the index set $N\backslash \{i\}$.

We consider \textit{direct decentralized policy parameterization}, where agent $i$'s policy is parameterized by $\theta_i$:
\begin{equation}\label{eq:direct distributed parameterization}
    \pi_{i,\theta_i}(a_i|s) = \theta_{i,(s,a_i)}, \quad i=1,2,\dots,n.
\end{equation}
For notational simplicity, we abbreviate $\pi_{i,\theta_i}(a_i|s)$ as $\pi_{\theta_i}(a_i|s)$, and $\theta_{i,(s,a_i)}$ as $\theta_{s,a_i}$. Here $\theta_i \in \Delta(\cA_i)^{|\cS|}$, i.e. $\theta_i$ is subject to the constraints $\theta_{s,a_i}\ge 0$ and $\sum_{a_i\in\cA_i}\theta_{s,a_i} = 1$ for all $s\in \cS$. The global joint policy is given by:  $\pi_\theta(a|s) = \prod_{i=1}^n\pi_{\theta_i}(a_i|s) =\prod_{i=1}^n\theta_{s,a_i}.$ We use $\cX_i :=  \Delta(\cA_i)^{|\cS|}, \cX := \cX_1\times\cdots\times\cX_n$ to denote the feasible region of $\theta_i$ and $\theta$.

Agent $i$'s value function {\cite{agarwal2019reinforcement}} $V_i^\theta:\cS\rightarrow \mathbb{R}, i\in N$ is defined as the discounted sum of future rewards starting at state $s$ via executing $\pi_\theta$, i.e.\vspace{-3pt}
\begin{equation*}
    \textstyle V_i^\theta(s) := \bE \left[\sum_{t=0}^\infty \gamma^t r_i(s_t,a_t)\big|~\pi_\theta,s_0=s\right],
\end{equation*}
where the expectation is with respect to the random trajectory $\tau = (s_t,a_t,r_{i,t})_{t=0}^\infty$ where $a_t\sim \pi_\theta(\cdot|s_t), s_{t+1} = P(\cdot|s_t,a_t)$. We denote agent $i$'s total reward starting from initial state $s_0\sim \rho$ as:
\vspace{-10pt}
\begin{equation*}
 \quad  \textstyle  J_i(\theta) = J_i(\theta_1, \dots,\theta_n) := \bE_{s_0\sim\rho} V_i^\theta(s_0).
\vspace{-1pt}
\end{equation*}
In the game setting, Nash equilibrium is often used to characterize the performance of agents' policies.
\vspace{-5pt}
\begin{defi} \label{def:NE}
{(Nash equilibrium, {c.f. \cite{fudenberg1991game,maschler2020game})}}
A policy $\theta^* = (\theta_1^*, \dots, \theta_n^*)$ is called a Nash equilibrium (NE) if 
$$J_i(\theta_i^*, \theta_{-i}^*) \ge J_i(\theta_i', \theta_{-i}^*), \quad \forall \theta_{i}'\in \cX_i,\quad i\in N$$
The equilibrium is called a strict NE if the inequality holds strictly for all $ \theta_i'\in \cX_i$ and $i\in N$.
The equilibrium is called a pure NE if $\theta^*$ corresponds to a deterministic policy.
The equilibrium is called a mixed NE if it is not pure.  Further, the equilibrium is called a fully mixed NE if every entry of $\theta^*$ is strictly positive, i.e.:
$\theta_{s,a_i}^*>0, ~\forall~ a_i\in\cA_i, ~\forall~ s\in\cS, ~ i\in N$.
\end{defi}

We define the \textit{discounted state visitation distribution} {\cite{agarwal2019reinforcement}} $d_\theta$ of a policy $\pi_\theta$ given an initial state distribution $\rho$ as:\vspace{-5pt}
\begin{equation}\label{eq:discounted state visitation distribution}
  \textstyle  d_\theta(s) := \bE_{s_0\sim\rho} (1-\gamma) \sum_{t=0}^\infty\gamma^t \Prtheta(s_t=s|s_0),
\end{equation}
where $\Prtheta(s_t=s|s_0)$ is the state visitation probability that $s_t=s$ when executing $\pi_\theta$ starting at state $s_0$. Throughout the paper, we make the following assumption on the SGs we study.
\begin{assump}\label{assump:like-ergodicity}
The stochastic game $\mathcal{M}$ satisfies:~
$    d_{\theta}(s) > 0, ~ \forall s\in\cS,~  \forall \theta\in \cX$.
\end{assump}
Assumption \ref{assump:like-ergodicity} requires that every state is visited with positive probability, which is a standard assumption for convergence proofs in the RL literature (e.g. \cite{agarwal2020, Mei20},{\cite{leonardos21,ozdaglar2021independent}}). Note that this assumption could be easily satisfied if the initial distribution $\rho$ satisfy $\rho(s)>0, \forall s\in \mathcal{S}$.

Similar to centralized RL {\cite{agarwal2019reinforcement}}, define agent $i$'s $Q$-function $Q_i^\theta$ and its advantage function $A_i^\theta$ as:\vspace{-3pt}
\begin{equation*}
\begin{split}
    Q_i^\theta(s,a) &:= \bE \left[\sum_{t=0}^\infty \gamma^t r_i(s_t,a_t)\big|~\pi_\theta,s_0=s, a_0=a\right], \\ A_i^\theta(s,a)&:= Q_i^\theta(s,a) - V_i^\theta(s).
\end{split}
\end{equation*}


\noindent \textbf{\textit{`Averaged'} Markov decision process (MDP):} We further define agent $i$'s \textit{`averaged' Q-function} $\overline {Q_i^{\theta}}: \cS\times\cA_i \rightarrow \bR$ and \textit{`averaged' advantage-function} $\overline {A_i^{\theta}}: \cS\times\cA_i \rightarrow \bR$ as:
\begin{equation}\label{eq:averaged-Q-advantage}
\begin{split}
  \textstyle  \overline {Q_i^{\theta}}(s,a_i)& \textstyle:=\sum_{a_{-i}}\pi_{\theta_{-i}}(a_{-i}|s)Q_i^\theta(s,a_i,a_{-i}),\\ \textstyle \overline {A_i^{\theta}}(s,a_i)&\textstyle:=\sum_{a_{-i}}\pi_{\theta_{-i}}(a_{-i}|s)A_i^\theta(s,a_i,a_{-i}).
\end{split}
\end{equation}
Similarly, we define agent $i$'s \textit{`averaged' transition probability distribution} $\overline{P_i^\theta}:\cS\times\cS\times\cA_i\rightarrow\bR$, and \textit{`averaged' reward} $\overline{r_i^\theta}:\cS\times\cA_i\rightarrow\bR$ as:
\begin{equation*}
\begin{split}
      \textstyle  \overline{P_i^\theta}(s'|s,a_i)&\textstyle:=\sum_{a_{-i}}\pi_{\theta{-i}}(a_{-i}|s)P(s'|s,a_i,a_{-i}),\\
   \textstyle \overline{r_i^\theta}(s,a_i)&\textstyle:=\sum_{a_{-i}}\pi_{\theta{-i}}(a_{-i}|s)r_i(s,a_i,a_{-i})
\end{split}
\end{equation*}
From its definition, the averaged Q-function satisfies the following Bellman equation:
\begin{lemma}\label{lemma:averaged-Bellman-equation}
$\overline{Q_i^\theta}$ satisfies:
\begin{equation}\label{eq:averaged-Bellman-equation}
 \overline{Q_i^\theta}(s,\!a_i) \!=\! \overline{r_i^\theta}(s,\!a_i) \!\!+\!\! \gamma\!\! \sum_{s',a_i'}\!\!\pi_{\theta_i}\!(a_{i}'|s')\overline{P_i^\theta}\!(s'|s,a_i)\overline{Q_i^\theta}(s',a_i')\vspace{-10pt}
\end{equation}
\end{lemma}
Lemma \ref{lemma:averaged-Bellman-equation} suggests that the averaged Q-function $\overline{Q_i^\theta}$ is indeed the Q-function for the MDP defined on action space $\cA_i$, with $\overline{r_i^\theta}, \overline{P_i^\theta}$ as its stage reward and transition probability, respectively. We define this MDP as the \textit{`averaged' MDP} of agent $i$, i.e., $\mathcal{M}_i^\theta = (\cS,\cA_i, \overline{P_i^\theta}, \overline{r_i^\theta},\gamma,\rho) $.
The notion of an `averaged' MDP will serve as an important intuition when designing the sample-based algorithm. Note that the `averaged' MDP is only well-defined when the policies of the other agents $\theta_{-i}$ are kept fixed. When this is indeed the case, agent $i$ can be treated as an independent learner with respect to its own `averaged' MDP. Thus, various classical policy evaluation RL algorithms can then be applied. Additionally, we can apply the performance difference lemma \cite{Kakade02} to the averaged MDP to derive a corresponding lemma for SGs which is useful throughout the paper.
\begin{lemma}(Performance difference lemma, for SGs, proof see Appendix \ref{apdx:proof-performance-difference-lemma})\label{lemma:averaged-performance-difference-lemma} Let $\theta' = (\theta_i', \theta_{-i})$
\begin{equation*}
    J_i(\theta_i', \theta_{-i})- J_i(\theta_i, \theta_{-i}) =\frac{1}{1-\gamma} \sum_{s,a_i} d_{\theta'}(s) \pi_{\theta'_i}(a_i|s)\overline{A_i^{\theta}}(s, a_i).
\end{equation*}
\end{lemma}
Note that in the single agent case ($n=1$), Lemma \ref{lemma:averaged-performance-difference-lemma} is the same as the original performance difference lemma known in literature, e.g., Lemma 6.1 in \cite{Kakade02}.
\vspace{-5pt}

\section{Gradient Play for General Stochastic Games}\label{sec:gradient-play-SG}
Under direct distributed parameterization, the gradient play algorithm is given by:
\begin{equation}\label{eq:gradient-play-discrete}
  \theta_i^{(t+1)} = \text{Proj}_{\cX_i}(\theta_i^{(t)}+\eta\nabla_{\theta_i}J_i(\theta_i^{(t)})),  ~\eta\! >\!0.
\end{equation}
Gradient play can be viewed as a `better response' strategy, where agents update their own parameters by gradient ascent with respect to their own rewards.

A first-order stationary point is defined as such:
\begin{defi}\label{def:first-order}{(First-order stationary policy)}
A policy $\theta^* = (\theta_1^*, \dots, \theta_n^*)$ is called a first-order stationary policy if $(\theta_i' -\theta_i^*)^\top \nabla_{\theta_i}J_i(\theta^*)\le 0,\ \forall \theta_i' \in \cX_i,\  i\in N$.
\end{defi}
It is not hard to verify that $\theta^*$ is a first-order stationary policy if and only if it is a fixed point under gradient play \eqref{eq:gradient-play-discrete}. Comparing Definition \ref{def:NE} (of NE) and Definition \ref{def:first-order}, we know that NEs are first-order stationary policies, but not necessarily vice versa. For each agent $i$, first-order stationarity does  not imply that $\theta_i^*$ is optimal among all possible $\theta_i$ given  $\theta_{-i}^*$. However, interestingly, we will show that NEs are equivalent to first-order stationary policies due to a gradient domination property that we will show later. Before that, we first calculate the explicit form of the gradient $\nabla_{\theta_i}J_i$.

Policy gradient theorem \cite{Sutton1999} gives an efficient formula for the gradient:
\begin{equation}\label{eq:policy-gradient-theorem}
    \nabla_{\!\theta} \bE_{s_{0}\sim\rho} \!V_i^\theta(s_0) \!=\! \frac{1}{1\!-\!\gamma} \bE_{s\sim d_\theta\!,a\sim\pi_\theta(\!\cdot|s)}[\nabla_\theta\log\pi_\theta(a|s)Q_i^\theta(s,a)], 
\end{equation}
Applying \eqref{eq:policy-gradient-theorem}, the gradient $\nabla_{\theta_i}J_i$ can be written explicitly as follows:
\begin{lemma}(Proof see Appendix \ref{apdx:proof-policy-gradient})\label{lemma:policy-gradient-direct-parameterization}
For direct distributed parameterization \eqref{eq:direct distributed parameterization},
\begin{equation}\label{eq:policy-gradient-direct-parameterization}
      \frac{\partial J_i(\theta)}{\partial{\theta_{s,a_i}}} = \frac{1}{1-\gamma}d_{\theta}(s)  \overline {Q_i^{\theta}}(s, a_i)
\end{equation}
\end{lemma}

\subsection{Gradient domination and the equivalence between NE and first-order stationary policy.}
Lemma 4.1 in \cite{agarwal2020} established gradient domination for centralized tabular MDP under direct parameterization. We can show that a similar property still holds for SGs. 
\begin{lemma}\textup{(Gradient domination)}\label{lemma:gradient-domination-MAMDP}
For direct distributed parameterization \eqref{eq:direct distributed parameterization}, we have that for any $\theta = (\theta_1, \dots,\theta_n) \in \cX$ and any $\theta_i'\in\cX_i, ~ i\in N$:
\begin{equation}\label{eq:gradient-domination-MAMDP}
    J_i(\theta_i', \!\theta_{\!-\!i})- J_i(\theta_i, \!\theta_{\!-\!i}) \!\le\! \left\|\frac{d_{\theta'}}{d_\theta}\right\|_\infty \!\max_{\overline\theta_i\!\in\!\cX_i} (\overline\theta_i - \theta_i)^\top \nabla_{\theta_i} J_i(\theta),
\end{equation}
where $\left\|\frac{d_{\theta'}}{d_\theta}\right\|_\infty:=\max_s \frac{d_{\theta'}(s)}{d_\theta(s)}$, and $\theta' = (\theta_i', \theta_{-i})$.
\begin{proof}
According to Lemma \ref{lemma:averaged-performance-difference-lemma}:
\begin{equation*}
   \!J_i(\theta_i', \!\theta_{-i})\!-\! J_i(\theta_i, \!\theta_{-i})\!=\!\frac{1}{1\!-\!\gamma} \!\sum_{s,a_i} \!d_{\theta'}\!(s) \pi_{\theta'_i}(a_i|s)\overline{A_i^{\theta}}(s, \!a_i).\!\vspace{-5pt}
\end{equation*}
From the definition of `averaged' advantage function:
\begin{equation*}
    \textstyle \sum_{a_i}\pi_{\theta_i}(a_i|s)\overline {A^{\theta}_i}(s,a_i) = 0, \quad \forall s\in \cS
\end{equation*}
which implies:
$\max_{a_i\in\cA_i} \overline {A^{\theta}_i}(s,a_i)\ge 0,$ thus we have that:
\begin{equation}\label{eq:cost-difference}
\begin{split}
    &J_i(\theta_i', \theta_{-i})- J_i(\theta_i, \theta_{-i}) =\frac{1}{1-\gamma} \sum_{s,a_i} d_{\theta'}(s) \pi_{\theta'_i}(a_i|s)\overline{A_i^{\theta}}(s, a_i)\\
    &\le\!\!  \sum_{s} \!\frac{d_{\theta'}(s)}{1-\gamma} \max_{a_i\in\cA_i} \!\overline {A_i^{\theta}}(s, a_i)\!=\!  \sum_{s} \frac{d_{\theta'}(s)}{d_\theta(s)}\frac{d_{\theta}(s)}{1\!-\!\gamma} \max_{a_i\in\cA_i} \overline {A_i^{\theta}}(s, a_i)\\
    &\le \frac{1}{1-\gamma}\left\|\frac{d_{\theta'}}{d_\theta}\right\|_\infty\sum_{s} d_{\theta}(s) \max_{a_i\in\cA_i} \overline {A_i^{\theta}}(s, a_i).   
\end{split}
\end{equation}
We can rewrite $\frac{1}{1-\gamma}\sum_{s} d_{\theta}(s) \max_{a_i\in\cA_i} \overline {A_i^{\theta}}(s, a_i)   $ as:
\begin{equation}\label{eq:first-order-detail}
\begin{split}
 &\frac{1}{1-\gamma}\sum_{s} d_{\theta}(s) \max_{a_i\in\cA_i} \overline {A_i^{\theta}}(s, a_i)  \\  &= \frac{1}{1-\gamma}\max_{\overline\theta_i\in\cX_i}\sum_{s,a_i} d_{\theta}(s)\pi_{\overline\theta_i}(a_i|s) \overline {A_i^{\theta}}(s, a_i)\\
    &= \max_{\overline\theta_i\in\cX_i}\sum_{s,a_i} (\pi_{\overline\theta_i}(a_i|s) - \pi_{\theta_i}(a_i|s))\frac{1}{1-\gamma}d_{\theta}(s) \overline {A_i^{\theta}}(s, a_i)\\
   & = \max_{\overline\theta_i\in\cX_i}\left(\sum_{s,a_i} (\pi_{\overline\theta_i}(a_i|s) - \pi_{\theta_i}(a_i|s))\frac{1}{1-\gamma}d_{\theta}(s) \overline {Q_i^{\theta}}(s, a_i)\right.\\
    &\quad- \underbrace{\left. \sum_{s}\frac{1}{1-\gamma}d_{\theta}(s)V(s)\sum_{a_i}(\pi_{\overline\theta_i}(a_i|s) - \pi_{\theta_i}(a_i|s)) \right)}_{=0}\\
    &= \max_{\overline\theta_i\in\cX_i}\sum_{s,a_i} (\pi_{\overline\theta_i}(a_i|s) - \pi_{\theta_i}(a_i|s))\frac{1}{1-\gamma}d_{\theta}(s) \overline {Q_i^{\theta}}(s, a_i)\\
    & = \max_{\overline\theta_i\in\cX_i} (\overline\theta_i - \theta_i)^\top \nabla_{\theta_i} J_i(\theta).
\end{split}
\end{equation}
Substituting this into \eqref{eq:cost-difference}, we may conclude that
$$J_i(\theta_i', \theta_{-i})- J_i(\theta_i, \theta_{-i}) \le \left\|\frac{d_{\theta'}}{d_\theta}\right\|_\infty \max_{\overline\theta_i\in\cX_i} (\overline\theta_i - \theta_i)^\top \nabla_{\theta_i} J_i(\theta)$$
and this completes the proof.
\end{proof}
\end{lemma}

 For the single-agent case ($n=1$), \eqref{eq:gradient-domination-MAMDP} is consistent with the result in \cite{agarwal2020}, i.e.:~~$  J(\theta') - J(\theta) \le   \left\|\frac{d_{\theta'}}{d_\theta}\right\|_\infty \max_{\overline\theta\in\cX} (\overline\theta - \theta)^\top \nabla J(\theta)$. However, when there are multiple agents, the condition is much weaker because the inequality requires $\theta_{-i}$ to be fixed. When $n=1$, gradient domination rules out the existence of stationary points that are not global optima. For the multi-agent case, the property can no longer guarantee the equivalence between first-order stationarity and global optimality; instead,  it links the stationary points with NEs as shown in the following theorem.

\begin{theorem}
\label{thm:equivalence-stationary-NE}
Under Assumption \ref{assump:like-ergodicity}, first-order stationary policies and NEs are equivalent. 
\begin{proof}
The definition of a Nash equilibrium naturally implies first order stationarity, because for any $\theta_i\in \cX_i$:
\begin{align*}
    J_i&((1-\eta)\theta_i^* + \eta\theta_i, \theta_{-i}^*) - J_i(\theta_i^*, \theta_{-i}^*) =\\ &\eta(\theta_i-\theta_i^*)^\top \nabla_{\theta_i}J_i(\theta^*) + o(\eta\|\theta_i - \theta_i^*\|) \le 0, \quad \forall~\eta >0
\end{align*}
Letting $\eta \rightarrow 0$ gives the first order stationary condition:
\begin{equation*}
    (\theta_i-\theta_i^*)^\top \nabla_{\theta_i}J_i(\theta^*) \le 0, \quad \forall \theta_i\in \cX_i,
\end{equation*}

It remains to be shown that all first order stationary policies are Nash equilibria. 
From Assumption \ref{assump:like-ergodicity} we know that for any pair of parameters $\theta', \theta^*$,
~~$\left\|\frac{d_{\theta'}}{d_{\theta^*}}\right\|_\infty \!<\!+\infty.$\\
Take $\theta' = (\theta_i', \theta_{-i}^*), \theta^* = (\theta_i^*, \theta_{-i}^*)$. According to Lemma \ref{lemma:gradient-domination-MAMDP}, we have that for any first order stationary policy $\theta^*$,
\begin{equation*}
    J_i(\theta_i', \theta_{-i}^*)\!-\! J_i(\theta_i^*, \theta_{-i}^*) \!\le\! \left\|\frac{d_{\theta'}}{d_{\theta^*}}\right\|_\infty\!\! \max_{\overline\theta_i\in\cX_i} (\overline\theta_i \!-\! \theta_i^*)^{\!\!\top} \nabla_{\theta_i} \!\!J_i(\theta^*) \!\le\! 0,
\end{equation*}
which completes the proof.
\end{proof}
\end{theorem}
We briefly note here that the equivalence between the first-order stationary points and NEs holds for all SGs that satisfy Assumption~\ref{assump:like-ergodicity}. One implication from the theorem is that for identical interest case where agents have the same rewards, we can only ensure the first order stationary points to be NEs when the policies are \textit{decentralized} policies. Note that NEs are often none-unique and often with different objective values. This is in contrast to the single agent/\textit{centralized} case where the first order stationary point is equivalent to the global optimal point\cite{agarwal2020}.
\vspace{-6pt}
 \subsection{Local convergence for strict NEs}
Although the equivalence of NEs and stationary points under gradient play has been established, it is in fact difficult to show that gradient play converges to these stationary points. Even in the simpler static (stateless) game setup, gradient play might fail to converge \cite{Shapley64,Crawford85,Jordan93,Krishna98}. One major difficulty is that the vector field $\{\nabla_{\theta_i}J_i(\theta)\}_{i=1}^n$ is not a conservative vector field
. Accordingly, its dynamics may display complicated behavior. Thus, as a preliminary study, instead of looking at global convergence, we focus on the local convergence and restrict our study to a special subset of NEs - the strict NEs. We begin by giving the following characterization of strict NEs:
\begin{lemma}\label{lemma:strict-NE-gap} 
Given a stochastic game $\mathcal{M}$, any strict NE $\theta^*$ is pure, meaning that for each $i$ and $s$, there exist one $a_i^*(s)$ such that $\theta_{s,a_i}^*=\mathbf{1}\{a_i=a_i^*(s)\}$. Additionally, we have 
\begin{align*}
    i)\ &  \textstyle a_i^*(s) = \argmax_{a_i} \overline{A_{i}^{\theta^*}}(s,a_i),\\
    ii)\ & \textstyle \overline{A_{i}^{\theta^*}}(s,a_i^*(s)) =0,\\
    iii)\ &\textstyle \overline{A_{i}^{\theta^*}}(s,a_i)<0, ~\forall~a_i\neq a_i^*(s).
\end{align*}

\end{lemma}
Based on this lemma, we define the following for studying the local convergence of a strict NE $\theta^*$:
\vspace{-3pt}
\begin{equation}
\begin{split}
    &\textstyle \Delta_i^{\theta^*}(s):= \min_{a_i\neq a_i^*(s)}\left|\overline{A_{i}^{\theta^*}}(s,a_i) \right|,\\ &\textstyle \Delta^{\theta^*}:=\min_i\min_s \frac{1}{1-\gamma}d_{\theta^*}(s)\Delta_i^{\theta^*}(s) > 0. 
\vspace{-2pt}
\end{split}
    \label{eq:delta}
\end{equation}

\begin{theorem}\label{theorem:local-convergence-general-game}
(Local finite time convergence around strict NE) 
Define the metric of policy parameters as:~  $D(\theta||\theta'):= \max_{1\le i\le n}\max_{s\in\cS} \|\theta_{i,s} - \theta'_{i,s}\|_1,$ where $\|\cdot\|_1$ denote the $\ell_1$- norm. 
Suppose $\theta^*$ is a \textit{strict} Nash equilibrium, then for any $\theta^{(0)}$ such that~
    $D(\theta^{(0)}||\theta^*)\! \le \!\frac{\Delta^{\theta^*}(1-\gamma)^3}{8n|\cS|\left(\sum_{i=1}^n|\cA_i| \right)},$~
running gradient play \eqref{eq:gradient-play-discrete} will guarantee
    $D(\theta^{(t+1)}||\theta^*)\le \max\left\{D(\theta^{(t)}||\theta^*) - \frac{\eta\Delta^{\theta^*}}{2},0\right\},$
which means that gradient play is going to converge within $ \lceil\frac{ 2D(\theta^{(0)}||\theta^*)}{\eta\Delta^{\theta^*}}\rceil$ steps.
\end{theorem}
\noindent Proofs are deferred to Appendix \ref{apdx:local-convergence-general-game}.

\begin{rmk}\label{rmk: local NE general SGs}
Note that the local convergence in Theorem \ref{theorem:local-convergence-general-game} only requires a finite number of steps. The key insight of the proof is that the gradient always points towards $\theta^*$, and that the algorithm projects the gradient update onto the probability simplex, thus by picking the stepsize $\eta$ arbitrarily large, exact convergence can be achieved by just one step. However, the caveat is that we need to assume that the initial policy is sufficiently close to $\theta^*$. 
For numerical stability considerations, one should pick reasonable stepsizes to run the algorithm to accommodate random initializations. Theorem \ref{theorem:local-convergence-general-game} also shows that the radius of region of attraction for strict NEs is at least $\frac{\Delta^{\theta^*}(1-\gamma)^3}{8n|\cS|\left(\sum_{i=1}^n|\cA_i| \right)}$, and thus $\theta^*$ with a larger $\Delta^{\theta^*}$, i.e., a larger value gap between the optimal action and other actions, will have a larger region of attraction. We would like to further remark that Theorem \ref{theorem:local-convergence-general-game} only focuses on the local convergence property; hence, we can interpret the theorem in the following way: if there exists a strict NE, then it is locally asymptotically stable under gradient play. However, it does not claim to solve the global existence or convergence of the strict NEs.
\end{rmk}

\vspace{-5pt}

\section{Gradient play for Markov potential games}
We have discussed that the main problem for the global convergence of gradient play for general SGs is that the vector field $\{\nabla_{\theta_i}J_i(\theta)\}_{i=1}^n$ is not conservative. Thus, in this section, we restrict our analysis to a special subclass where the vector field is conservative, which in turn enjoys global convergence. This subclass is generally referred to as a Markov potential game (MPG) in the literature.
\begin{defi}(Markov potential game \cite{leonardos21})\label{defi:MPG}
A stochastic game $\mathcal{M}$ is called a Markov potential game if there exists a potential function $\phi: \cS \times \cA_1\times\cdots\times\cA_n \!\rightarrow\! \mathbb{R}$ such that for any agent $i$ and any pair of policy parameters $(\theta_i',\theta_{-i}), (\theta_i,\theta_{-i})$ :
\begin{align*}
    &\bE \left[\sum_{t=0}^\infty \gamma^t r_i(s_t,a_t)\big|\pi = (\theta_i',\theta_{-i}),s_0=s\right] -\bE \left[\sum_{t=0}^\infty \gamma^t r_i(s_t,a_t)\big|\pi = (\theta_i,\theta_{-i}),s_0=s\right]\\
    =&\bE \left[\sum_{t=0}^\infty \gamma^t \phi(s_t,a_t)\big|\pi = (\theta_i',\theta_{-i}),s_0=s\right] - \bE \left[\sum_{t=0}^\infty \gamma^t \phi(s_t,a_t)\big|\pi = (\theta_i,\theta_{-i}),s_0=s\right],~~\forall~s.
\end{align*}
\end{defi}
As shown in the definition, the condition of a MPG is admittedly rather strong and difficult to verify for general SGs.  \cite{Macua18,Gonzalez13} found that continuous MPGs can model applications such as the great fish war \cite{Levhari80}, the stochastic lake game \cite{Dechert06}, medium access control \cite{Macua18} etc. There are also efforts attempting to identify conditions such that a SG is a MPG, e.g., \cite{Macua18,leonardos21,mguni2021learning}. In Appendix \ref{apdx:More-about-MPG}, we provide a more detailed discussion on MPGs, including a necessary condition of MPG, counterexamples of stage-wise potential games that are not MPG, sufficient conditions for a SG to be a MPG, and application examples of MPG. Nevertheless, identifying sufficient and necessary conditions and broadening the applications of MPG are important furture directions. 
  

Given a policy $\theta$, we define the \textit{`total potential function'}\footnote{{Note that our definition of MPG is slightly stronger than the definition in \cite{leonardos21} as it requires the total potential function to take the particular form as a discounted sum of the potential function, i.e., $\Phi(\theta) := \bE_{s_0\sim \rho(\cdot)} \left[\sum_{t=0}^\infty \gamma^t \phi(s_t,a_t)\big|~\pi_\theta\right]$. However, most of our results (Theorem \ref{theorem:potential-game-convergence} and Theorem \ref{thm:main}) still hold under the weaker definition in \cite{leonardos21}. Theorem \ref{theorem:MPG-local-maximum-saddle-points} is the only result that relies on the slightly stronger version of the definition}}
$      \Phi(\theta) := \bE_{s_0\sim \rho(\cdot)} \left[\sum_{t=0}^\infty \gamma^t \phi(s_t,a_t)\big|~\pi_\theta\right]$ for a MPG. The following proposition guarantees a MPG has at least one NE and it is a pure NE (proof see Appendix \ref{apdx:proof-potential-game-pure-NE}).{We also define the quantity $\Phi_{\max}, \Phi_{\min}$ as $\Phi_{\max}:=\frac{\phi_{\max}}{1-\gamma}, \Phi_{\min}:=\frac{\phi_{\min}}{1-\gamma}$, where $\phi_{\min}:= \min_{s,a}\phi(s,a)$ and $\phi_{\max}:= \max_{s,a}\phi(s,a)$. It can be easily verified that $\Phi_{\min}\le \Phi(\theta)\le \Phi_{\max}$ for all $\theta$.}
\begin{prop}\label{prop:potential-game-pure-NE}
For a Markov potential game, there is at least one global maximum  $\theta^*$ of the total potential function $\Phi$, i.e.: $\theta^* \in \argmax_{\theta\in\cX}\Phi(\theta)$ that is a pure NE. 
\end{prop}

From the definition of the total potential function we obtain the following relationship
\begin{equation}\label{eq:total-potential-function-relationship}
    J_i(\theta_i', \theta_{-i}) - J_i(\theta_i, \theta_{-i})  = \Phi(\theta_i', \theta_{-i}) - \Phi(\theta_i, \theta_{-i}).
\end{equation}
Thus,
\vspace{-10pt}
\begin{equation*}
    \nabla_{\theta_i}J_i(\theta) = \nabla_{\theta_i}\Phi(\theta),
\end{equation*}
which means that gradient play \eqref{eq:gradient-play-discrete} is equivalent to running projected gradient ascent with respect to the total potential function $\Phi$, i.e.: 
$$\theta^{(t+1)} = \text{Proj}_{\cX}(\theta^{(t)}+\eta\nabla_{\theta}\Phi(\theta_i^{(t)})), ~\eta >0.$$

\subsection{Global convergence}
With the above property, we can establish the global convergence for gradient play to a $\epsilon$-NE for MPG. For the sake of self-completeness, we include the theorem here. Before that, we define $\epsilon$-NE. 

\begin{defi}{($\epsilon$-Nash equilibrium)}
Define the `NE-gap' of a policy $\theta$ as:
\begin{align*}
    \NEgap_i(\theta)&:= \max_{\theta_i'\in\cX_i} J_i(\theta_i', \theta_{-i}) - J_i(\theta_i, \theta_{-i});\\ \NEgap(\theta)&:= \max_i \NEgap_i(\theta).
\end{align*}
A policy $\theta$ is an $\epsilon$-Nash equilibrium if:
   ~~$ \NEgap(\theta) \le \epsilon.$
\end{defi}

\begin{theorem}\label{theorem:potential-game-convergence}
 Suppose for all $\theta\in\cX$, with stepsize $\eta = \frac{(1-\gamma)^3}{2\sum_{i=1}^n|\cA_i|}$, the $\NEgap$ of $\theta^{(t)}$ asymptotically converge to 0 under gradient play \eqref{eq:gradient-play-discrete}, i.e., $\lim_{t\rightarrow\infty}\NEgap(\theta^{(t)}) = 0$. Further, we have:
\begin{equation}\label{eq:MPG-EGP-rate}
\begin{split}
     &\quad \frac{1}{T}\sum_{1\le t \le T} \NEgap(\theta^{(t)})^2 \le \epsilon^2, 
     \\
     \textup{whenever } ~& T \ge \frac{64M^2(\Phi_{\max} - \Phi_{\min})|\cS|\sum_{i=1}^n|\cA_i|}{(1-\gamma)^3\epsilon^2}, 
\end{split} 
\end{equation} 
where ~$M:= \max_{\theta, \theta'\in \cX}\left\|\frac{d_\theta}{d_{\theta'}}\right\|_\infty$ (by Assumption \ref{assump:like-ergodicity}, we know that this quantity is well-defined). \footnote{{Another way to interpret the result is that the average term $\frac{1}{T}\sum_{1\le t \le T} \NEgap(\theta^{(t)})^2 \le \epsilon^2$ could be translated to a constant probability guarantee on single $\NEgap(\theta^{(t)})$. For instance, if we randomly pick one $\theta^{(t)}$ from $1 \le t \le T$, then it guarantees that  $\NEgap(\theta^{(t)})^2 \le 3\epsilon^2$ with probability at least $2/3$ (the constants $3,2/3$ could be replaced by $\frac{1}{1-p}, p$ where $p\in (0,1)$).}}
\end{theorem}
The factor $M$ is also known as the distribution mismatch coefficient that characterizes how the state visitation varies with the policies. Given an initial state distribution $\rho$ that has positive measure on every state, $M$ can be at least bounded by $M \!\le\! \frac{1}{1-\gamma}\max_\theta\left\|\frac{d_{\theta}}{\rho}\right\|_\infty\!\le\!\frac{1}{1-\gamma}\frac{1}{\min_s\rho(s)}$. The proof structure of Theorem \ref{theorem:potential-game-convergence} resembles the proof of convergence for single-agent MDPs in \cite{agarwal2020}, where they leverage classical nonconvex optimization results \cite{beck2017,Ghadimi16} and gradient domination to get the convergence rate of $O\left(\frac{64\gamma\left\|\frac{d_{\theta^*}}{\rho}\right\|_\infty^2|\cS||\cA|}{(1-\gamma)^6\epsilon^2}\right)$ to the global optimum (see Appendix \ref{apdx:proof-potential-game-convergence} for proof details). In fact, our result matches this bound when there is only one agent (the exponential factor on $(1\!-\!\gamma)$ looks slightly different because some factors are hidden implicitly in $M$ and $(\Phi_{\max}\! -\! \Phi_{\min})$ in our bound).

\subsection{Local geometry of NEs}
Theorem \ref{theorem:potential-game-convergence} suggests that gradient play is guaranteed to converge to a NE, however, which exact NE it converges to is not specified in the theorem.
The qualities of NEs can vary significantly. For example, consider a simple two-agent identical-interest normal form game with reward table given in Table \ref{tab:reward_table-normal-form}. There are three NEs. Two of them are strict NEs, where both agents choose the same action, i.e. $a_1 = a_2 = 1 \textup{ or } 2$. Both NEs are of reward $1$. Another NE is a fully mixed NE, where both agents choose action $1$ and $2$ randomly with probability $\frac{1}{2}$. This NE is only of reward $\frac{1}{2}$. This significant quality difference between different types of NEs motivates us to further understand whether gradient play can find NEs with relatively good qualities. Since the NE that gradient play converges to depends on the initialization and the local geometry around the NE,
as a preliminary study, we characterize the local geometry and landscape for strict NEs and fully mixed NEs (stated in the following theorem). More future investigation is needed for non-strict, non-fully-mixed NEs.   
\begin{table}[htbp]
\centering
\begin{tabular}{|c||c|c|}
\hline
 &$a_2 = 1$&$a_2 = 2$\\
\hline\hline
$a_1 = 1$ & 1&0\\
\hline
$a_1 = 2$ &0&1\\
\hline
\end{tabular}
\caption{}
\label{tab:reward_table-normal-form}
\vspace{-10pt}
\end{table}
\begin{theorem}\label{theorem:MPG-local-maximum-saddle-points} 
For a Markov potential game with $\Phi_{\min}\!<\!\Phi_{\max}$ (i.e., $\Phi$ is not a constant function):
\hspace{-10pt}
\begin{itemize} [topsep=0pt,itemsep=0ex,partopsep=1ex,parsep=1ex]
    \item A strict NE $\theta^*$ is equivalent to a strict local maximum of the total potential function $\Phi$, i.e.:
$\exists ~\delta,$ such that for all $\theta\!\in\! \cX, \theta\!\neq\!\theta^* $ that satisfies  $ \|\theta - \theta^*\| \le \delta,$ we have ~$\Phi(\theta) < \Phi(\theta^*)$.
\item Any fully mixed NE $\theta^*$ is a saddle point of the total potential function $\Phi$, i.e., $\|\nabla\Phi(\theta^*)\|=0$, and $ \forall~\delta>0,~~\exists~\theta\in\cX, ~\text{such that}~ \|\theta\!-\!\theta^*\|\le \delta \text{ and } \Phi(\theta) \!>\! \Phi(\theta^*).$
\end{itemize}
\end{theorem}
The full proof of Theorem \ref{theorem:MPG-local-maximum-saddle-points} is deferred to Appendix \ref{apdx:MPG-local-maximum-saddle-points}.

\begin{rmk} Theorem \ref{theorem:MPG-local-maximum-saddle-points} implies that strict NEs are asymptotically locally stable under first-order methods such as gradient play; while the fully mixed NEs are unstable under gradient play. 
Note that the theorem does not claim stability or instability for other types of NEs, e.g., pure NEs or non-fully mixed NEs.
Nonetheless, we believe that these preliminary results can serve as a valuable platform towards a better understanding of the geometry of the problem. We remark that the conclusion about strict NEs in Theorem \ref{theorem:MPG-local-maximum-saddle-points} does not hold for settings other than tabular MPG; for instance, for continuous games, one can use quadratic functions to construct simple counterexamples \cite{mazumdar20}. Also, similar to Remark \ref{rmk: local NE general SGs}, this theorem focuses on the local geometry of the NEs but does not claim the global existence or convergence of either strict NEs or fully mixed NEs.
\end{rmk}

\subsection{Sample-based learning: algorithm and sample complexity}\label{subsec:sample-based-MPG}
In this section, we no longer assume access to the exact gradient, but instead need to estimate it via samples. Throughout the section, we make the following additional assumption on MPGs:
\begin{assump}{($(\tau,\sigma_S)$-Sufficient exploration on states)}\label{assump:sufficient-exploration-on-state} 
There exist a positive integer $\tau$ and a $\sigma_S\in(0,1)$ such that for any policy $\theta$ and any initial state-action pair $(s,a_i),~\forall i$, we have
\begin{equation}
\textstyle
    \Pr^\theta(s_{\tau}|s_0=s,a_0=a) \ge \sigma_S,~~\forall s_{\tau},
\end{equation}
{i.e., it poses a condition on the mixing time of the Markov chain induced by any policy $\theta$: there exists a sufficiently long time $\tau$, so the probability of being at any state at time $\tau$ is at least $\sigma_S$ for any initial state and action pair.}
\end{assump}
{\noindent Note that similar assumptions are common in proving finite time convergence of RL algorithms (e.g. \cite{qu2020,srikant19, Li22}) where ergodicity of the Markov chain induced by certain policies is generally assumed, which results in every state-action pair being visited with positive probability in the stationary distribution.}.

We further introduce the \textit{state transition probability under $\theta$} $\overline{P_\cS^\theta}:\cS\times\cS\rightarrow\mathbb{R}$ as: 
\begin{equation*}
\textstyle
    \overline{P_\cS^\theta}(s'|s):=\sum_a\pi_{\theta}(a|s)P(s'|s,a).
\end{equation*}
We consider fully decentralized learning, where agent $i$'s observation  only includes state $s_t$, its own action $a_{i,t}$, and its own reward $r_{i,t}:=r_i(s_t,a_t)$ at time $t$. Such fully decentralized learning is plausible due to the fact that  when $\theta_{-i}$ is \textit{fixed}, agent $i$ can be treated as an independent learner with the underlying MDP being the `averaged' MDP described in Section~\ref{sec:Problem-setup}. With this key observation, we design a `model-based' on-policy learning algorithm, where agents perform policy evaluation in the inner loop and gradient ascent at the outer loop. The algorithm is provided in Algorithm~\ref{alg:sample-based learning}. Roughly, it consists of three main steps: 1) (Inner loop) Estimate the averaged transition probability and reward using on-policy samples $\overline{P_i^{\theta}},\overline{r_i^{\theta}},\overline{P_\cS^\theta}$. 2) (Inner loop) Calculate averaged Q-function $\overline{Q_i^\theta}$ and discounted state visitation distribution $d_\theta$, and compute the estimated gradient accordingly, 3) (Outer loop) Running projected gradient ascent with estimated gradients. Before discussing our algorithm in more detail, we highlight that the idea of using the ``averaged'' MDP can be used to design other learning methods including model-free methods, e.g., using the temporal difference methods to perform policy evaluation. One caveat is that the ``averaged'' MDP is only well-defined when all the other agents use fixed policies. This makes it difficult to extend the two-timescale framework {(i.e. with an inner loop and outer loop)} to single-timescale settings, which is an interesting future direction. {Further, note that the current algorithm requires full state observation, it remains an intriguing open question to extend it to the case with only partial observability.} {We would also like to point out that the algorithm initialization still requires extra
consensus/coordination among the players to agree on the hyperparameters $T_J, T_G$ etc, which guarantees that agents go through the same
equal-length phases to sample the trajectories, and compute gradient estimates.}
\begin{algorithm}[h]
\begin{algorithmic}
\REQUIRE learning rate $\eta$, greedy parameter $\alpha$, sample trajectory length $T_J$, total iteration steps $T_G$
\STATE For each agent $i$
\FOR {$k=0,1\dots,T_G-1$}
\FOR {$t=0,1,\dots, T_J$}
\STATE\hspace{-10pt} {Sample $s_0\sim \rho$}, implement policy $\theta^{(k)}$ and collect trajectory $\cD_i^{(k)}$: $\cD_i^{(k)}\!\leftarrow\! \cD_i^{(k)} \!\cup\! \{\!s_t , a_{i,t}, r_{i,t}\!\},~a_{i,t}\!\sim\!\pi_{\theta_i^{(k)}}(\cdot|s_t)$
\ENDFOR
\STATE Estimate $\widehat{P_i^\theta}, \widehat{r_i^\theta}, \widehat{P_\cS^\theta}, \widehat{M_i^\theta}$ by \eqref{eq:P-i-hat},\ \eqref{eq:r-i-hat}, \ \eqref{eq:P-S-hat} respectively.
\STATE Calculate $\widehat{Q_i^\theta}, \widehat{d_\theta}$ by \eqref{eq:Q-i-hat}, \ \eqref{eq:d-hat} respectively.
\STATE Estimate the gradient by \eqref{eq:estimate-gradient}:
\STATE Run projected gradient ascent as in \eqref{eq:projection-alpha-greedy}
\ENDFOR
\end{algorithmic}
\caption{Sample-based learning}
\label{alg:sample-based learning}
\end{algorithm}

\noindent\textbf{Step 1: empirical estimation of $\overline{P_i^{\theta}},\overline{r_i^{\theta}},\overline{P_\cS^\theta}$:}
Given a sequence $\{s_t, a_{i,t}, r_{i,t}\}_{t=0}^{T_J}$ generated by a policy $\theta:=(\theta_i,\theta_{-i})$, the empirical estimation $\widehat{P_i^{\theta}}$ of $\overline{P_i^{\theta}}$ is given by:
\begin{align}
    \widehat{P_i^{\theta}}(s'|s,a_i)&:=\left\{
    \begin{array}{ll}
      \frac{\sum_{t=0}^{T_J-1}\mathbf{1}\{s_{t+1}=s',s_t=s,a_{i,t}=a_i\}}{\sum_{t=1}^{T_J-1}\mathbf{1}\{s_t=s,a_{i,t}=a_i\}}, \vspace{5pt}\\   \quad~~ {\small\textup{for  } {\sum_{t=1}^{T_J-1}\mathbf{1}\{s_t=s,a_{i,t}=a_i\}}\ge 1};\vspace{5pt}\\
      \mathbf{1}\{s'=s\},    \vspace{5pt}\\   \quad~~ {\small\textup{for  } {\sum_{t=1}^{T_J-1}\mathbf{1}\{s_t=s,a_{i,t}=a_i\}}= 0}.
    \end{array}
    \right.\label{eq:P-i-hat}
    \vspace{-5pt}
\end{align}
Here we separately treat the special case where the state and action pair is not visited through the whole trajectory, i.e., ${\sum_{t=1}^{T_J-1}\mathbf{1}\{s_t=s,a_{i,t}=a_i\}}=0$ to make $\widehat{P_i^{\theta}}$ well-defined. 

Similarly, the estimates $\widehat{r_i^{\theta}},\widehat{P_\cS^\theta}$ of $\overline{r_i^{\theta}}, \overline{P_\cS^\theta}$ are given by:
\begin{align}
    \widehat{r_i^{\theta}}(s,a_i)&\!:=\!\left\{
    \begin{array}{ll}
    \frac{\sum_{t=0}^{T_J} \mathbf{1}\{s_t=s,a_{i,t}=a_i\}r_{i,t}}{\sum_{t=0}^{T_J}\mathbf{1}\{s_t=s,a_{i,t}=a_i\}},\vspace{5pt}\\\quad{\small \textup{for }{\sum_{t=1}^{T_J-1}\mathbf{1}\{s_t=s,a_{i,t}=a_i\}}\ge 1};\vspace{5pt}\\
    0,\\
    \quad{\small \textup{for }{\sum_{t=1}^{T_J-1}\mathbf{1}\{s_t=s,a_{i,t}=a_i\}}=0}.
    \end{array}
    \right.\label{eq:r-i-hat}\\
    \widehat{P_\cS^\theta}(s'|s)&\!:=\!\left\{
    \begin{array}{ll}
    \frac{\sum_{t=0}^{T_J-1}\mathbf{1}\{s_{t+1}=s',s_t=s\}}{\sum_{t=1}^{T_J-1}\mathbf{1}\{s_t=s\}}, ~ \sum_{t=1}^{T_J\!-\!1}\!\mathbf{1}\{s_t\!=\!s\} \!\ge\! 1;\vspace{5pt}\\
     \mathbf{1}\{s'=s\}, \qquad\qquad\quad~ \sum_{t=1}^{T_J\!-\!1}\!\mathbf{1}\{s_t\!=\!s\} \!=\! 0.
     \end{array}
    \right.\label{eq:P-S-hat}
\end{align}

\noindent\textbf{Step 2: estimation of $\overline{Q_i^\theta}, d_\theta$:} We slightly abuse notation and use $\overline{Q_i^\theta}, \overline{r_i^\theta} \in \bR^{|\cS||\cA_i|}$ to also denote the vectors corresponding to the averaged Q-function and reward function of agent $i$. Similarly, $\rho, d_\theta\in\bR^{|\cS|}$ are used to denote the vectors for  $\rho(s)$ and  $d_\theta(s)$. Define $M_i^\theta\in\bR^{|\cS||\cA_i|\times|\cS||\cA_i|}$:
\begin{equation*}
\textstyle
    \overline{M^\theta_i}{_{(s,a_i)\rightarrow(s',a_i')}}:= \pi_{\theta_i}(a_i'|s')\overline{P_i^\theta}(s'|s,a_i).
\end{equation*}
Then from Lemma \ref{lemma:averaged-Bellman-equation}, $\overline{Q_i^\theta}$ is given by:
\begin{equation*}
    (I-\gamma \overline{M^\theta_i})\overline{Q_i^\theta} = \overline{r_i^\theta} ~\Longrightarrow~ \overline{Q_i^\theta} = (I-\gamma \overline{M^\theta_i})^{\!-\!1}\overline{r_i^\theta}.
\end{equation*}
The estimated averaged Q-function $\widehat{Q_i^\theta}$ is given by:\footnote{From the Perron-Frobenius theorem, we know that the absolute values of the eigenvalues of $\widehat{M^\theta_i}$ are upper bounded by $1$, which guarantees that the matrix $I-\gamma \widehat{M^\theta_i}$ is invertible.}
\begin{equation}
\begin{split}
       &\widehat{Q_i^\theta} = (I-\gamma \widehat{M^\theta_i})^{\!-\!1}\widehat{r_i^\theta},\\
\textup{   where }&
    \widehat{M^\theta_i}_{(s,a_i)\rightarrow(s',a_i')}:= \pi_{\theta_i}(a_i'|s')\widehat{P_i^\theta}(s'|s,a_i). \label{eq:Q-i-hat}
\end{split}
\end{equation}
Similarly, from \eqref{eq:discounted state visitation distribution}, we have that $d_\theta$ and $\widehat{d_\theta}$ are given by (derivation see Appendix \ref{apdx:d-theta}):
\vspace{-5pt}
\begin{equation}
\textstyle    {d_\theta}\!=\!(1\!-\!\gamma)\left(I\!-\!\gamma\overline{P_\cS^\theta}^\top\right)^{\!-\!1}\hspace{-10pt}\rho, \quad
    \widehat{d_\theta}\!:=\!(1\!-\!\gamma)\left(I\!-\!\gamma\widehat{P_\cS^\theta}^\top\!\!\right)^{\!-\!1}\hspace{-10pt}\rho. \label{eq:d-hat}
\end{equation}
Then accordingly, the estimated gradient is computed as:
\begin{equation}
    \widehat{\partial}_{\theta_{s,a_i}}J_i(\theta^{(k)}) = \frac{1}{1-\gamma}\widehat{d_\theta}(s)\widehat{Q_i^\theta}(s,a_i).\label{eq:estimate-gradient}
\end{equation}

\vspace{3pt}
\noindent\textbf{Step 3: Projected gradient ascent onto the set of $\alpha$-greedy policies:} Let $U_n \!=\! [\frac{1}{n},\dots,\frac{1}{n}]\!\in\! \Delta(n)$ denote the $n$ dimensional uniform distribution. Define $\Delta\!^\alpha(n)\!:=\! \{\theta|~\exists \theta'\!\in\!\Delta(n), s.t.~ \theta=(1\!-\!\alpha)\theta'+\alpha U_n\}.$ We use $\cX_i^\alpha \!:=\!  \Dalpha^{|\cS|}, ~ \cX^\alpha \!:=\! \cX_1^\alpha\times\cX_2^\alpha\times\cdots\times\cX_n^\alpha$ to denote the set of the $\alpha$-greedy policies for $\theta_i$ and $\theta$ respectively. Every step after doing gradient ascent, the parameter $\theta$ will further be projected onto $\cX^\alpha$, i.e.:
\vspace{-5pt}
\begin{equation}\label{eq:projection-alpha-greedy}
    \theta_i^{(k+1)} = Proj_{\cX_i^\alpha}(\theta_i^{(k)} + \eta\widehat{\nabla}_{\theta_i}J_i(\theta^{(k)})).
\end{equation}
The reason of projecting onto $\cX^\alpha$ instead of $\cX$ is to make sure that every action has positive possibility of being selected in order to get a relatively accurate estimation of averaged $Q$-function. Intuitively, a larger $\alpha$ introduces a larger additional error in the NE-gap; however, a smaller $\alpha$ requires more samples to estimate the gradient. Thus the choice of $\alpha$ is the tradeoff between the two effects.

\begin{theorem}(Sample complexity)\label{thm:main}
Assume that the MPG satisfies Assumption \ref{assump:sufficient-exploration-on-state}. Let $M:= \max_{\theta, \theta'\in \cX}\left\|\frac{d_\theta}{d_{\theta'}}\right\|_\infty$. In Algorithm \ref{alg:sample-based learning}, for $\eta\le\frac{(1-\gamma)^3}{4\sum_i|\cA_i|}$, $\alpha=\frac{(1-\gamma)\epsilon}{6M}$, and
\begin{equation*}
\begin{split}
    T_{\!J}\!&\ge\!\! \frac{206976\tau nM^4|\cS|^{3}\!\max_i\!|\cA_i|^3}{(1-\gamma)^8\epsilon^4\sigma_S^2}\!\log\!\left(\!\frac{16\tau T_{\!G}|\cS|^{2}\!\sum_i\!|\cA_i|\!}{\delta}\!\right) \!+\! \tau ,\\ &\quad\qquad\qquad~ T_G\!\ge\!\frac{648M^2(\!\Phi_{\max}\!-\!\Phi_{\min}\!)|\cS|}{\eta\epsilon^2},
\end{split}
\end{equation*}
with probability at least $1\!-\!\delta$, we have that:
\begin{equation*}
    \frac{1}{T_G}\sum\nolimits_{k=1}^{T_G}\NEgap(\theta^{(k)})^2\le\epsilon^2.
\end{equation*}
That is, {with a proper choice of stepsize, e.g., $\eta\!=\!\frac{(1-\gamma)^3}{4\sum_i\!|\cA_i|}$}, the algorithm can find an $\epsilon$-NE with probability at least $1-\delta$ with
\begin{equation}\label{eq:sample-complexity-simplified}
    T_JT_G \sim \tilde{O}\left(\frac{n}{\epsilon^6} \textup{poly}\left(\frac{1}{1-\gamma}, |\cS|, \max_i|\cA_i|\right)\right)\vspace{-10pt}
\end{equation}
samples, where $\tilde{O}$ hides log factors.
\end{theorem}
{We would like to first compare our result with one related work with sample complexity of learning MPGs \cite{leonardos21}. Interestingly, both
sample complexities are $O(
1/\epsilon^6)$. It is an interesting question to study whether such dependence is fundamental
or not for learning with simultaneous-updating agents.  Yet \cite{leonardos21} considers Monte Carlo, model-free gradient estimation, while our algorithm takes a model-based approach which suffers less from high variance and the notion of `averaged' MDP can potentially be extrapolated to other settings.}

\noindent\textbf{Proof Sketch:} The proof of Theorem \ref{thm:main} consists of three major steps. The first step is to bound the estimation error of parameters $\overline{P_i^{\theta}},\overline{r_i^{\theta}},\overline{P_\cS^\theta}$ of the `averaged'-MDP. This step leverages Assumption \ref{assump:like-ergodicity} and Azuma-Hoeffding inequality 
to get high probability bounds for the parameters. The second step translates the estimation error of the `averaged'-MDP into the gradient estimation error. Then, the third step treats gradient estimation step as an oracle that gives biased gradient information, where the bias is the estimation error. The final result is obtained by analyzing the performance of biased projected gradient ascent algorithm. The detailed proofs are provided in Appendix \ref{apdx:proof-main-sample-based}.

\noindent\textbf{Comparison with centralized learning:} The best known sample complexity bound for single-agent/centralized MDP is $\tilde{O}\left(\frac{|\cS||\cA|}{(1-\gamma)^3\epsilon^2}\right)$ \cite{Sidford18}. Compared with \eqref{eq:sample-complexity-simplified}, the centralized bound scales better with respect to $\epsilon, |\cS|, |\cA_i|, \frac{1}{1-\gamma}$. However, as argued in the previous subsection, the total action space $|\cA| = \prod_{i=1}^n |\cA_i|$ in the centralized bound scales exponentially with the number of agent $n$, while our complexity bound only scales linearly. Here, we briefly state the fundamental difficulties of learning in the SG setting compared with centralized learning, which also explains why our bound scales worse with respect to the factors $\epsilon, |\cS|, |\cA_i|, \frac{1}{1-\gamma}$. 1) Firstly, the optimization landscape in the SG setting is more complicated. For centralized learning, the gradient domination property is stronger and accelerated gradient methods (e.g. via natural policy gradient or entropy regularization) can speed up the convergence of exact gradient from $O(\frac{1}{\epsilon^2})$ to $O(\frac{1}{\epsilon})$ \cite{agarwal2020}, or even $O(\log(\frac{1}{\epsilon}))$ \cite{Mei20}. In contrast, for multi-agent settings, due to the more complicated optimization landscape, these methods can no longer improve the dependency on $\epsilon$, and thus makes the outer loop complexity $T_G$ larger. 2) Secondly, the behavior of other agents makes the environment non-stationary, i.e., the averaged Q-function $\overline{Q_i^\theta}$ as well as the averaged transition probability distribution $\overline{P_i^\theta}$ depends on the policy of other agent $\theta_{-i}$. Thus, unlike centralized learning, where the state transition probability matrix can be estimated in an off-policy or even offline manner, i.e. using data samples from different policies, $\overline{P_i^\theta}$ can only be estimated in a online manner, using samples generated by exactly the same policy $\theta$, which increases the inner loop complexity $T_J$. 3) Thirdly, the complicated interactions amongst agents necessitate more care during the learning process. Algorithms designed for centralized learning that achieve near-optimal sample complexity are generally Q-learning type algorithms. However, in SGs, it can be shown that having each agent maximize its own averaged Q-function may actually lead to non-convergent behavior. 
Thus, we need to consider algorithms that update in a less aggressive manner, e.g. soft Q-learning, or policy gradient (which is considered in this paper). 

\section{Numerical simulations}\label{apdx:numerics}
\begin{table}[htbp]
\begin{minipage}{.4\linewidth}
\centering
\begin{tabular}{|c||c|c|}
\hline
 &$a_2 = 1$&$a_2 = 2$\\
\hline\hline
$a_1 = 1$ & (-1,-1)& (-3,0)\\
\hline
$a_1 = 2$ &(0,-3)& (-2,-2)\\
\hline
\end{tabular}
\caption{Game 1: Reward}
    \label{tab:reward_table}
\end{minipage}
\hspace{10pt}
\begin{minipage}{.5\linewidth}
\centering
\begin{tabular}{|c||c|c|}
\hline
 &$s_2 = 1$&$s_2 = 2$\\
\hline\hline
$s_1 = 1$ & 2&0\\
\hline
$s_1 = 2$ &0&1\\
\hline
\end{tabular}
\caption{Game 2: Reward}
    \label{tab:reward_table-coordination-game}
\end{minipage}
\vspace{-10pt}
\end{table}
This section studies three numerical examples to corroborate our theoretical results. The multi-stage prisoner's dilemma (Game 1) confirms the local stability results for general-sum SGs; the coordination game (Game 2) considers local stability as well as convergence rate of exact gradient play for MPG; the state-based coordination game (Game 3) tests the performance of the sample-based algorithm proposed in Section \ref{subsec:sample-based-MPG}.
\subsection{Game 1: multi-stage prisoner's dilemma}
The first example --- multi-stage prisoner's dilemma model\cite{Arslan16} --- studies exact gradient play for general SGs. It is a $2$-agent SG, with $\cS = \cA_1 = \cA_2 = \{1,2\}$. The reward for each agent $r_i(s, a_1, a_2), ~i\in\{1,2\}$ is independent of state $s$ and is given by Table \ref{tab:reward_table}. The state transition probability is determined by agents' previous actions:
\vspace{-2pt}
\begin{align*}
    P(s_{t+1}=1|(a_{1,t}, a_{2,t})=(1,1)) &=  1-\epsilon, \\
    P(s_{t+1}=1|(a_{1,t}, a_{2,t})\neq (1,1)) &=  \epsilon.
\end{align*}
Here action $a_i = 1$ means that agent $i$ choose to \textit{cooperate} and $a_i=2$ means \textit{betray}. The state $s$ serves as a noisy indicator, with error rate $\epsilon$, of whether both agents cooperated ($s_t = 1$) or not ($s_t=2$) in the previous stage $t-1$. 
\begin{table}[htbp]
\centering
\begin{minipage}{.4\linewidth}
\vspace{-12pt}
\includegraphics[width=\linewidth]{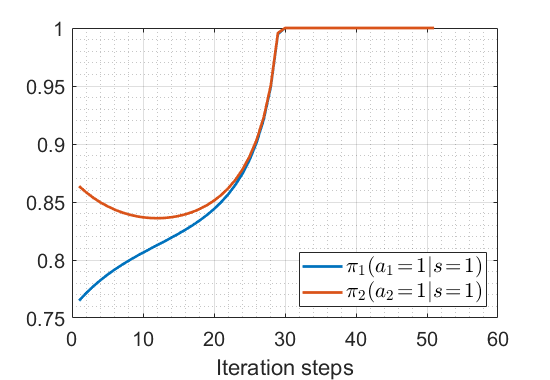}
\captionof{figure}{(Game 1:) Convergence to the cooperative NE}
\label{fig:prisoner_dilemma_NE_convergence}
\end{minipage}
\hspace{10pt}
\begin{minipage}{.5\linewidth}
\centering
        \begin{tabular}{|c|c|c|}
    \hline
        $\epsilon$ &$\Delta^{\!\theta^*}\!(s\!=\!1)$& ratio \%\\
        \hline
        0.1&433.3&(47.8$\pm$ 5.1)\%\\
        \hline
        0.05&979.3&(66.3$\pm$ 4.3)\%\\
        \hline
        0.01&2498.6&(77.4$\pm$ 2.8)\%\\

    \hline
    \end{tabular}
    \vspace{-5pt}
    \caption{(Game 1:) Relationship of $\Delta^{\!\theta^*}\!(s\!=\!1)$, convergence ratio, and $\epsilon$.  $\Delta^{\!\theta^*}\!(s\!=\!1)$ is calculated using \eqref{eq:delta}. Convergence ratio is calculated by $\frac{\#\textup{Trials that converge to }\theta^*}{\#\textup{Total number of trials}}$.
    }
    \label{tab:ep}
\end{minipage}
\begin{minipage}{.3\linewidth}
\centering
 \includegraphics[width=\linewidth]{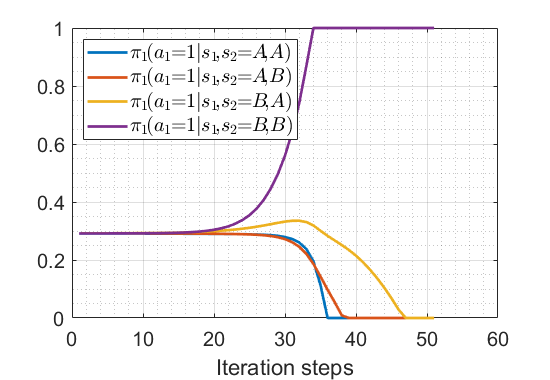}
    \captionof{figure}{(Game 2:) Starting from a close neighborhood of a fully mixed NE}
    \label{fig:coorporation_game_saddle_point}
\end{minipage}
\hfill
\begin{minipage}{.3\linewidth}
\vspace{-7pt}
\centering
 \includegraphics[width=\linewidth]{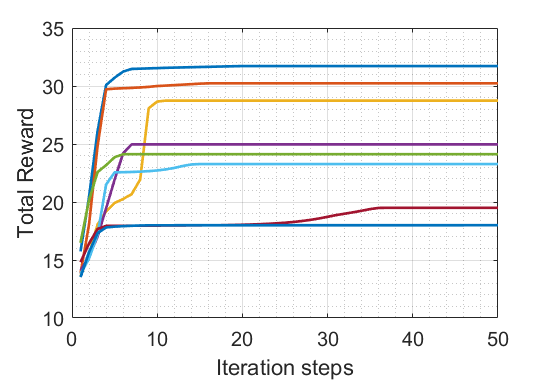}
    \captionof{figure}{(Game 2:) Total reward for multiple runs}
    \label{fig:coorporation_game_traj_r}
\end{minipage}
\hfill
\begin{minipage}{.3\linewidth}
\vspace{-7pt}
\centering
 \includegraphics[width=\linewidth]{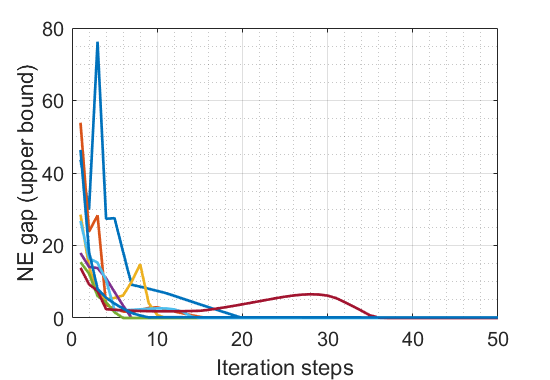}
    \captionof{figure}{(Game 2:) NE-gap for multiple runs}
    \label{fig:coorporation_game_traj_gap}
\end{minipage}
\end{table}

The single-stage game corresponds to the famous \emph{Prisoner’s Dilemma}, and it is well-known that there is a unique NE $(a_1,a_2) = (2,2)$, where both agent decide to betray. The dilemma arises from the fact that there exists a joint non-NE strategy $(1,1)$ such that both players obtain a higher reward than what they get under the NE. However, in the multi-stage case, the introduction of an additional state $s$ allows agents to make decisions based on whether they have cooperated before. It turns out that cooperation can be achieved given that the discounting $\gamma$ is close to $1$ and that the indicator for $s$ is accurate enough, i.e. $\epsilon$ is close to $0$. Apart from the fully betray strategy, where both agents will betray regardless of $s$, there is another strict NE $\theta^*$ that is
~$\theta^{*}_{s=1,a_i=1} = 1,~ \theta^{*}_{s=2,a_i=1} = 0,$~
where agents will cooperate given that they have cooperated in previous stage, and betray otherwise.

We simulate gradient play for this model and mainly focus on the convergence to the cooperative equilibrium $\theta^*$. We fix $\gamma=0.95$. The initial policy is set as: $\theta^{(0)}_{s=1,a_i=1} \!=\! 1\!-\!0.4\delta_i,~~ \theta^{(0)}_{s=2,a_i=1} \!=\! 0$, where the $\delta_i$'s are uniformly sampled from $[0,1]$. The initialization implies that at the beginning, both agents are willing to cooperate to some extent given that they cooperated at the previous stage. Figure \ref{fig:prisoner_dilemma_NE_convergence} shows a trial converging to the NE starting from a randomly initialized policy. Then we study the size of attraction region for $\theta^*$ and how it varies with the indicator's error rate $\epsilon$, which is shown in Table~\ref{tab:ep}. The size of the region of attraction for $\theta^*$ can be reflected by the ratio of convergence ( $\frac{\#\textup{Trials that converge to }\theta^*}{\#\textup{Total number of trials}}$) for multiple trials with different initial points. Here we calculate one ratio using 100 trials and the mean and standard deviation (std) are calculated by computing the ratio 10 times using different trials. An empirical estimate of the volume of the region is the convergence ratio times the volume of the uniform sampling area; hence the larger the ratio, the larger the region of attraction.
 Intuitively speaking, the more accurately the state $s$ represents the cooperation situation of the agents, the less incentive agents will have for betraying when observing $s=1$, that is, the larger $\Delta^{\!\theta^*}\!(s\!=\!1)$ will become, and thus the larger the convergence ratio will be. This intuition matches the simulation result as well as the theoretical guarantees on the local convergence around a strict NE in Theorem \ref{theorem:local-convergence-general-game}.

\vspace{-9pt}
\subsection{Game 2: coordination game}
Our second example is based on coordination game \cite{russell1998coordination}. It is an identical-reward game which is one special class of Markov potential game. Consider a $2$-agent identical reward coordination game problem with state space $\cS = \cS_1\times\cS_2$ and action space $\cA = \cA_1\times\cA_2$, where $\cS_1 = \cS_2 = \cA_1 = \cA_2 = \{1,2\}$. The state transition probability is given by:
\vspace{-2pt}
\begin{align*}
    P(s_{i,t+1} = 1|a_{i,t}=1) = 1-\epsilon,~ P(s_{i,t+1} = 1|a_{i,t}=2) = \epsilon,
\end{align*}
where $i = 1,2$. The reward table is given by Table \ref{tab:reward_table-coordination-game}. Here we can view the actions $\left\{1,2\right\}$ as two different social networks that agents can choose. They are rewarded only if they are in the same network. Network $1$ has a higher reward than network $2$. The state $s_i$ stands for the network that agent $i$ is really at after taking an action. $\epsilon$ stands for the randomness in reaching a network after taking the action. 

There is at least one fully-mixed NE where both agents join network $1$ with probability $\frac{1-3\epsilon}{3(1-2\epsilon)}$ regardless of the current occupancy of networks, and there are $13$ different strict NEs that can be verified numerically. 
Figure \ref{fig:coorporation_game_saddle_point} shows a gradient play trajectory whose initial point lies in a close neighborhood of the mixed NE. As the algorithm progresses, we see that the trajectory in Figure \ref{fig:coorporation_game_saddle_point} diverges from the mixed-NE, indicating that the fully-mixed NE is indeed a saddle point. This corroborates our finding in Theorem \ref{theorem:MPG-local-maximum-saddle-points}. 
Figure \ref{fig:coorporation_game_traj_r} shows the evolution of total reward $J(\theta^{(t)})$ for gradient play for different random initial points $\theta^{(0)}$. Different initial points converge to one of $13$ different strict NEs each with a different total reward (some strict NEs with relatively small region of attraction are omitted in the figure).  
While the total rewards are different, as shown in Figure \ref{fig:coorporation_game_traj_gap}, we see that the NE-gap of each trajectory (corresponding to same initial points in Figure \ref{fig:coorporation_game_traj_r}) converges to $0$. This suggests that the algorithm is indeed able to converge to a NE. Notice that the NE-gaps do not decrease monotonically.

\vspace{-5pt}
\subsection{Game 3: state-based coordination game}
\begin{table}[htbp]
\begin{minipage}{.65\linewidth}
\centering
 \includegraphics[width=.35\linewidth]{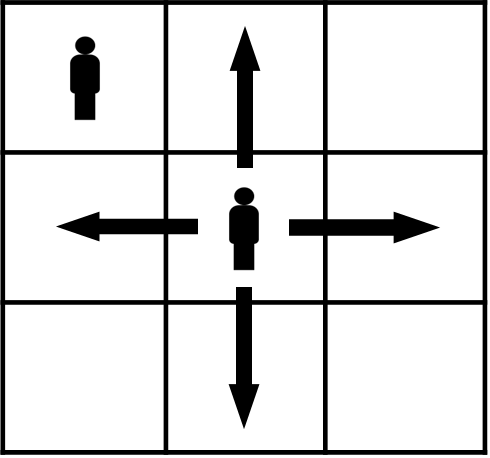}
 \hspace{20pt}
 \includegraphics[width=.365\linewidth]{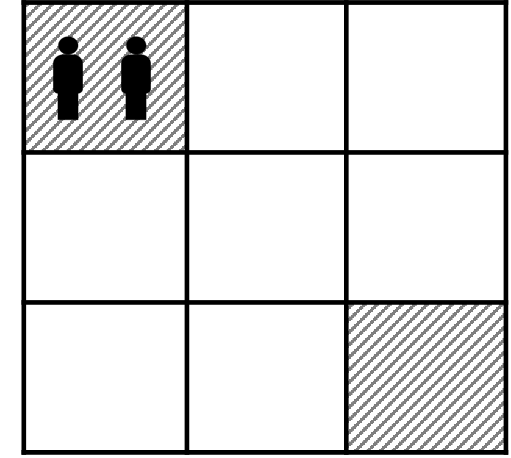}
\captionof{figure}{(Game 3:) State-based coordination game, rewards are nonzero if both players locate at the same shaded grids}
\label{fig:state-based-coordination-game}
\end{minipage}
\hfill
\begin{minipage}{.3\linewidth}
\centering
 \hspace{-14pt}\includegraphics[width=.9\linewidth]{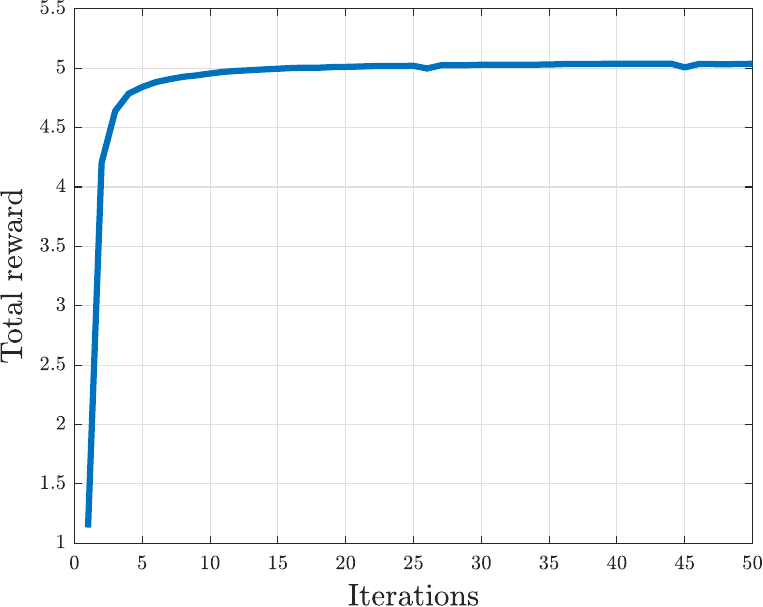}
    \vspace{-5pt}
    \captionof{figure}{Total reward $J(\theta^{(t)})$ keeps increasing}
    \label{fig:total-reward}
\end{minipage}
\begin{minipage}{.65\linewidth}
\centering
\includegraphics[width=\linewidth]{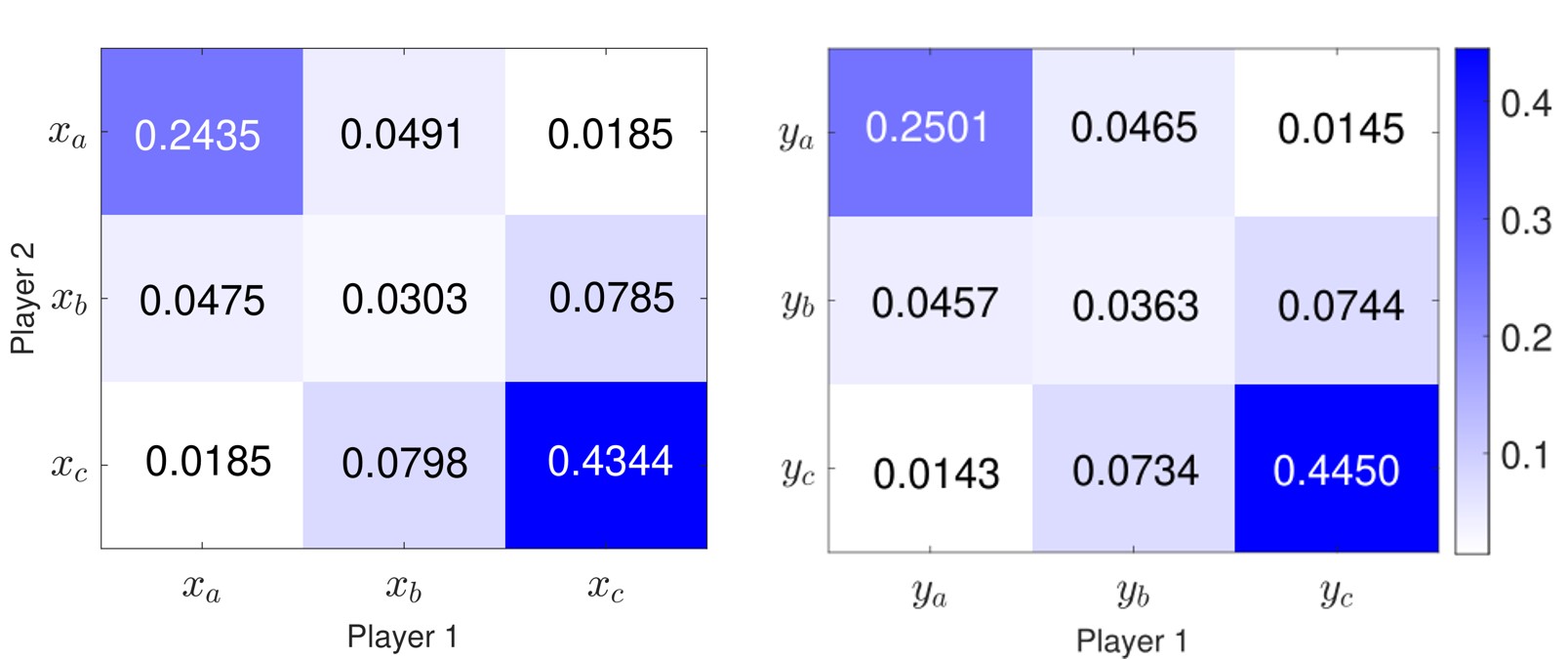}
\captionof{figure}{Marginal distribution $d_\theta^x(s_{1,x}, s_{2,x})$ and $d_\theta^y(s_{1,y}, s_{2,y})$}
\label{fig:d_theta_xy}
\end{minipage}
\hfill
\begin{minipage}{.33\linewidth}
\vspace{5pt}
\centering
\includegraphics[width=.8\linewidth]{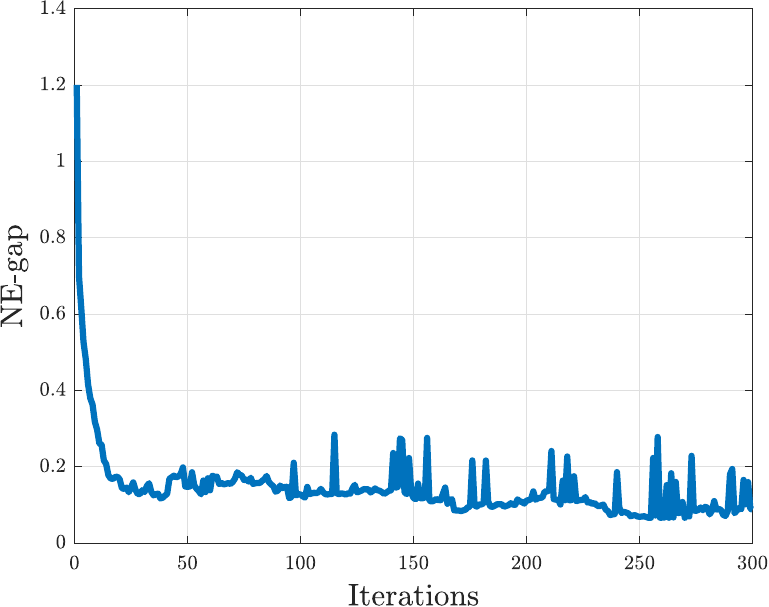}
\captionof{figure}{$\NEgap$ converges to a value close to zero. Here the $\NEgap$ is measured by $\max_i\max_{(s,a_i) \overline{A_i^\theta}(s,a_i)}$}
\label{fig:NE-gap}
\end{minipage}
\vspace{-5pt}
\end{table}

Our third numerical example studies the empirical performance of the sample-based learning algorithm, Algorithm~\ref{alg:sample-based learning}. Here we consider a generalization of coordination game (Game 2) where the two players now try to coordinate on a 2D grid. The two-player state-based coordination game on a $3\times 3$ grid is defined as follows: the state space is given by $\cS = \cS_1\times\cS_2, ~\cS_1=\cS_2 = \cS_x \times \cS_y = \{x_a, x_b, x_c\}\times\{y_a,y_b,y_c\}$, action space is given by $\cA = \cA_1\times\cA_2,~\cA_1 = \cA_2 = \{\stay, \lleft, \rright, \up, \down\}$, i.e., agent can choose to stay at current grid or move left/right/up/down to its neighboring grids. We assume that there is random noise during the transition, where the agent might end up in a neighboring grid of the target location with error probability $\epsilon$.
 The reward is given by:
\begin{equation*}
    r(s_1, s_2) = \mathbf{1}\{s_1 = s_2 = \{x_a,y_a\} \textup{ or } \{x_c, y_c\}\},
\end{equation*}
i.e. the two agents are only rewarded if they stay at the upper-left or lower-right corner at the same time.

For numerical simulation, we take $T_G = 300, T_J = 10000, \alpha = 0.1, \eta = 10, \epsilon = 0.1$; the numerical results are as displayed in Figure \ref{fig:total-reward} - \ref{fig:NE-gap}. Figure \ref{fig:total-reward} shows that total reward increases as the number of iterations increase, and Figure \ref{fig:NE-gap} shows that the NE-gap converges to a value close to zero. However, because we project the policy to the $\alpha$-greedy set $\cX^\alpha$, the NE-gap cannot converge to exactly zero. Figure \ref{fig:d_theta_xy} visualizes the discounted state visitation distribution. To make the visualization more intuitive, we look at the marginalized discounted state visitation distribution $d_\theta^x$ defined below:
\begin{align*}
\textstyle 
    d_\theta^x(s_{1,x}, s_{2,x}) \!= \!\sum_{s_{1,y},s_{2,y}} \hspace{-5pt}d_\theta(s_{1,x},s_{1,y}, s_{2,x},s_{2,y}).
\end{align*}
$d_\theta^y$ is defined similarly. From Figure \ref{fig:d_theta_xy} we can see that most of the probability measure concentrates on $\{(x_a, x_a), (x_c,x_c)\}$, $\{(y_a,y_a), (y_c, y_c)\}$, indicating that the two agents are able to coordinate most of the time.

\vspace{-5pt}
\section{Conclusion and Discussion}

This paper studies the optimization landscape and convergence of gradient play for SGs. For general SGs, we establish local convergence for strict NEs. For MPGs, we establish the global convergence with respect to NE gap and the local stability results for strict NEs as well as fully mixed ones. A sample-based NE-learning algorithm with sample complexity guarantee is also proposed under this setting. There are many interesting future directions. Firstly, the current assumption of MPGs is relatively strong compared with the notion of potential games in the one-shot setup, which might restrict its application to broader settings. More effort would be needed to identify other special types of SGs that facilitate efficient learning. It would also be meaningful to investigate real-life applications, such as dynamic congestion and routing. Secondly, other sample-based learning methods, such as actor-critic, natural policy gradient, Gauss-Newton methods, could also be considered, which might improve the sample complexity.  

\vspace{-5pt}

\bibliography{bib.bib}
\bibliographystyle{IEEEtran}
\vspace{-5pt}
\appendix

\section{More about Markov potential games}\label{apdx:More-about-MPG}
This section is dedicated to a more thorough understanding of Markov potential games, which includes necessary or sufficient conditions for MPG and a few (counter)examples.

\textbf{B.1. A necessary condition and counterexamples}
\begin{defi}(\cite{Monderer96})
Define the path in the parameter space as~$\tau = (\theta^{(0)},\theta^{(1)},\dots,\theta^{(N)})$, where $\theta^{t}, \theta^{t+1}$ differ only one component $i_t$, i.e. $\theta^{(t+1)} = (\theta^{(t+1)}_{i_t}, \theta^{(t)}_{-i_t} )$. A closed path is a path such that $\theta^{(0)} = \theta^{(N)}$. Define:
\begin{equation*}
    I(\tau) := \sum_{t=1}^NJ_{i_t}(\theta^{(t)}) -J_{i_t}(\theta^{(t-1)})
\end{equation*}
\end{defi}
The following theorem is a direct generalization of Theorem 2.8 in \cite{Monderer96} to MPG setting:
\begin{lemma}\label{lem:neccessary-condition-MPG}
For Markov potential games, $I(\tau)=0$ for any finite closed path $\tau$.
\end{lemma}
\begin{proof}
The proof is quite straightforward from the definition of MPG
\begin{align*}
    I(\tau) &= \sum_{t=1}^NJ_{i_t}(\theta^{(t)}) -J_{i_t}(\theta^{(t-1)})\\
    &= \sum_{t=1}^N\Phi(\theta^{(t)}) -\Phi(\theta^{(t-1)})\\
    &= \Phi(\theta^{(N)}) - \Phi(\theta^{(0)})\\
    & = 0.
\end{align*}
\end{proof}
Although the proof of Lemma \ref{lem:neccessary-condition-MPG} is straightforward, it serves as a useful tool in proving that a game is not a MPG. For example, applying the theorem we can get that the following conditions are not sufficient conditions for MPG.
\begin{prop}\label{prop:counterexample-MPG}
None of the following conditions on a SG necessarily imply that it is a MPG:
\begin{enumerate}[label=(\arabic*)]
    \item There exists $\phi(s,a)$ such that at each $s$, $\ r_i(s,a_i',a_{-i})-r_i(s,a_i,a_{-i})=\phi(s,a_i',a_{-i})-\phi(s,a_i,a_{-i})$;
    \item There exists $\phi(s,a)$ such that for every $s,\hat{s}$, $\ r_i(s,a_i',a_{-i})-r_i(\hat{s},a_i,a_{-i})=\phi(s,a_i',a_{-i})-\phi(\hat{s},a_i,a_{-i})$;
    \item Rewards $r_i$ are independent of $s$, and they have a potential function, i.e., 
$r_i(a_i,a_{-i})-r_i(a_i',a_{-i})=\phi(a_i,a_{-i})-\phi(a_i',a_{-i})$.
\end{enumerate}
\end{prop}
\begin{proof}(of Proposition \ref{prop:counterexample-MPG})
A simple counterexample showing that the conditions in \eqref{prop:counterexample-MPG} are not sufficient is the multi-stage prisoner's dilemma (Game 1) introduced in the numerical section (Appendix \ref{apdx:numerics}). Since the reward table for multi-stage prisoner's dilemma is the same as the one-shot prisoner's dilemma (which is known to be a potential game), Game 1 satisfies condition (3) in Proposition \ref{prop:counterexample-MPG}, which implies condition (2), which in turn implies condition (1). In the following we are going to use Lemma \ref{lem:neccessary-condition-MPG} to show that Game 1 is not a MPG. We define the following individual policies:
\begin{align*}
 \theta_i^{\text{Defect}}: \qquad \theta^{\text{Defect}}_{s=1,a_i=1} &= 0, ~~\theta^{\text{Defect}}_{s=2,a_i=1} = 0\\
 \theta_i^{\text{Coop}}: \qquad
 \theta^{\text{Coop}}_{s=1,a_i=1} &= 1,~~
 \theta^{\text{Coop}}_{s=2,a_i=1} = 0\\
 \theta_i^{\text{Always\_coop}}: \qquad
 \theta^{\text{Always\_coop}}_{s=1,a_i=1} &= 1,~~
 \theta^{\text{Always\_coop}}_{s=2,a_i=1} = 1
\end{align*}
Let:
\begin{align*}
    &\theta^{(0)} = (\theta_1^{\text{Defect}}, \theta_2^{\text{Coop}}),\quad
    &&\theta^{(1)} = (\theta_1^{\text{Coop}}, \theta_2^{\text{Coop}}),\\
    &\theta^{(2)} = (\theta_1^{\text{Coop}}, \theta_2^{\text{Always\_coop}}),\quad
    &&\theta^{(3)} = (\theta_1^{\text{Defect}}, \theta_2^{\text{Always\_coop}})
\end{align*}
and define a path $\tau$ by:
\begin{equation*}
\tau = (\theta^{(0)}, \theta^{(1)}, \theta^{(2)}, \theta^{(3)}, \theta^{(4)}), ~~\theta^{(4)} = \theta^{(0)}.
\end{equation*}
For the sake of easy calculation, we set $\epsilon = 0$ and set initial state as $s_0=1$ in Game 1. In this example, it is not hard to see that $J_{i_t}(\theta^{(t)}) - J_{i_t}(\theta^{(t-1)}) > 0$ for each $t \in \{1,2,3,4\}$, implying that $I(\tau) > 0$. This indicates that although Game 1 satisfies condition (3) (as well as conditions (1) and (2)), it is still not a MPG. 
\end{proof}


\textbf{B.2. A sufficient condition}

Proposition \ref{prop:counterexample-MPG} suggests that MPG is a quite restrictive assumption. Even if the reward table for a SG is the same as the reward table of a one-shot potential game, the SG may still not be a MPG. Nevertheless, we can show that the following condition is sufficient for a stochastic game to be a MPG:
\begin{lemma}\label{lem:SG-MPG}
A stochastic game is a MPG if condition (1) in Proposition \ref{prop:counterexample-MPG} is satisfied and that $P(s'|s,a) = P(s'|s)$.
\end{lemma}
\begin{proof}
$P(s'|s,a) = P(s'|s)$ implies that the discounted state visitation distribution $d_\theta$ does not depend on $\theta$, and thus we denote it as $d(s)$ instead.
Condition (1) implies that $\phi(s,a_i,a_{-i}) - r_i(s,a_i,a_{-i})$ only depends on $s$ and $a_{-i}$ but not $a_i$, and so we denote the difference as $\delta_i(s,a_{-i})$, i.e., 
$$\phi(s,a_i,a_{-i}) - r_i(s,a_i,a_{-i}) = \delta_i(s,a_{-i}).$$
The total reward of agent $i$ can be written as:
\begin{align*}
    J_i(\theta) &= \sum_sd(s)\sum_{a}\pi_\theta(a|s)r_i(s,a)\\
    &=\sum_sd(s)\sum_{a_i}\pi_{\theta_i}(a_i|s)\sum_{a_{-i}}\pi_{\theta_{-i}}(a_{-i}|s)r(s,a_i,a_{-i})
\end{align*}
Similarly, total potential function can be written as:
\begin{align*}
    \Phi(\theta) &= \sum_sd(s)\sum_{a}\pi_\theta(a|s)\phi(s,a)\\
    &=\sum_sd(s)\sum_{a_i}\pi_{\theta_i}(a_i|s)\sum_{a_{-i}}\pi_{\theta_{-i}}(a_{-i}|s)\phi(s,a_i,a_{-i})
\end{align*}
Thus,
\begin{align*}
     \Phi(\theta) -J_i(\theta) &= \sum_sd(s)\sum_{a_i}\pi_{\theta_i}(a_i|s)\sum_{a_{-i}}\pi_{\theta_{-i}}(a_{-i}|s)\left(\phi(s,a_i,a_{-i})-r(s,a_i,a_{-i})\right)\\
     &=\sum_sd(s)\sum_{a_i}\pi_{\theta_i}(a_i|s)\sum_{a_{-i}}\pi_{\theta_{-i}}(a_{-i}|s)\delta_i(s, a_{-i})\\
     &= \sum_sd(s)\sum_{a_{-i}}\pi_{\theta_{-i}}(a_{-i}|s)\delta_i(s, a_{-i}),
\end{align*}
which does not depend on parameter $\theta_i$, i.e.,
\begin{equation*}
    \Phi(\theta_i',\theta_{-i}) -J_i(\theta_i',\theta_{-i}) = \Phi(\theta_i,\theta_{-i}) -J_i(\theta_i,\theta_{-i}), \quad \forall (\theta_i',\theta_{-i}), (\theta_i,\theta_{-i}) \in \mathcal{X},
\end{equation*}
which completes the proof.
\end{proof}

\textbf{B.3. MPG with local states and an application example}

From Proposition~\ref{prop:counterexample-MPG}, we see that it is difficult for a SG to be a MPG even if the game is a potential game at each state. Lemma~\ref{lem:SG-MPG} only presents a very special case where the action does not affect the state, meaning that this MPG is merely a collection of potential games. To provide a MPG which is beyond the identical-interest case and the case in Lemma~\ref{lem:SG-MPG}, inspired by the setting in \cite{qu2020} and \cite{Macua18}, here we consider a special multi-agent setting where $\cS=\cS_1\times\dots\times\cS_n$ and $\cS_i$ is the local state space of agent $i$. In addition, the transition probability takes the decomposed form, $P(s'|s,a) = \prod_{i=1}^nP(s_i'|s_i,a_i)$. The rest of the SG setting is the same as the SG in Section~\ref{sec:Problem-setup}. 
Deviating slightly from the main text, we consider the \textit{localized policy} where each agent take actions based on its own state, 
\begin{equation*}
    \pi_{\theta}(a_t|s_t) = \prod_{i=1}^n\pi_{\theta_i}(a_{i,t}|s_{i,t})
\end{equation*}
with the \textit{localized direct parameterization}:
\begin{equation*}
    \pi_{\theta_i}(a_{i,t}|s_{i,t}) = \theta_{(s_i,a_i)}, \quad \theta_i \in \Delta(\cA_i)^{|\cS_i|}
\end{equation*}
 We use $\cX_i^{\textup{local}}:=\Delta(\cA_i)^{|\cS_i|}$ to denote the feasibility region of $\theta_i$, and the feasibility region of $\theta$ is denoted as $\cX^{\textup{local}}:= \cX_i^{\textup{local}}\times\dots\times\cX_n^{\textup{local}}$.


\begin{lemma}\label{thm:MPG-mean-field}
If there is a function $\phi(s,a)$ such that for every agent $i$, $r_i(s_i,s_{-i},a_i,a_{-i})=\phi(s_i,s_{-i},a_i,a_{-i})+\psi_{i}(s_{-i},a_{-i})$ where $\psi_i$ only depends on $s_{-i},a_{-i}$, then this SG is a MPG, i.e., for any parameters $(\theta_i',\theta_{-i}), (\theta_i,\theta_{-i})\in \cX^{\textup{local}}$, the equation in Definition \ref{defi:MPG} is satisfied.
\end{lemma}
The proof is straightforward given the local structure of the MDP and the localized policies. This MPG enjoys nontrivial multi-agent application examples such as medium access control \cite{Macua18}, dynamic congestion control \cite{bertrand20}, etc. 
Below we provide medium access control as one of the examples. 

\textit{Real application - medium access control.}
We consider the discretized version of the dynamic medium access control game introduced in \cite{Macua18}, where each agent is a user that tries to transmit data via a single transmission medium by injecting power to the wireless network. Each user's goal is to maximize its data rate and battery lifespan. If multiple users transmit at the same time, they will interfere with each other and decrease their data rate. Here user $i$'s state is $s_i \in \cS_i = \{0,1,\dots, B_{i,\max}\}$, which denotes its own battery level, where $B_{i,\max}$ is its initial battery level. We use $\delta_i$ to denote its discharging factor. Its action $a_i\in\cA_i=\{0,1,\dots,P_{i,\max}\}$ denotes the power injected to the network at each time step, where $P_{i,\max}$ is the maximum allowed power. The state transition is deterministic, describing the discharging process of the battery proportional to the transmission power, which is given by:
$$s_{i,t+1} = s_{i,t} - \delta_i a_{i,t}.$$
The stage reward of user $i$ is given by:
\begin{equation*}
    r_i(s,a) = \log\left(1+\frac{|h_i|^2a_i}{1+\sum_{j\neq i}|h_j|^2a_j}\right) + \alpha s_i,
\end{equation*}
where $h_i$ is the random fading channel coefficient for user $i$.

 By noticing that $r_i(s,a)=\log\left(1+\sum_{i=1}^n|h_i|^2a_i\right) + \alpha \sum_{i} s_i-\log\left(1+\sum_{j\neq i}|h_j|^2a_j\right) - \alpha \sum_{j\neq i}s_j$, we can apply Lemma \ref{thm:MPG-mean-field} to verify that the medium access control problem is indeed a MPG and that the potential function $\phi$ is given as:
\begin{equation*}
    \phi(s,a) = \log\left(1+\sum_{i=1}^n|h_i|^2a_i\right) + \alpha \sum_{i=1}^ns_i.
\end{equation*}

\section{Proof of Lemma \ref{lemma:averaged-performance-difference-lemma}}\label{apdx:proof-performance-difference-lemma}
\begin{proof}
According to the performance difference lemma (c.f. \cite{Kakade02}), let $\theta':= (\theta_i', \theta_{-i})$,
\begin{talign*}
        &J_i(\theta_i', \theta_{-i})- J_i(\theta_i, \theta_{-i}) = \frac{1}{1-\gamma} \sum_{s,a} d_{\theta'}(s) \pi_{\theta'}(a|s) A_i^{\theta}(s, a)\\
    &=\frac{1}{1-\gamma} \sum_{s,a_i} d_{\theta'}(s) \pi_{\theta'_i}(a_i|s)\sum_{a_{-i}} \pi_{\theta_{-i}}(a_{-i}|s)A_i^{\theta}(s, a_i,a_{-i})\\
     &=\frac{1}{1-\gamma} \sum_{s,a_i} d_{\theta'}(s) \pi_{\theta'_i}(a_i|s)\overline{A_i^{\theta}}(s, a_i).
\end{talign*}
\end{proof}
\section{Derivation of \eqref{eq:d-hat} (calculation of $d_\theta$)}\label{apdx:d-theta}
From the definition of $d_\theta$:
\begin{equation*}
    d_\theta(s) = \bE_{s_0\sim\rho} (1-\gamma) \sum_{t=0}^\infty\gamma^t \Prtheta(s_t=s|s_0),\vspace{-5pt}
\end{equation*}
we have that:
\begin{align*}
    d_\theta &= (1-\gamma) \left(\rho+\gamma\overline{P_S^\theta}^\top\rho + \gamma^2\overline{P_S^\theta}^{2^\top}\rho + \cdots\right)\\
    &= (1-\gamma)(I + \gamma\overline{P_S^\theta} + \gamma^2\overline{P_S^\theta}^2\cdots)^\top\rho\\
    &= (1-\gamma)(I-\gamma\overline{P_S^\theta})^{-\top}\rho
\end{align*}

\section{Proof of Lemma \ref{lemma:policy-gradient-direct-parameterization}}\label{apdx:proof-policy-gradient}
\begin{proof}
According to policy gradient theorem \eqref{eq:policy-gradient-theorem}:
\begin{align*}
    \frac{\partial J_i(\theta)}{\partial{\theta_{s,a_i}}} &= \frac{1}{1-\gamma}\sum_{s'} \sum_{a'} d_{\theta}(s') \pi_\theta(a'|s') \frac{\partial \log\pi_\theta(a'|s')}{\partial \theta_{s,a_i}}Q_i^\theta(s,a)
\end{align*}
Since for direct parameterization:
\begin{align*}
    \frac{\partial \log\pi_\theta(a'|s')}{\partial \theta_{s,a_i}} =  \frac{\partial \log\pi_{\theta_i}(a_i'|s')}{\partial \theta_{s,a_i}} &= \mathbf{1}\{a_i' = a_i, s' = s\}\frac{1}{\theta_{s,a_i}}\\
    &= \mathbf{1}\{a_i' = a_i, s' = s\}\frac{1}{\pi_{\theta_i}(a_i|s)}
\end{align*}
Thus we have that:
\begin{equation*}
\begin{split}
      \frac{\partial J_i(\theta)}{\partial{\theta_{s,a_i}}} &= \frac{1}{1-\gamma}\sum_{s'} \sum_{a'} d_{\theta}(s') \pi_\theta(a'|s') \mathbf{1}\{a_i' = a_i, s' = s\}\frac{1}{\pi_\theta(a_i|s)} Q_i^\theta(s,a)\\
    &= \frac{1}{1-\gamma}\sum_{a_{-i}'}d_\theta (s) \pi_{\theta_i}(a_i|s)\pi_{\theta_{-i}}(a_{-i}'|s)\frac{1}{\pi_\theta(a_i|s)} Q_i^\theta(s,a_i, a_{-i}')\\
    &=\frac{1}{1-\gamma}d_\theta(s) \sum_{a_{-i}'}\pi_{\theta_{-i}}(a_{-i}'|s)Q_i^\theta(s,a_i, a_{-i}')\\
    &= \frac{1}{1-\gamma}d_{\theta}(s)  \overline {Q_i^{\theta}}(s, a_i)
\end{split}
\end{equation*}
\end{proof}

\section{Proof of Theorem \ref{theorem:local-convergence-general-game} and Lemma \ref{lemma:strict-NE-gap}} \label{apdx:local-convergence-general-game}
\begin{proof}(Lemma \ref{lemma:strict-NE-gap})
For a given strict NE $\theta^*$ randomly set:
$$a_i^*(s) \in \argmax_{a_i} \overline{A_{i}^{\theta^*}}(s,a_i),$$
and set $\theta_i$ be:
$$\theta_{s,a_i} = \mathbf{1}\{a_i=a_i^*(s)\}.$$And set $\theta:= (\theta_i, \theta_{-i}^*)$
From performance difference lemma (Lemma \ref{lemma:averaged-performance-difference-lemma}):
\begin{align*}
    J_i(\theta_i, \theta_{-i}^*) - J_i(\theta_i^*, \theta_{-i}^*) 
    &= \sum_{s,a_i}d_\theta(s)\pi_{\theta_i}(s,a_i)\overline{A_{i}^{\theta^*}}(s,a_i)\\
    &= \sum_s d_\theta(s)\max_{a_i}\overline{A_{i}^{\theta^*}}(s,a_i)\ge 0
\end{align*}
Because $\theta^*$ is a strict NE, thus the inequality above forces $\theta_i^* = \theta$, and that $\max_{a_i}\overline{A_{i}^{\theta^*}}(s,a_i)=0$. The uniqueness of $\theta^*$ also implies uniqueness of $a_i^*(s)$, and thus,
$$\overline{A_{i}^{\theta^*}}(s,a_i)<0, ~~\forall~a_i\neq a_i^*(s),$$
which completes the proof of the lemma.
\end{proof}

The proof of Theorem \ref{theorem:local-convergence-general-game} relies on the following auxiliary lemma, whose proof we defer to Appendix \ref{apdx:auxiliary}. 

\begin{restatable}{lemma}{auxiliary}
\label{lemma:auxiliary}
Let $\cX$ denote the probability simplex of dimension $n$. Suppose $\theta\in \cX, g\in \bR^n$ and that there exists $i^*\in \{1,2,\dots,n\}$ and $\Delta >0$ such that:
\begin{align*}
    \theta_{i^*} &\ge \theta_i, \quad \forall i\neq i^*\\
    g_{i^*} &\ge g_i + \Delta, \quad \forall i \neq i^*.
\end{align*}
Let
\begin{equation*}
    \theta' = Proj_\cX(\theta+g),
\end{equation*}
then:
\begin{equation*}
    \theta'_{i^*} \ge \min\{1, \theta_{i^*} + \frac{\Delta}{2}\}
\end{equation*}
\end{restatable}

\begin{proof}(Theorem \ref{theorem:local-convergence-general-game})
For a fixed agent $i$ and state $s$, the gradient play (\eqref{eq:gradient-play-discrete}) update rule of policy $\theta_{i,s}$ is given by:
\begin{equation}\label{eq:gradient-play-each-element}
    \theta_{i,s}^{(t+1)} = Proj_{\Delta(|\cA_i|)}(\theta_{i,s}^{(t)} + \frac{\eta}{1-\gamma} d_{\theta^{(t)}}(s)\overline{Q_i^{\theta^{(t)}}}(s,\cdot)), 
\end{equation}
where $\Delta(|\cA_i|)$ denotes the probability simplex in $|\cA_i|$-th dimension and $\overline{Q_i^{\theta^{(t)}}}(s,\cdot))$ is a $|\cA_i|$-th dimensional vector with $a_i$-th element equals to $\overline{Q_i^{\theta^{(t)}}}(s,a_i))$. We will show that this update rule satisfies the conditions in \Cref{lemma:auxiliary}, which will then allow us to prove that
\begin{align*}
    D(\theta^{(t+1)}||\theta^*) \le \max\{0, D(\theta^{(t)}||\theta^*) - \frac{\eta\Delta^{\theta^*}}{2}\}.
\end{align*}
Letting $a_i^*(s)$ be the same definition as in Lemma \ref{lemma:strict-NE-gap}, we have that:
\begin{align}
    &\frac{1}{1-\gamma} d_{\theta^{(t)}}(s)\overline{Q_i^{\theta^{(t)}}}(s,a_i^*(s)) -  \frac{1}{1-\gamma} d_{\theta^{(t)}}(s)\overline{Q_i^{\theta^{(t)}}}(s,a_i) \notag\\
    \ge& \frac{1}{1-\gamma} d_{\theta^{*}}(s)\overline{Q_i^{\theta^{*}}}(s,a_i^*(s)) -  \frac{1}{1-\gamma} d_{\theta^{^*}}(s)\overline{Q_i^{\theta^{*}}}(s,a_i) \notag\\
    -& \left|\frac{1}{1-\gamma}d_{\theta^{*}}(s)\overline{Q_i^{\theta^{*}}}(s,a_i^*(s)) - \frac{1}{1-\gamma} d_{\theta^{(t)}}(s)\overline{Q_i^{\theta^{(t)}}}(s,a_i^*(s)) \right|\notag\\
    -& \left|\frac{1}{1-\gamma}d_{\theta^{*}}(s)\overline{Q_i^{\theta^{*}}}(s,a_i) - \frac{1}{1-\gamma} d_{\theta^{(t)}}(s)\overline{Q_i^{\theta^{(t)}}}(s,a_i) \right|\notag\\
    \ge& \frac{1}{1-\gamma} d_{\theta^{*}}(s) \left(\overline{A_i^{\theta^{*}}}(s,a_i^*(s) - \overline{A_i^{\theta^{*}}}(s,a_i))\right) -2\|\nabla_{\theta_i} J_i(\theta^{(t)}) - \nabla_{\theta_i} J_i(\theta^*)\|\label{eq:use-smoothness-1}\\
    \ge& \Delta^{\theta^*} - \frac{4}{(1-\gamma)^3}\left(\sum_{i=1}^n|\cA_i| \right)\|\theta^{(t)} - \theta^*\|\label{eq:use-smoothness-2}\\
    \ge & \Delta^{\theta^*} - \frac{4}{(1-\gamma)^3}\left(\sum_{i=1}^n|\cA_i| \right)\sum_{i=1}^n\sum_s \|\theta^{(t)}_{i,s} -\theta^*_{i,s})\|_1\notag\\
    \ge&\Delta^{\theta^*} - \frac{4}{(1-\gamma)^3}n|\cS|\left(\sum_{i=1}^n|\cA_i| \right) D(\theta^{(t)}||\theta^*),\notag
\end{align}
where \eqref{eq:use-smoothness-1} to \eqref{eq:use-smoothness-2} uses smoothness property in Lemma \ref{lemma:smoothness}.

We use proof of induction as supposed for $\ell \le t-1$, we have:
\begin{equation*}
    D(\theta^{(\ell+1)}||\theta^*)\le \max\{D(\theta^{(\ell)}||\theta^*) - \frac{\eta\Delta^{\theta^*}}{2},0\},
\end{equation*}
thus
\begin{equation*}
     D(\theta^{(t)}||\theta^*) \le D(\theta^{(0)}||\theta^*) \le \frac{\Delta^{\theta^*}(1-\gamma)^3}{8n|\cS|\left(\sum_{i=1}^n|\cA_i| \right)}.
\end{equation*}
Then we can further conclude that:
\begin{align*}
    &(1-\gamma) d_{\theta^{(t)}}(s)\overline{Q_i^{\theta^{(t)}}}(s,a_i^*(s)) -  (1-\gamma) d_{\theta^{(t)}}(s)\overline{Q_i^{\theta^{(t)}}}(s,a_i) \\
    \ge&\Delta^{\theta^*} - \frac{4}{(1-\gamma)^3}n|\cS|\left(\sum_{i=1}^n|\cA_i| \right) D(\theta^{(t)}||\theta^*) \\\ge& \frac{\Delta^{\theta^*}}{2}, \quad \forall~a_i\neq a_i^*(s)
\end{align*}
Additionally, for $D(\theta^{(t)}||\theta^*) \le \frac{\Delta^{\theta^*}(1-\gamma)^3}{8n|\cS|\left(\sum_{i=1}^n|\cA_i| \right)}$, we may conclude:
$$\theta^{(t)}_{s,a_i^*(s)} \ge 1/2 \ge \theta^{(t)}_{s,a_i} \quad \forall a_i \neq a_i^*(s),$$
then by applying \Cref{lemma:auxiliary} to \eqref{eq:gradient-play-each-element} we have:
\begin{align*}
    &\theta^{(t+1)}_{s,a_i^*(s)} \ge \min\{1, \theta^{(t)}_{s,a_i^*(s)} + \frac{\eta\Delta^{\theta^*}}{4}\}\\
    \Longrightarrow \quad& \|\theta^{(t+1)}_{i,s}-\theta^*_{i,s}\|_1=2\left(1-\theta^{(t+1)}_{s,a_i^*(s)}\right)\\&\le \max\{0, \|\theta^{(t)}_{i,s}-\theta^*_{i,s}\|_1 - \frac{\eta\Delta^{\theta^*}}{2}\},\quad \forall ~s\in\cS,~i=1,2,\dots,n\\
    \Longrightarrow \quad& D(\theta^{(t+1)}||\theta^*) \le \max\{0, D(\theta^{(t)}||\theta^*) - \frac{\eta\Delta^{\theta^*}}{2}\},
\end{align*}
which completes the proof.
\end{proof}
\section{Proof of Proposition \ref{prop:potential-game-pure-NE}}\label{apdx:proof-potential-game-pure-NE}
\begin{proof}
First of all, from the definition of NE, the global maximum of the potential function is a NE. We now show that this global maximum is a deterministic policy. From classical results (e.g. \cite{Sutton18}) we know that there is an optimal deterministic centralized policy $$\pi^*(a=(a_1,\dots,a_n)|s) = \mathbf{1}\{a=a^*(s)=(a_1^*(s),\dots,a_n^*(s))\}$$ that maximizes:
\begin{equation*}
    \pi^* = \argmax_{\pi:\cS\rightarrow\Delta(\cA)} \bE \left[\sum_{t=0}^\infty \gamma^t \phi(s_t,a_t)\big|\pi,s_0=s\right].
\end{equation*}
We now show that this centralized policy can also be represented by direct distributed policy parameterization. Set $\theta^*$ as:
\begin{equation*}
    \pi_{\theta_i^*}(a_i|s) = \mathbf{1}\{a_i = a_i^*(s)\},
\end{equation*}
then:
\begin{equation*}
    \pi^*(a|s) = \prod_{i=1}^n\pi_{\theta_i^*}(a_i|s)
\end{equation*}
Since $\pi^*$ globally maximizes the discounted summation of potential function $\phi$ among centralized policies, which includes all possible direct distributedly parameterized policies, $\theta^*$ also maximizes the total potential function $\Phi$ globally among all direct distributed parameterization, which completes the proof.
\end{proof}
\section{Proof of Theorem \ref{theorem:potential-game-convergence}}\label{apdx:proof-potential-game-convergence}
\subsection{Useful Optimization Lemmas}

\begin{lemma}
Let $\Phi(\theta)$ be $\beta$-smooth in $\theta$, define gradient mapping:
$$G^\eta(\theta):= \frac{1}{\eta}\left(Proj_\cX(\theta+\eta\nabla \Phi(\theta)) - \theta\right).$$
The update rule for projected gradient is:
$$\theta^+ = \theta + \eta G^\eta(\theta) = Proj_\cX(\theta+\eta\nabla \Phi(\theta)).$$
Then:
$$(\theta'-\theta^+)^\top \nabla \Phi(\theta^+) \le (1+\eta\beta) \|G^\eta(\theta)\|\|\theta' - \theta^+\| \quad \forall \theta' \in \cX.$$
\end{lemma}
\begin{proof}
By a standard property of Euclidean projections onto a convex set, we get that
\begin{align*}
    &\qquad (\theta+\eta\nabla \Phi(\theta) - \theta^+)^\top (\theta' - \theta^+) \le 0\\
    &\Longrightarrow ~~ \eta\nabla \Phi(\theta)^\top (\theta' - \theta^+) + (\theta - \theta^+)^\top (\theta' - \theta^+) \le 0\\
    &\Longrightarrow ~~ \eta\nabla \Phi(\theta)^\top (\theta' - \theta^+) - \eta G^\eta(\theta)^\top (\theta' - \theta^+) \le 0\\
    &\Longrightarrow ~~ \nabla \Phi(\theta)^\top (\theta' - \theta^+) \le \|G^\eta(\theta)\|\theta' - \theta^+\|\\
    &\Longrightarrow ~~ \nabla \Phi(\theta^+)^\top (\theta' - \theta^+) \le \|G^\eta(\theta)\|\theta' - \theta^+\| + (\nabla \Phi(\theta^+) - \nabla \Phi(\theta))^\top (\theta' - \theta^+)\|\\
    & \qquad\qquad\qquad\qquad\qquad~~ \le  \|G^\eta(\theta)\|\theta' - \theta^+\| + \beta \|\theta^+ - \theta\|\|\theta' - \theta^+\|\\
    & \qquad\qquad\qquad\qquad\qquad~~ = (1+\eta\beta)\|G^\eta(\theta)\|\theta' - \theta^+\| 
\end{align*}
\end{proof}

\begin{lemma}(Sufficient ascent) \label{lemma:sufficient-ascent-exact}
Suppose $\Phi(\theta)$ is $\beta$-smooth. Let $\theta^+ = Proj_{\cX}(\theta+\eta{\nabla}\Phi(\theta))$. Then for $\eta\le\frac{1}{\beta}$,
\begin{equation*}
    \begin{split}
        \Phi(\theta^+) - \Phi(\theta)\ge \frac{\eta}{2}\|G^\eta(\theta)\|^2
    \end{split}
\end{equation*}
\end{lemma}
\begin{proof}
From the smoothness property we have that:
\begin{equation}\label{eq:smoothness-property}
    \Phi(\theta^+)-\Phi(\theta) \ge \nabla_\theta\Phi(\theta)^\top(\theta^+-\theta) -\frac{\beta}{2}\|\theta^+-\theta\|^2
\end{equation}
Since $\theta^+ = Proj_{\cX}(\theta+\eta{\nabla}\Phi(\theta))$, we have that:
\begin{equation*}
    (\theta + \eta\nabla\Phi(\theta) - \theta^+)^\top(\theta'-\theta^+) \le 0, ~~\forall~\theta'\in\cX
\end{equation*}
take $\theta'=\theta$, we get:
\begin{equation*}
    \nabla\Phi(\theta)^\top(\theta^+-\theta) \ge \frac{1}{\eta}\|\theta^+-\theta\|^2. 
\end{equation*}
Thus:
\begin{align*}
    \Phi(\theta^+)-\Phi(\theta) &\ge \nabla_\theta\Phi(\theta)^\top(\theta^+-\theta) -\frac{\beta}{2}\|\theta^+-\theta\|^2\\
    &\ge (\frac{1}{\eta} - \frac{\beta}{2})\|\theta^+-\theta\|^2\\
    & \ge \frac{1}{2\eta} \|\theta^+-\theta\|^2\\
    &= \frac{\eta}{2}\|G^\eta(\theta)\|^2,
\end{align*}
which completes the proof.
\end{proof}
\begin{lemma}(Corollary of Lemma \ref{lemma:sufficient-ascent-exact})\label{lemma:non-cvx-opt-result}
For $\Phi(\theta)$ that is $\beta$ smooth and bounded $\Phi_{\min}\le\Phi(\theta)\le\Phi_{\max}$, running projected gradient ascent:
\begin{equation*}
    \theta^{(t+1)} = Proj_{\cX}(\theta+\eta{\nabla}\Phi(\theta^{(t)}))
\end{equation*}
with $\eta = \frac{1}{\beta}$, will guarantee that:
\begin{equation*}
    \lim_{t\rightarrow+\infty}\|G^\eta(\theta^{(t)})\| =0.
\end{equation*}
Further, we have that:
\begin{equation}\label{eq:telescoping}
    \frac{1}{T}\sum_{t=0}^{T-1}\|G^\eta(\theta^{(t)})\|^2 \le \frac{2\beta(\Phi_{\max}-\Phi_{\min})}{T}
\end{equation}
\end{lemma}
\begin{proof}
From Lemma \ref{lemma:sufficient-ascent-exact} we have:
\begin{align*}
    \Phi(\theta^{(t+1)}) - \Phi(\theta^{(t)}) \ge \frac{1}{2\beta}\|G^\eta(\theta^{(t)})\|^2 \ge 0
\end{align*}
Thus $\Phi(\theta^{(t)})$ is non decreasing, and since it is bounded, we know that $\Phi(\theta^{(t)})$ asymptotically convergence to some value $\Phi^*$, and thus show that
\begin{equation*}
    \lim_{t\rightarrow\infty}\|G^\eta(\theta^{(t)})\| = 0.
\end{equation*}
Additionally, from \eqref{eq:telescoping},
\begin{align*}
    \Phi(\theta^{(T)}) - \Phi(\theta^{(0)}) \ge \sum_{t=0}^{T-1} \frac{1}{2\beta}\|G^\eta(\theta^{(t)})\|^2\\
    \Longrightarrow~~  \frac{1}{T}\sum_{t=0}^{T-1} \|G^\eta(\theta^{(t)})\|^2 \le \frac{2\beta (\Phi_{\max}-\Phi_{\min})}{T},
\end{align*}
which completes the proof.
\end{proof}
\subsection{Proof of Theorem \ref{theorem:potential-game-convergence}}
\begin{proof}
Recall the definition of gradient mapping:
$$G^\eta(\theta) = \frac{1}{\eta}\left(Proj_\cX(\theta+\eta\nabla \Phi(\theta)) - \theta\right).$$
From gradient domination property \eqref{eq:gradient-domination-MAMDP} we have that:
\begin{align*}
    \NEgap_i(\theta^{(t+1)}) &= \max_{\theta_i'\in\cX_i}J_i(\theta_i'. \theta_{-i}^{(t+1)}) - J_i(\theta_i^{(t+1)}, \theta_{-i}^{(t+1)})\\
    &\le \max_{\theta_i'\in\cX_i} \left\|\frac{d_{(\theta_i'. \theta_{-i}^{(t+1)})}}{d_{(\theta_i^{(t+1)}, \theta_{-i}^{(t+1)})}}\right\|_\infty \max_{\overline\theta_i\in \cX_i} \left(\overline\theta_i - \theta_i^{(t+1)}\right)^\top \nabla_{\theta_i}J_i(\theta^{(t+1)})\\
    &\le M\max_{\overline\theta_i\in \cX_i} \left(\overline\theta_i - \theta_i^{(t+1)}\right)^\top \nabla_{\theta_i}\Phi(\theta^{(t+1)})\\
    &\le M(1+\eta\beta) \max_{\overline\theta_i\in \cX_i} \|\overline\theta_i - \theta_i^{(t+1)}\|\|G^\eta(\theta^{(t)})\|\\
    &\le M(1+\eta\beta)2\sqrt{|\cS|}\|G^\eta(\theta^{(t)})\|\\
    & = 4M\sqrt{|\cS|}\|G^\eta(\theta^{(t)})\|
\end{align*}
 where the last step follows as $\|\overline\theta_i - \theta_i^{(t+1)}\|\le 2\sqrt{|\cS|}$. Thus
 \begin{equation*}
     \NEgap(\theta^{(t+1)})\le 4M\sqrt{|\cS|}\|G^\eta(\theta^{(t)})\|
 \end{equation*}
 Then from Lemma \ref{lemma:non-cvx-opt-result} we have that:
 \begin{equation*}
     \lim_{t\rightarrow\infty}\|G^\eta(\theta^{(t)})\| = 0~~ \Longrightarrow~~ \lim_{t\rightarrow\infty}\NEgap(\theta^{(t)}) = 0,
 \end{equation*}
 and that:
 \begin{align*}
     &\frac{1}{T}\sum_{t=0}^{T-1}\|G^\eta(\theta^{(t)})\|^2 \le \frac{2\beta(\Phi_{\max}-\Phi_{\min})}{T}\\
     \Longrightarrow~~  &\frac{1}{T}\sum_{t=1}^{T}\NEgap(\theta^{(t)})^2 \le \frac{32\beta M^2|\cS|(\Phi_{\max}-\Phi_{\min})}{T}
 \end{align*}
we can get our required bound of $\epsilon$ if we set:
\begin{equation*}
    \frac{32\beta M^2|\cS|(\Phi_{\max}-\Phi_{\min})}{T} \le \epsilon^2,
\end{equation*}
or equivalently
\begin{align*}
    T&\ge\frac{32M^2\beta(\Phi_{\max} - \Phi_{\min})|\cS|}{\epsilon^2}\\
    &=\frac{64M^2(\Phi_{\max} - \Phi_{\min})|\cS|\sum_i|\cA_i|}{\epsilon^2(1-\gamma)^3},
\end{align*}
which completes the proof.
\end{proof}

\section{Proof of Theorem \ref{theorem:MPG-local-maximum-saddle-points}}\label{apdx:MPG-local-maximum-saddle-points}
\begin{proof}(of the first claim)
The proof requires knowledge of Lemma \ref{lemma:strict-NE-gap} in Section \ref{sec:gradient-play-SG} thus we would recommend readers to first go through Lemma \ref{lemma:strict-NE-gap} first. The lemma immediately leads to the conclusion that a strict NE $\theta^*$ should be deterministic.
Let $a_i^*(s), \Delta_i^{\theta^*}(s), \Delta_i^{\theta^*}$ be the same definition as in Lemma \ref{lemma:strict-NE-gap}.

For any $\theta\in\cX$, Taylor expansion suggests that:
\begin{align*}
    \Phi(\theta) - \Phi(\theta^*) &= (\theta-\theta^*)^\top \nabla\Phi(\theta^*) + o(\|\theta-\theta^*\|)\\
    &=\sum_i (\theta_i-\theta_i^*)^\top \nabla_{\theta_i}J_i(\theta^*)+ o(\|\theta-\theta^*\|)\\
    &=\frac{1}{1-\gamma} \sum_i\sum_s\sum_{a_i} d_{\theta^*}(s)\overline{A_{i}^{\theta^*}}(s,a_i)(\theta_{s,a_i} - \theta^*_{s,a_i})+ o(\|\theta-\theta^*\|)\\
     &\le-\frac{1}{1-\gamma} \sum_i\sum_s d_{\theta^*}(s)\Delta_i^{\theta^*}(s)\left(\sum_{a_i\!\neq\! a_i^*(s)}(\theta_{s,a_i} - \theta^*_{s,a_i})\right)+ o(\|\theta-\theta^*\|)\\
    &=-\frac{1}{1-\gamma} \sum_i\sum_s d_{\theta^*}(s)\Delta_i^{\theta^*}(s)\frac{1}{2}\|\theta_{i,s} - \theta_{i,s}^*\|_1+ o(\|\theta-\theta^*\|)\\
    &\le -\frac{\Delta^{\theta^*}}{2}\sum_i\sum_s \|\theta_{i,s} - \theta_{i,s}^*\|_1+ o(\|\theta-\theta^*\|)\\
    &\le -\frac{\Delta^{\theta^*}}{2}\|\theta - \theta^*\|+ o(\|\theta-\theta^*\|).
\end{align*}
Thus for $\|\theta-\theta^*\|$ sufficiently small, 
$$\Phi(\theta) - \Phi(\theta^*) <0 \textup{  holds,}$$ this suggests that strict NEs are strict local maxima. We now show that this also holds vice versa.

Strict local maxima satisfy first-order stationarity by definition, and thus by Theorem \ref{thm:equivalence-stationary-NE} they are also NEs, we only need to show that they are strict. We prove by contradiction, suppose that there exists a local maximum $\theta^*$ such that it is non-strict NE, i.e., there exists $\theta_i'\in \cX_i, \theta_i'\neq\theta_i^*$ such that:
\begin{equation*}
    J_i(\theta_i', \theta_{-i}^*) = J_i(\theta_i^*, \theta_{-i}^*)
\end{equation*}
According to \eqref{eq:first-order-detail} and first-order stationarity of $\theta^*$:
\begin{align*}
    \frac{1}{1-\gamma}\sum_{s} d_{\theta^*}(s) \max_{a_i\in\cA_i} \overline {A_i^{\theta^*}}(s, a_i) = \max_{\overline\theta_i\in\cX_i} (\overline\theta_i - \theta_i^*)^\top \nabla_{\theta_i} J_i(\theta^*)\le 0.
\end{align*}
Since $\max_{a_i\in\cA_i} \overline {A_i^{\theta}}(s, a_i) \ge 0$ for all $\theta$, we may conclude:
$$\max_{a_i\in\cA_i} \overline {A_i^{\theta^*}}(s, a_i) = 0,~~\forall~s\in\cS.$$
We denote $\theta':= (\theta_i', \theta_{-i^*})$, according to Lemma \ref{lemma:averaged-performance-difference-lemma}
\begin{align*}
    0 = J_i(\theta_i', \theta_{-i}^*) - J_i(\theta_i^*, \theta_{-i}^*) = \frac{1}{1-\gamma} \sum_{s,a_i} d_{\theta'}(s)\pi_{\theta'_i}(a_i|s)\overline{A_i^{\theta^*}}(s, a_i) \le 0.
\end{align*}
Since $d_{\theta'}(s)>0,~\forall~s$, this further implies that
$$\sum_{a_i}\pi_{\theta'_i}(a_i|s)\overline{A_i^{\theta^*}}(s, a_i) = 0,~~\forall~s\in\cS,$$
i.e., $\pi_{\theta'_i}(a_i|s)$ is nonzero only if $\overline{A_i^{\theta^*}}(s, a_i) = 0$. Define $\theta_i^\eta:= \eta\theta_i' + (1-\eta)\theta_i^*$, then$$\sum_{a_i}\pi_{\theta^\eta_i}(a_i|s)\overline{A_i^{\theta^*}}(s, a_i) = 0,~~\forall~s\in\cS.$$
Thus let $\theta^\eta:=(\theta^\eta_i,\theta^*_{-i})$
\begin{align*}
    J_i(\theta_i^\eta, \theta_{-i}^*) - J_i(\theta_i^*, \theta_{-i}^*) = \frac{1}{1-\gamma} \sum_{s,a_i} d_{\theta^\eta}(s)\pi_{\theta^\eta_i}(a_i|s)\overline{A_i^{\theta^*}}(s, a_i) =0.
\end{align*}
Since $\|\theta_i^\eta - \theta_i^*\|\rightarrow0$ as $\eta\rightarrow0$, this contradicts the assumption that $\theta^*$ is a strict local maximum. This suggests that all strict  local maxima are strict NEs, which completes the proof.
\end{proof}

\begin{proof}(of the second claim)
First, we define the corresponding value function, $Q$-function and advantage function for potential function $\phi$.
\begin{align*}
    V_{\phi}^\theta(s) &:= \bE \left[\sum_{t=0}^\infty \gamma^t \phi(s_t,a_t)\big|~\pi=\theta,s_0=s\right]\\
    Q_{\phi}^\theta(s,a)&:=\left[\sum_{t=0}^\infty \gamma^t \phi(s_t,a_t)\big|~\pi=\theta,s_0=s,a_0=a\right]\\
    A_{\phi}^\theta(s,a) &:= Q_{\phi}^\theta(s,a) - V_{\phi}^\theta(s).
\end{align*}
For an index set $\cI \subseteq \{1,2,\dots,n\}$ we define the following averaged advantage potential function of index set $\cI$ as:
\begin{equation*}
    \overline{A_{\phi,\cI}^{\theta}}(s,a_\cI):= \sum_{a_{-\cI}} A_{\phi}^\theta(s,a_\cI, a_{-\cI}).
\end{equation*}
We choose an index set $\cI \subseteq \{1,2,\dots,n\}$ such that there exists $s^*, a_\cI^*$ such that:
\begin{equation}\label{eq:s-star-a-star}
    \overline{A_{\phi,\cI}^{\theta^*}}(s^*, a_\cI^*) > 0,
\end{equation}
and that for any other index set $\cI'$ with smaller cardinality:
\begin{equation}
    \overline{A_{\phi,\cI'}^{\theta^*}}(s, a_{\cI'}) \le 0, ~~\forall~s, a_{\cI'}, ~~\forall~|\cI'| <|\cI|.
\end{equation}
Because $\Phi$ is not a constant, this guarantees the existence of such an index set $\cI$. Further, since
\begin{equation*}
    \sum_{a_{\cI'}} \pi_{\theta_{\cI'}^*}(a_{\cI'}|s) \overline{A_{\phi,\cI'}^{\theta^*}}(s, a_{\cI'}) = 0, ~~\forall~s,
\end{equation*}
and that $\theta^*$ is fully-mixed, we have that:
\begin{equation}\label{eq:cI-cI-prime}
    \overline{A_{\phi,\cI'}^{\theta^*}}(s, a_{\cI'}) = 0, ~~\forall~s, a_{\cI'}, ~~\forall~|\cI'| <|\cI|.
\end{equation}
We set $\theta:= (\theta_\cI, \theta_{-\cI}^*)$, where $\theta_\cI$ is a convex combination of $\theta_\cI^*, \theta_\cI' \in \cX$:
$$\theta_\cI = (1-\eta)\theta_\cI^* + \eta\theta_\cI', ~~\eta>0.$$
According to performance difference lemma (Lemma \ref{lemma:averaged-performance-difference-lemma}) we have:
\begin{align*}
   & (1-\gamma)\left(\Phi(\theta_\cI, \theta_{-\cI}^*) - \Phi(\theta_\cI^*, \theta_{-\cI}^*)\right) = \sum_{s,a_\cI}d_\theta(s)\pi_{\theta_\cI}(a_\cI|s) \overline{A_{\phi,\cI}^{\theta^*}}(s,a_\cI)\\
    &=\sum_{s,a_\cI}d_\theta(s)\prod_{i\in\cI}\left((1-\eta)\pi_{\theta_i^*}(a_i|s)+\eta\pi_{\theta_i'}(a_i|s) \right) \overline{A_{\phi,\cI}^{\theta^*}}(s,a_\cI)\\
    &=\sum_{s,a_\cI}d_\theta(s)\left((1\!-\!\eta)\pi_{\theta_{i_0}^*}(a_{i_0}|s)+\eta\pi_{\theta_{i_0}'}(a_{i_0}|s)\right)\!\prod_{i\in\cI\!\backslash\!\{\!i_0\!\}\!}\!\left((1\!-\!\eta)\pi_{\theta_i^*}(a_i|s)+\eta\pi_{\theta_i'}(a_i|s) \right) \overline{A_{\phi,\cI}^{\theta^*}}(s,a_\cI), ~~(\forall~i_0\in\cI)\\
    &= (1-\eta)\sum_{s,a_\cI}d_\theta(s)\prod_{i\in\cI\backslash\{\!i_0\!\}\!}\!\left((1\!-\!\eta)\pi_{\theta_i^*}(a_i|s)+\eta\pi_{\theta_i'}(a_i|s) \right) \overline{A_{\phi,\cI\backslash\{i_0\}}^{\theta^*}}(s,a_{\cI\backslash\{i_0\}})\\
    &\quad+\eta \sum_{s,a_\cI}d_\theta(s)\pi_{\theta_{i_0}'}(a_{i_0}|s)\prod_{i\in\cI\backslash\{i_0\}}\left((1-\eta)\pi_{\theta_i^*}(a_i|s)+\eta\pi_{\theta_i'}(a_i|s) \right) \overline{A_{\phi,\cI}^{\theta^*}}(s,a_\cI).
\end{align*}
According to \eqref{eq:cI-cI-prime}, we know that:
$$\overline{A_{\phi,\cI\backslash\{i_0\}}^{\theta^*}}(s,a_{\cI\backslash\{i_0\}})=0,$$
thus
\begin{align*}
    &(1-\gamma)\left(\Phi(\theta_\cI, \theta_{-\cI}^*) - \Phi(\theta_\cI^*, \theta_{-\cI}^*)\right) =\\&\eta \sum_{s,a_\cI}d_\theta(s)\pi_{\theta_{i_0}'}(a_{i_0}|s)\prod_{i\in\cI\backslash\{i_0\}}\left((1-\eta)\pi_{\theta_i^*}(a_i|s)+\eta\pi_{\theta_i'}(a_i|s) \right) \overline{A_{\phi,\cI}^{\theta^*}}(s,a_\cI).
\end{align*}
Applying similar procedures recursively and using the fact that:
$$\overline{A_{\phi,\cI\backslash\{i\}}^{\theta^*}}(s,a_{\cI\backslash\{i\}})=0,~~ \forall~i\in \cI,$$
we get:
\begin{align*}
  \Phi(\theta_\cI, \theta_{-\cI}^*) - \Phi(\theta_\cI^*, \theta_{-\cI}^*) =  \frac{\eta^{|\cI|}}{1-\gamma} \sum_{s,a_\cI}d_\theta(s) \prod_{i\in\cI}\pi_{\theta_i'}(a_i|s)\overline{A_{\phi,\cI}^{\theta^*}}(s,a_\cI).
\end{align*}
Set $\pi_{\theta_i'}(a_i|s)$ as:
\begin{align*}
    \pi_{\theta_i'}(a_i|s^*) &= \left\{
    \begin{array}{cc}
        1 & a_i=a_i^* \\
        0 & \textup{otherwise}
    \end{array}
    \right.\\
    \pi_{\theta_i'}(a_i|s) &= \pi_{\theta_i^*}(a_i|s), \quad s\neq s^*,
\end{align*}
where $s^*,a_i^*$ are defined in \eqref{eq:s-star-a-star}. Then:
\begin{equation*}
    \Phi(\theta_\cI, \theta_{-\cI}^*) - \Phi(\theta_\cI^*, \theta_{-\cI}^*) =   \frac{\eta^{|\cI|}}{1-\gamma}d_\theta(s^*)\overline{A_{\phi,\cI}^{\theta^*}}(s^*,a_\cI^*)>0,
\end{equation*}
which completes the proof.
\end{proof}
\section{Bounding the gradient estimation error of Algorithm \ref{alg:sample-based learning}}
The accuracy of gradient estimation is essential in the sample-based algorithm \ref{alg:sample-based learning}. In this section, we will give a high probability bound of the estimation error, which is stated in the following theorem:
\begin{theorem}{(Error bound for gradient estimation)}\label{thm:gradient-estimation}
Assume that the stochastic game that satisfies Assumption \ref{assump:sufficient-exploration-on-state}. In Algorithm \ref{alg:sample-based learning}, for
\begin{equation*}
    T_J \ge  \frac{32\tau(1+\alpha)^2|\cS|^3\sum_i|\cA_i|\max_{i}|\cA_i|^2}{(1-\gamma)^6\epsilon_g^2\alpha^2\sigma_S^2}\log\left(\frac{16\tau T_G|\cS|^2\sum_i|\cA_i|}{\delta}\right) + 1,
\end{equation*}
with probability at least $1-\delta$, we have:
\begin{equation*}
    \|\widehat{\nabla}\Phi(\theta^{(k)}) - \nabla\Phi(\theta^{(k)})\|_2\le\epsilon_g, ~~\forall ~0\le k\le T_G-1.
\end{equation*}
\end{theorem}

The proof of the theorem includes bounding the estimation error of $\overline{Q_i^\theta}$ (\ref{apdx:estimation-error-Q}) and $d_\theta$ (\ref{apdx:estimation-error-d}). Let's first introduce the definition of `sufficient exploration' which is going to play an important role in this section.

In the main text Assumption \ref{assump:sufficient-exploration-on-state} we have introduced $(\tau,\sigma_S)$-sufficient exploration on states. In this section we introduce a similar definition $(\tau, \sigma)$-sufficient exploration:
\begin{defi}{($(\tau,\sigma)$-Sufficient Exploration)}\label{defi:sufficient-exploration}
A stochastic game and a policy $\theta$ is said to satisfy $(\tau,\sigma)$-sufficient exploration condition if there exists positive integer $\tau$ and $\sigma\in(0,1)$ such that for policy $\theta$ and any initial state-action pair $(s,a_i),~\forall i$, we have
\begin{equation*}
    \Pr(s_{\tau},a_{i,\tau}|s_0=s,a_0=a) \ge \sigma,~~\forall s_{\tau},a_{i,\tau}
\end{equation*}
\end{defi}
Note that `$(\tau,\sigma)$-sufficient exploration' is a stronger condition compared with `$(\tau,\sigma_S)$-sufficient exploration on states'. Additionally it is not hard to verify that for any stochastic game that satisfies $(\tau,\sigma_S)$-sufficient exploration on states, and any $\theta\in\cX^\alpha$, it will also satisfy $(\tau, \frac{\alpha\sigma_S}{\max_i|\cA_i|})$-sufficient exploration condition.

\subsection{Bounding the estimation error of the averaged-Q function}\label{apdx:estimation-error-Q}
We first state the main theorem in this subsection:
\begin{theorem}(Estimation error of averaged Q-functions)\label{thm:policy-evaluation}
Assume that the stochastic game with policy $\theta$ satisfies $(\sigma,\tau)$-sufficient exploration condition (Definition \ref{defi:sufficient-exploration}), then for a fixed $i$, running Algorithm \ref{alg:sample-based learning} will guarantee that:
\begin{equation*}
    \Pr\left(\|\widehat{Q_i^\theta}-\overline{Q_i^\theta}\|_\infty\ge\epsilon\right) \le 4\tau|\cS|^2|\cA_i|\exp\left(-\frac{(1-\gamma)^4\epsilon^2\sigma^2\lfloor\frac{T}{\tau}\rfloor}{32|\cS|^2}\right),
\end{equation*}
further:
\begin{equation*}
    \Pr\left(\|\widehat{Q_i^\theta}-\overline{Q_i^\theta}\|_\infty\ge\epsilon,~\exists~i\right) \le 8\tau|\cS|^2\left(\sum_{i=1}^n|\cA_i|\right)\exp\left(-\frac{(1-\gamma)^4\epsilon^2\sigma^2\lfloor\frac{T}{\tau}\rfloor}{32|\cS|^2}\right),
\end{equation*}
i.e., when
\begin{equation*}
    T_J\ge \frac{32\tau|\cS|^2}{(1-\gamma)^4\epsilon^2\sigma^2}\log\left(\frac{8\tau|\cS|^2|\sum_i|\cA_i|}{\delta}\right) +\tau
\end{equation*}
with probability at least $1-\delta$, $\|\widehat{Q_i^\theta}-\overline{Q_i^\theta}\|_\infty\le\epsilon,~\forall~i$
\end{theorem}

In the following, we will introduce some lemmas which will play an important role in bounding the estimation error of the averaged-Q function:
\begin{lemma}\label{lemma:bound-P-hat}
Assume that the stochastic game with policy $\theta$ satisfies $(\sigma,\tau)$-sufficient exploration condition (Definition \ref{defi:sufficient-exploration}), then fix $s',s,a_i$,for $\epsilon\le 1$,
\begin{equation*}
    \Pr\left(\left|\widehat{P_i^\theta}(s'|s,a_i)-\overline{P_i^\theta}(s'|s,a_i)\right|\ge\epsilon\right)\le 4\tau \exp\left(-\frac{\epsilon^2\sigma^2\lfloor\frac{T}{\tau}\rfloor}{32}\right)
\end{equation*}
\end{lemma}
\begin{proof}
According to the definition of $\widehat{P_i^\theta}$, we have that
\begin{align*}
    &\left\{\widehat{P_i^\theta}(s'|s,a_i)-\overline{P_i^\theta}(s'|s,a_i)\ge\epsilon \right\}\\& \subseteq \left\{\sum_{t=0}^{T-1}\left(\mathbf{1}\{s_{t+1}=s',s_t=s,a_{i,t}=a_i\} - (\overline{P_i^\theta}(s'|s,a_i)+\epsilon)\mathbf{1}\{s_t=s,a_{i,t}=a_i\}\right)\ge0 \right\}\\&\qquad\mathop{\cup}\left\{\sum_{t=0}^{T-1}\mathbf{1}\{s_t=s,a_{i,t}=a_i\}=0 \right\} \\
    &\subseteq\mathop{\cup}_{m=0}^{\tau-1}\left\{\sum_{k=0}^{\lfloor\frac{T-1-m}{\tau}\rfloor}\!\left(\mathbf{1}\{s_{k\tau+m+1}=s',s_{k\tau+m}=s,a_{i,k\tau+m}=a_i\}\! -\! (\overline{P_i^\theta}(s'|s,a_i)+\epsilon)\mathbf{1}\{s_{k\tau+m}=s,a_{i,k\tau+m}=a_i\}\right)\ge0\!\right\}\\
    &\qquad\mathop{\cup}_{m=0}^{\tau-1}\left\{\sum_{k=0}^{\lfloor\frac{T-1-m}{\tau}\rfloor}\mathbf{1}\{s_{k\tau+m}=s,a_{i,k\tau+m}=a_i\}=0 \right\}
\end{align*}
Let:
\begin{align*}
A_m&:=\left\{\sum_{k=0}^{\lfloor\frac{T-1-m}{\tau}\rfloor}\left(\mathbf{1}\{s_{k\tau+m+1}\!=\!s',s_{k\tau+m}\!=\!s,a_{i,k\tau+m}\!=\!a_i\} - (\overline{P_i^\theta}(s'|s,a_i)+\epsilon)\mathbf{1}\{s_{k\tau+m}\!=\!s,a_{i,k\tau+m}\!=\!a_i\}\right)\ge0\right\}\\
A_m'&:=\left\{\sum_{k=0}^{\lfloor\frac{T-1-m}{\tau}\rfloor}\mathbf{1}\{s_{k\tau+m}=s,a_{i,k\tau+m}=a_i\}=0 \right\}\\
    X_{k,m}&:= \mathbf{1}\{s_{k\tau+m+1}=s',s_{k\tau+m}=s,a_{i,k\tau+m}=a_i\} - (\overline{P_i^\theta}(s'|s,a_i)+\epsilon)\mathbf{1}\{s_{k\tau+m}=s,a_{i,k\tau+m}=a_i\}\\
    X_{k,m}'&:=\mathbf{1}\{s_{k\tau+m}=s,a_{i,k\tau+m}=a_i\}\\
    Y_{k,m}&:=X_{k,m}-\bE[X_{k,m}|\cF_{(k-1)\tau+m}]\\
    Y_{k,m}'&:=X_{k,m}'-\bE[X_{k,m}'|\cF_{(k-1)\tau+m}]
\end{align*}
Then $\{Y_{k,m}\}_{k=0}^{\lfloor\frac{T-1-m}{\tau}\rfloor}$ is a martingale difference sequence. Because $\epsilon\le 1$, it is easy to verify that $|X_{k,m}|\le2, |X_{k,m}'|\le1$. We have that:
$$|Y_{k,m}| \le |X_{k,m}|+\bE[|X_{k,m}||\cF_{(k-1)\tau+m}]\le 4,~~|Y_{k,m}'| \le |X_{k,m}'|+\bE[|X_{k,m}'||\cF_{(k-1)\tau+m}]\le 2.$$
Further,
\begin{equation*}
\bE[X'_{k,m}|\cF_{(k-1)\tau+m}] = \bE[\mathbf{1}\{s_{k\tau+m}=s,a_{i,k\tau+m}=a_i\}|\cF_{(k-1)\tau+m}]\ge \sigma,
\end{equation*}
and that
\begin{align}
    &\bE[X_{k,m}|\cF_{(k-1)\tau+m}]\notag\\
    &= \bE[\mathbf{1}\{s_{k\tau+m+1}=s',s_{k\tau+m}=s,a_{i,k\tau+m}=a_i\} - (\overline{P_i^\theta}(s'|s,a_i)+\epsilon)\mathbf{1}\{s_{k\tau+m}=s,a_{i,k\tau+m}=a_i\}|\cF_{(k-1)\tau+m}] \label{eq:X_k_m_prime_conditional_expect_line_2}\\
    &= -\epsilon\bE[\mathbf{1}\{s_{k\tau+m}=s,a_{i,k\tau+m}=a_i\}|\cF_{(k-1)\tau+m}]\le -\epsilon\sigma. \label{eq:X_k_m_prime_conditional_expect_line_3}
\end{align}
To move from \eqref{eq:X_k_m_prime_conditional_expect_line_2} to \eqref{eq:X_k_m_prime_conditional_expect_line_3}, we used the fact that:
\begin{align*}
    \bE[\mathbf{1}\{s_{t+1}=s',s_{t}=s,a_{i,t+m}=a_i\}|\cF_{t-1}] &= P(s|s_{t-1},a_{t-1})\sum_{a_{-i}}\pi_{\theta}(a_i,a_{-i}|s)P(s'|s,a_i,a_{-i})\\
    &=P(s|s_{t-1},a_{t-1})\pi_{\theta_i}(a_i|s)\sum_{a_{-i}}\pi_{\theta_{-i}}(a_{-i}|s)P(s'|s,a_i,a_{-i})\\
    &=P(s|s_{t-1},a_{t-1})\pi_{\theta_i}(a_i|s)\overline{P_i^\theta}(s'|s,a_i)\\
    &=\bE[\overline{P_i^\theta}(s'|s,a_i)\mathbf{1}\{s_{t}=s,a_{i,t+m}=a_i\}|\cF_{t-1}]
\end{align*}
and the inequality in \eqref{eq:X_k_m_prime_conditional_expect_line_3} is derived directly from Definition \ref{defi:sufficient-exploration}.

According to Azuma-Hoeffding inequality \cite{hoeffding94,azuma67}:
\begin{align*}
    \Pr(A_m) &= \Pr\left(\sum_{k=0}^{\lfloor\frac{T-1-m}{\tau}\rfloor} X_{k,m}\ge0\right)\\
    &=\Pr\left(\sum_{k=0}^{\lfloor\frac{T-1-m}{\tau}\rfloor} Y_{k,m}\ge -\sum_{k=0}^{\lfloor\frac{T-1-m}{\tau}\rfloor}\bE[X_{k,m}|\cF_{(k-1)\tau+m}]\right)\\
    &\le \Pr\left(\sum_{k=0}^{\lfloor\frac{T-1-m}{\tau}\rfloor} Y_{k,m}\ge \left\lfloor\frac{T-1-m+\tau}{\tau}\right\rfloor\epsilon\sigma\right)\\
    &\le \exp\left(-\frac{\epsilon^2\sigma^2\left\lfloor\frac{T}{\tau}\right\rfloor}{32}\right)
\end{align*}
Similarly, from Azuma-Hoeffding inequality,
\begin{align}
    \Pr(A_m') &= \Pr\left(\sum_{k=0}^{\lfloor\frac{T-1-m}{\tau}\rfloor} X_{k,m}'=0\right)\notag\\
    &=\Pr\left(\sum_{k=0}^{\lfloor\frac{T-1-m}{\tau}\rfloor} Y_{k,m}'= -\sum_{k=0}^{\lfloor\frac{T-1-m}{\tau}\rfloor}\bE[X_{k,m}'|\cF_{(k-1)\tau+m}]\right)\notag\\
    &\le \Pr\left(\sum_{k=0}^{\lfloor\frac{T-1-m}{\tau}\rfloor} Y_{k,m}'\le -\left\lfloor\frac{T-1-m+\tau}{\tau}\right\rfloor\sigma\right)\notag\\
    &\le \exp\left(-\frac{\sigma^2\left\lfloor\frac{T}{\tau}\right\rfloor}{8}\right)\label{eq:Am'}
\end{align}
Thus
\begin{align*}
    &\Pr\left(\widehat{P_i^\theta}(s'|s,a_i)-\overline{P_i^\theta}(s'|s,a_i)\ge\epsilon \right) \le \sum_{m=0}^{\tau-1}\Pr\left(A_m\right) +\Pr\left(A_m'\right)\\
    &\le \tau \exp\left(-\frac{\epsilon^2\sigma^2\left\lfloor\frac{T}{\tau}\right\rfloor}{32}\right) + \tau\exp\left(-\frac{\sigma^2\left\lfloor\frac{T}{\tau}\right\rfloor}{8}\right) \le 2\tau \exp\left(-\frac{\epsilon^2\sigma^2\left\lfloor\frac{T}{\tau}\right\rfloor}{32}\right)
\end{align*}
Similarly
\begin{align*}
    \Pr\left(\widehat{P_i^\theta}(s'|s,a_i)-\overline{P_i^\theta}(s'|s,a_i)\le-\epsilon \right) \le  2\tau \exp\left(-\frac{\epsilon^2\sigma^2\left\lfloor\frac{T}{\tau}\right\rfloor}{32}\right)\\
    \Longrightarrow \Pr\left(\left|\widehat{P_i^\theta}(s'|s,a_i)-\overline{P_i^\theta}(s'|s,a_i)\right|\ge\epsilon\right)\le4\tau \exp\left(-\frac{\epsilon^2\sigma^2\left\lfloor\frac{T}{\tau}\right\rfloor}{32}\right)
\end{align*}
which completes the proof.
\end{proof}

\begin{lemma}\label{lemma:bound-r-hat}
Assume that the stochastic game with policy $\theta$ satisfies $(\sigma,\tau)$-sufficient exploration condition (Definition \ref{defi:sufficient-exploration}), then fix $s',s,a_i$, for $\epsilon\le 1$,
\begin{equation*}
    \Pr\left(\left|\widehat{r_i^\theta}(s,a_i)-\overline{r_i^\theta}(s,a_i)\right|\ge\epsilon \right) \le 4\tau \exp\left(-\frac{\epsilon^2\sigma^2\lfloor\frac{T}{\tau}\rfloor}{32}\right)
\end{equation*}
\end{lemma}
\begin{proof}
The proof is similar to Lemma \ref{lemma:bound-P-hat}.
\begin{align*}
    &\left\{\widehat{r_i^\theta}(s,a_i)-\overline{r_i^\theta}(s,a_i)\ge\epsilon \right\}\\
    &\subseteq \left\{\sum_{t=0}^T\mathbf{1}\{s_t=s,a_{i,t}=a_i\}r_i(s_t,a_t) - (\overline{r_i^\theta}(s,a_i)+\epsilon)\mathbf{1}\{s_t=s,a_{i,t}=a_i\}\ge0\right\}\\
    &\qquad\mathop{\cup}\left\{\sum_{t=0}^{T-1}\mathbf{1}\{s_t=s,a_{i,t}=a_i\}=0 \right\}\\
    &\subseteq \mathop{\cup}_{m=0}^{\tau-1}\left\{\sum_{k=0}^{\lfloor\frac{T-m}{\tau}\rfloor}\mathbf{1}\{s_{k\tau+m}=s,a_{i,k\tau+m}=a_i\}r_i(s_{k\tau+m},a_{k\tau+m}) - (\overline{r_i^\theta}(s,a_i)+\epsilon)\mathbf{1}\{s_{k\tau+m}=s,a_{i,k\tau+m}=a_i\}\ge0\right\}\\
     &\qquad\mathop{\cup}_{m=0}^{\tau-1}\left\{\sum_{k=0}^{\lfloor\frac{T-1-m}{\tau}\rfloor}\mathbf{1}\{s_{k\tau+m}=s,a_{i,k\tau+m}=a_i\}=0 \right\}
\end{align*}
Let:
\begin{align*}
    A_m&:=\left\{\sum_{k=0}^{\lfloor\frac{T-m}{\tau}\rfloor}\mathbf{1}\{s_{k\tau+m}=s,a_{i,k\tau+m}=a_i\}r_i(s_{k\tau+m},a_{k\tau+m}) - (\overline{r_i^\theta}(s,a_i)+\epsilon)\mathbf{1}\{s_{k\tau+m}=s,a_{i,k\tau+m}=a_i\}\ge0\right\}\\
    A_m'&:=\left\{\sum_{k=0}^{\lfloor\frac{T-1-m}{\tau}\rfloor}\mathbf{1}\{s_{k\tau+m}=s,a_{i,k\tau+m}=a_i\}=0 \right\}\\
    X_{k,m}&:=\mathbf{1}\{s_{k\tau+m}=s,a_{i,k\tau+m}=a_i\}r_i(s_{k\tau+m},a_{k\tau+m}) - (\overline{r_i^\theta}(s,a_i)+\epsilon)\mathbf{1}\{s_{k\tau+m}=s,a_{i,k\tau+m}=a_i\}\\
    Y_{k,m}&:= X_{k,m} - \bE[X_{k,m}|\cF_{(k-1)\tau + m}]
\end{align*}
Then $\{Y_{k,m}\}_{k=0}^{\lfloor\frac{T-1-m}{\tau}\rfloor}$ is a martingale difference sequence. Because $\epsilon\le 1$, it is easy to verify that $|X_{k,m}|\le2$. We have that:
$$|Y_{k,m}| \le |X_{k,m}|+\bE[|X_{k,m}||\cF_{(k-1)\tau+m}]\le 4.$$

Further,
\begin{align*}
    &\bE[X_{k,m}|\cF_{(k-1)\tau+m}]\\
    &= \bE[\mathbf{1}\{s_{k\tau+m}=s,a_{i,k\tau+m}=a_i\}r_i(s_{k\tau+m},a_{k\tau+m}) - (\overline{r_i^\theta}(s,a_i)+\epsilon)\mathbf{1}\{s_{k\tau+m}=s,a_{i,k\tau+m}=a_i\}|\cF_{(k-1)\tau+m}]\\
    &= -\epsilon\bE[\mathbf{1}\{s_{k\tau+m}=s,a_{i,k\tau+m}=a_i\}|\cF_{(k-1)\tau+m}]\le -\epsilon\sigma
\end{align*}
the second line to the third line of the equation is derived by the fact that:
\begin{align*}
    \bE[\mathbf{1}\{s_{t}=s,a_{i,t+m}=a_i\}r_i(s_t,a_t)|\cF_{t-1}] &= P(s|s_{t-1},a_{t-1})\sum_{a_{-i}}\pi_\theta(a_i,a_{-i}|s)r_i(s,a_i,a_{-i})\\
    &=P(s|s_{t-1},a_{t-1})\pi_{\theta_i}(a_i|s)\overline{r_i^\theta}(s,a_i)\\
    &= \bE[\overline{r_i^\theta}(s,a_i)\mathbf{1}\{s_t=s,a_{i,t}=a_t\}|\cF_{t-1}]
\end{align*}
and the inequality in the third line is derived directly from Definition \ref{defi:sufficient-exploration}.

According to Azuma-Hoeffding inequality:
\begin{align*}
    \Pr(A_m) &= \Pr\left(\sum_{k=0}^{\lfloor\frac{T-m}{\tau}\rfloor} X_{k,m}\ge0\right)\\
    &=\Pr\left(\sum_{k=0}^{\lfloor\frac{T-m}{\tau}\rfloor} Y_{k,m}\ge -\sum_{k=0}^{\lfloor\frac{T-m}{\tau}\rfloor}\bE[X_{k,m}|\cF_{(k-1)\tau+m}]\right)\\
    &\le \Pr\left(\sum_{k=0}^{\lfloor\frac{T-m}{\tau}\rfloor} Y_{k,m}\ge \left\lfloor\frac{T-m+\tau}{\tau}\right\rfloor\epsilon\sigma\right)\\
    &\le \exp\left(-\frac{\epsilon^2\sigma^2\left\lfloor\frac{T}{\tau}\right\rfloor}{32}\right)
\end{align*}
Same as \eqref{eq:Am'}, we have that:
\begin{equation*}
    \Pr(A_m') \le \exp\left(-\frac{\sigma^2\left\lfloor\frac{T}{\tau}\right\rfloor}{8}\right)
\end{equation*}
Thus
\begin{align*}
    \Pr\left(\widehat{r_i^\theta}(s,a_i)-\overline{r_i^\theta}(s,a_i)\ge\epsilon  \right) \le \sum_{m=0}^{\tau-1}\Pr\left(A_m\right) + \Pr\left(A_m'\right)\le 2\tau \exp\left(-\frac{\epsilon^2\sigma^2\left\lfloor\frac{T}{\tau}\right\rfloor}{32}\right)
\end{align*}
Similarly
\begin{align*}
    \Pr\left(\widehat{r_i^\theta}(s,a_i)-\overline{r_i^\theta}(s,a_i)\le-\epsilon \right) \le  2\tau \exp\left(-\frac{\epsilon^2\sigma^2\left\lfloor\frac{T}{\tau}\right\rfloor}{32}\right)\\
    \Longrightarrow \Pr\left(\left|\widehat{r_i^\theta}(s,a_i)-\overline{r_i^\theta}(s,a_i)\right|\ge\epsilon \right)\le4\tau \exp\left(-\frac{\epsilon^2\sigma^2\left\lfloor\frac{T}{\tau}\right\rfloor}{32}\right)
\end{align*}
which completes the proof.
\end{proof}
Lemma \ref{lemma:bound-P-hat} and \ref{lemma:bound-r-hat} lead to the following corollary:
\begin{coro}\label{coro:bound-r-M-hat}
\begin{align}
    \Pr(\|\widehat{M^\theta_i} - \overline{M^\theta_i}\|_\infty \ge\epsilon) &\le 4\tau|\cS|^2|\cA_i|\exp\left(-\frac{\epsilon^2\sigma^2\lfloor\frac{T}{\tau}\rfloor}{32|\cS|^2}\right)\label{eq:bound-M-hat}\\
    \Pr(\|\widehat{r^\theta_i} - \overline{r^\theta_i}\|_\infty \ge\epsilon) &\le 4\tau|\cS||\cA_i|\exp\left(-\frac{\epsilon^2\sigma^2\lfloor\frac{T}{\tau}\rfloor}{32}\right)\label{eq:bound-r-hat}
\end{align}
\end{coro}
\begin{proof}
We first prove \eqref{eq:bound-M-hat}
\begin{align*}
    \|\widehat{M^\theta_i} - \overline{M^\theta_i}\|_\infty &=\max_{(s,a_i)}\sum_{(s',a_i')}\pi_{\theta_i}(a_i'|s')\left|\widehat{P_i^\theta}(s'|s,a_i)-\overline{P_i^\theta}(s'|s,a_i)\right|\\
    &=\max_{(s,a_i)}\sum_{s'}\left|\widehat{P_i^\theta}(s'|s,a_i)-\overline{P_i^\theta}(s'|s,a_i)\right|
\end{align*}
Thus,
\begin{align*}
    \left\{\|\widehat{M^\theta_i} - \overline{M^\theta_i}\|_\infty \ge\epsilon\right\} &= \mathop{\cup}_{(s,a_i)}\left\{\sum_{s'}\left|\widehat{P_i^\theta}(s'|s,a_i)-\overline{P_i^\theta}(s'|s,a_i)\right| \ge\epsilon\right\}\\
    &\subseteq \mathop{\cup}_{(s,a_i)}\mathop{\cup}_{s'}\left\{\left|\widehat{P_i^\theta}(s'|s,a_i)-\overline{P_i^\theta}(s'|s,a_i)\right| \ge\frac{\epsilon}{|\cS|}\right\}
\end{align*}
Then according to Lemma \ref{lemma:bound-P-hat},
\begin{align*}
    \Pr\left(\|\widehat{M^\theta_i} - \overline{M^\theta_i}\|_\infty \ge\epsilon\right) &\le \sum_{(s',s,a_i)}\Pr\left(\left|\widehat{P_i^\theta}(s'|s,a_i)-\overline{P_i^\theta}(s'|s,a_i)\right| \ge\frac{\epsilon}{|\cS|}\right)\\
    &\le 4\tau|\cS|^2|\cA_i|\exp\left(-\frac{\epsilon^2\sigma^2\lfloor\frac{T}{\tau}\rfloor}{32|\cS|^2}\right).
\end{align*}
Now we prove \eqref{eq:bound-r-hat}. Since
\begin{align*}
    \left\{\|\widehat{r^\theta_i} - \overline{r^\theta_i}\|_\infty \ge\epsilon\right\} = \mathop{\cup}_{(s,a_i)}\left\{\left|\widehat{r_i^\theta}(s,a_i)-\overline{r_i^\theta}(s,a_i)\right|\ge\epsilon \right\},
\end{align*}
according to Lemma \ref{lemma:bound-r-hat},
\begin{align*}
    \Pr\left(\|\widehat{r^\theta_i} - \overline{r^\theta_i}\|_\infty \ge\epsilon\right) &\le \sum_{(s,a_i)}\Pr\left(\left|\widehat{r_i^\theta}(s,a_i)-\overline{r_i^\theta}(s,a_i)\right|\ge\epsilon \right)\\
    &\le 4\tau|\cS||\cA_i|\exp\left(-\frac{\epsilon^2\sigma^2\left\lfloor\frac{T}{\tau}\right\rfloor}{32}\right),
\end{align*}
which completes the proof of the corollary.
\end{proof}
We are now ready to prove Theorem \ref{thm:policy-evaluation}.

\begin{proof}(of Theorem \ref{thm:policy-evaluation})
From the definition of $\widehat{Q_i^\theta}, \overline{Q_i^\theta}$, 
\begin{align*}
    \overline{Q_i^\theta} &= (I-\gamma \overline{M^\theta_i})^{-1}\overline{r_i^\theta},\\
    \widehat{Q_i^\theta} &= (I-\gamma \widehat{M^\theta_i})^{-1}\widehat{r_i^\theta},
\end{align*}
we have that
\begin{align*}
    \|\widehat{Q_i^\theta}-\overline{Q_i^\theta}\|_\infty&=\left\|(I-\gamma \widehat{M^\theta_i})^{-1}\widehat{r_i^\theta} - (I-\gamma \overline{M^\theta_i})^{-1}\overline{r_i^\theta}\right\|_\infty\\
    &=\left\|(I-\gamma \overline{M^\theta_i})^{-1}(\widehat{r_i^\theta}-\overline{r_i^\theta}) + \left((I-\gamma \widehat{M^\theta_i})^{-1}-(I-\gamma\overline{M^\theta_i})^{-1}\right)\widehat{r_i^\theta}\right\|_\infty\\
    &\le \left\|(I-\gamma \overline{M^\theta_i})^{-1}(\widehat{r_i^\theta}-\overline{r_i^\theta})\right\|_\infty + \left\| \gamma(I-\gamma \overline{M^\theta_i})^{-1}(\widehat{M^\theta_i} - \overline{M^\theta_i})(I-\gamma \widehat{M^\theta_i})^{-1}\widehat{r_i^\theta}\right\|_\infty.
\end{align*}
Because both $\overline{M^\theta_i}$ and $\widehat{M^\theta_i}$ are transition probability matrices, thus:
\begin{align*}
    \|\overline{M^\theta_i}x\|_\infty &\le \|x\|_\infty\\
    \|\widehat{M^\theta_i}x\|_\infty &\le \|x\|_\infty\\
    \|(I-\gamma \overline{M^\theta_i})^{-1}x\|_\infty &\le \frac{1}{1-\gamma}\|x\|_\infty\\
     \|(I-\gamma \widehat{M^\theta_i})^{-1}x\|_\infty &\le \frac{1}{1-\gamma}\|x\|_\infty
\end{align*}
Thus,
\begin{align*}
    \|\widehat{Q_i^\theta}-\overline{Q_i^\theta}\|_\infty
    &\le \left\|(I-\gamma \overline{M^\theta_i})^{-1}(\widehat{r_i^\theta}-\overline{r_i^\theta})\right\|_\infty + \left\| \gamma (I-\gamma \overline{M^\theta_i})^{-1}(\widehat{M^\theta_i} - \overline{M^\theta_i})(I-\gamma \widehat{M^\theta_i})^{-1}\widehat{r_i^\theta}\right\|_\infty\\
    &\le \frac{1}{1-\gamma}\|\widehat{r_i^\theta}-\overline{r_i^\theta}\|_\infty + \frac{\gamma}{(1-\gamma)^2}\|\widehat{M^\theta_i} - \overline{M^\theta_i}\|_\infty\|\widehat{r_i^\theta}\|_\infty\\
    &\le \frac{1}{1-\gamma}\|\widehat{r_i^\theta}-\overline{r_i^\theta}\|_\infty + \frac{\gamma}{(1-\gamma)^2}\|\widehat{M^\theta_i} - \overline{M^\theta_i}\|_\infty
\end{align*}
Thus if
\begin{align*}
    \|\widehat{r_i^\theta}-\overline{r_i^\theta}\|_\infty \le (1-\gamma)^2\epsilon,\quad \|\widehat{M^\theta_i} - \overline{M^\theta_i}\|_\infty \le (1-\gamma)^2\epsilon,
\end{align*}
we have that:
\begin{equation*}
    \|\widehat{Q_i^\theta}-\overline{Q_i^\theta}\|_\infty \le \epsilon,
\end{equation*}
Thus from Corollary \ref{coro:bound-r-M-hat},
\begin{align*}
    \Pr\left(\|\widehat{Q_i^\theta}-\overline{Q_i^\theta}\|_\infty\ge\epsilon\right) &\le \Pr\left( \|\widehat{r_i^\theta}-\overline{r_i^\theta}\|_\infty \ge (1-\gamma)^2\epsilon\right) + \Pr\left(  \|\widehat{M^\theta_i} - \overline{M^\theta_i}\|_\infty \ge (1-\gamma)^2\epsilon\right)\\
    &\le 4\tau|\cS||\cA_i|\exp\left(-\frac{(1-\gamma)^4\epsilon^2\sigma^2\lfloor\frac{T}{\tau}\rfloor}{32}\right) + 4\tau|\cS|^2|\cA_i|\exp\left(-\frac{(1-\gamma)^4\epsilon^2\sigma^2\lfloor\frac{T}{\tau}\rfloor}{32|\cS|^2}\right)\\
    &\le 8\tau|\cS|^2|\cA_i|\exp\left(-\frac{(1-\gamma)^4\epsilon^2\sigma^2\lfloor\frac{T}{\tau}\rfloor}{32|\cS|^2}\right),
\end{align*}
which completes the proof.
\end{proof}

\subsection{Bounding the estimation error of $d_\theta$}\label{apdx:estimation-error-d}
We first state our main result:
\begin{theorem}(Estimation error of $d_\theta$)\label{thm:estimation-d-theta}
Under Assumption \ref{assump:sufficient-exploration-on-state},
\begin{equation*}
    \Pr\left(\|\widehat{d_\theta}-d_\theta\|_1\ge\epsilon\right)
    \le 4\tau|\cS|^2 \exp\left(-\frac{(1-\gamma)^2\epsilon^2\sigma_S^2\left\lfloor\frac{T}{\tau}\right\rfloor}{32\gamma^2|\cS|^2}\right),
\end{equation*}
i.e., when
\begin{equation*}
T\ge \frac{32\tau|\cS|^2}{(1-\gamma)^2\epsilon^2\sigma_S^2}\log\left(\frac{4\tau|\cS|^2|}{\delta}\right) +1,
\end{equation*}
with probability at least $1-\delta$, $\|\widehat{d_\theta}-d_\theta\|_1\le\epsilon$.
\end{theorem}
Similar to the previous section, the proof of the theorem begins by bounding the estimation error $|\widehat{P_\cS^\theta}(s'|s)-\overline{P_\cS^\theta}(s'|s)|$.
\begin{lemma}\label{lemma:bound-P-S-hat}
Under Assumption \ref{assump:sufficient-exploration-on-state}, fix $s',s,a_i$,for $\epsilon\le 1$,
\begin{equation*}
    \Pr\left(\left|\widehat{P_\cS^\theta}(s'|s)-\overline{P_\cS^\theta}(s'|s)\right|\ge\epsilon\right)\le 4\tau \exp\left(-\frac{\epsilon^2\sigma_S^2\left\lfloor\frac{T}{\tau}\right\rfloor}{32}\right)
\end{equation*}
\end{lemma}
\begin{proof}
According to the definition of $\widehat{P_\cS^\theta}$, we have that
\begin{align*}
    &\left\{\widehat{P_\cS^\theta}(s'|s)-\overline{P_\cS^\theta}(s'|s)\ge\epsilon \right\}\\& \subseteq \left\{\sum_{t=0}^{T-1}\left(\mathbf{1}\{s_{t+1}=s',s_t=s\} - (\overline{P_\cS^\theta}(s'|s)+\epsilon)\mathbf{1}\{s_t=s\}\right)\ge0 \right\}\cup\left\{\sum_{t=0}^{T-1}\mathbf{1}\{s_t=s\}=0 \right\}\\
    &\subseteq\mathop{\cup}_{m=0}^{\tau-1}\left\{\sum_{k=0}^{\lfloor\frac{T-1-m}{\tau}\rfloor}\left(\mathbf{1}\{s_{k\tau+m+1}=s',s_{k\tau+m}=s\} - (\overline{P_\cS^\theta}(s'|s)+\epsilon)\mathbf{1}\{s_{k\tau+m}=s\}\right)\ge0\right\}\\
    &\qquad\mathop{\cup}_{m=0}^{\tau-1}\left\{\sum_{k=0}^{\lfloor\frac{T-1-m}{\tau}\rfloor}\mathbf{1}\{s_{k\tau+m}=s\}=0\right\}
\end{align*}
Let:
\begin{align*}
A_m&:=\left\{\sum_{k=0}^{\lfloor\frac{T-1-m}{\tau}\rfloor}\left(\mathbf{1}\{s_{k\tau+m+1}=s',s_{k\tau+m}=s\} - (\overline{P_\cS^\theta}(s'|s)+\epsilon)\mathbf{1}\{s_{k\tau+m}=s\}\right)\ge0\right\}\\
A_m'&:=\left\{\sum_{k=0}^{\lfloor\frac{T-1-m}{\tau}\rfloor}\mathbf{1}\{s_{k\tau+m}=s\}=0\right\}\\
    X_{k,m}&:= \mathbf{1}\{s_{k\tau+m+1}=s',s_{k\tau+m}=s\} - (\overline{P_\cS^\theta}(s'|s)+\epsilon)\mathbf{1}\{s_{k\tau+m}=s\}\\
    X_{k,m}'&:= \mathbf{1}\{s_{k\tau+m}=s\}\\
    Y_{k,m}&:=X_{k,m}-\bE[X_{k,m}|\cF_{(k-1)\tau+m}]\\
    Y_{k,m}'&:=X_{k,m}'-\bE[X_{k,m}'|\cF_{(k-1)\tau+m}]
\end{align*}
Then $\{Y_{k,m}\}_{k=0}^{\lfloor\frac{T-1-m}{\tau}\rfloor}$ is a martingale difference sequence. Because $\epsilon\le 1$, it is easy to verify that $|X_{k,m}|\le2, |X_{k,m}|\le1$. We have that:
$$|Y_{k,m}| \le |X_{k,m}|+\bE[|X_{k,m}||\cF_{(k-1)\tau+m}]\le 4,~~|Y_{k,m}'| \le |X_{k,m}'|+\bE[|X_{k,m}'||\cF_{(k-1)\tau+m}]\le 2.$$

Further,
\begin{equation*}
\bE[X'_{k,m}|\cF_{(k-1)\tau+m}] = \bE[\mathbf{1}\{s_{k\tau+m}=s\}|\cF_{(k-1)\tau+m}]\ge \sigma_S,
\end{equation*}
and that
\begin{align*}
    &\bE[X_{k,m}|\cF_{(k-1)\tau+m}]\\
    &= \bE[\mathbf{1}\{s_{k\tau+m+1}=s',s_{k\tau+m}=s\} - (\overline{P_\cS^\theta}(s'|s)+\epsilon)\mathbf{1}\{s_{k\tau+m}=s\}|\cF_{(k-1)\tau+m}]\\
    &= -\epsilon\bE[\mathbf{1}\{s_{k\tau+m}=s\}|\cF_{(k-1)\tau+m}]\le -\epsilon\sigma_S
\end{align*}
the second line to the third line of the equation is derived by the fact that:
\begin{align*}
    \bE[\mathbf{1}\{s_{t+1}=s',s_{t}=s\}|\cF_{t-1}] &= P(s|s_{t-1},a_{t-1})\sum_{a}\pi_{\theta}(a|s)P(s'|s,a)\\
    &=P(s|s_{t-1},a_{t-1})\overline{P_\cS^\theta}(s'|s)\\
    &=\bE[\overline{P_\cS^\theta}(s'|s)\mathbf{1}\{s_{t}=s\}|\cF_{t-1}]
\end{align*}
and the inequality in the third line is derived directly from Assumption \ref{assump:sufficient-exploration-on-state}.

According to Azuma-Hoeffding inequality:
\begin{align*}
    \Pr(A_m) &= \Pr\left(\sum_{k=0}^{\lfloor\frac{T-1-m}{\tau}\rfloor} X_{k,m}\ge0\right)\\
    &=\Pr\left(\sum_{k=0}^{\lfloor\frac{T-1-m}{\tau}\rfloor} Y_{k,m}\ge -\sum_{k=0}^{\lfloor\frac{T-1-m}{\tau}\rfloor}\bE[X_{k,m}|\cF_{(k-1)\tau+m}]\right)\\
    &\le \Pr\left(\sum_{k=0}^{\lfloor\frac{T-1-m}{\tau}\rfloor} Y_{k,m}\ge \left\lfloor\frac{T-1-m+\tau}{\tau}\right\rfloor\epsilon\sigma_S\right)\\
    &\le \exp\left(-\frac{\epsilon^2\sigma_S^2\left\lfloor\frac{T}{\tau}\right\rfloor}{32}\right)
\end{align*}
Similarly, from Azuma-Hoeffding inequality,
\begin{align*}
    \Pr(A_m') &= \Pr\left(\sum_{k=0}^{\lfloor\frac{T-1-m}{\tau}\rfloor} X_{k,m}'=0\right)\\
    &=\Pr\left(\sum_{k=0}^{\lfloor\frac{T-1-m}{\tau}\rfloor} Y_{k,m}'= -\sum_{k=0}^{\lfloor\frac{T-1-m}{\tau}\rfloor}\bE[X_{k,m}'|\cF_{(k-1)\tau+m}]\right)\\
    &\le \Pr\left(\sum_{k=0}^{\lfloor\frac{T-1-m}{\tau}\rfloor} Y_{k,m}'\le -\left\lfloor\frac{T-1-m+\tau}{\tau}\right\rfloor\sigma_S\right)\\
    &\le \exp\left(-\frac{\sigma_S^2\left\lfloor\frac{T}{\tau}\right\rfloor}{8}\right)
\end{align*}
Thus
\begin{align*}
    &\Pr\left(\widehat{P_\cS^\theta}(s'|s)-\overline{P_\cS^\theta}(s'|s)\ge\epsilon \right) \le \sum_{m=0}^{\tau-1}\Pr\left(A_m\right)+\Pr\left(A_m'\right)\\
    &\le \tau \exp\left(-\frac{\epsilon^2\sigma_S^2\left\lfloor\frac{T}{\tau}\right\rfloor}{32}\right) + \tau \exp\left(-\frac{\sigma_S^2\left\lfloor\frac{T}{\tau}\right\rfloor}{8}\right)\le2\tau \exp\left(-\frac{\epsilon^2\sigma_S^2\left\lfloor\frac{T}{\tau}\right\rfloor}{32}\right) 
\end{align*}
Similarly
\begin{align*}
    \Pr\left(\widehat{P_\cS^\theta}(s'|s)-\overline{P_\cS^\theta}(s'|s)\le-\epsilon \right) \le  2\tau \exp\left(-\frac{\epsilon^2\sigma_S^2\left\lfloor\frac{T}{\tau}\right\rfloor}{32}\right)\\
    \Longrightarrow \Pr\left(\left|\widehat{P_\cS^\theta}(s'|s)-\overline{P_\cS^\theta}(s'|s)\right|\ge\epsilon\right)\le4\tau \exp\left(-\frac{\epsilon^2\sigma_S^2\left\lfloor\frac{T}{\tau}\right\rfloor}{32}\right)
\end{align*}
which completes the proof.
\end{proof}

\begin{coro}\label{coro:bound-P-S-hat}
\begin{align}
    \Pr\left(\left\|\widehat{P_\cS^\theta} - \overline{P_\cS^\theta}\right\|_\infty \ge\epsilon\right) &\le 4\tau|\cS|^2\exp\left(-\frac{\epsilon^2\sigma_S^2\lfloor\frac{T}{\tau}\rfloor}{32|\cS|^2}\right)\label{eq:bound-P-S-hat}
\end{align}
\end{coro}
\begin{proof}
\begin{align*}
    \left\|\widehat{P_\cS^\theta} - \overline{P_\cS^\theta}\right\|_\infty &=\max_{s}\sum_{s'}\left|\widehat{P_\cS^\theta}(s'|s) - \overline{P_\cS^\theta}(s'|s)\right|
\end{align*}
Thus,
\begin{align*}
    \left\{\left\|\widehat{P_\cS^\theta} - \overline{P_\cS^\theta}\right\|_\infty \ge\epsilon\right\} &= \mathop{\cup}_{s}\left\{\sum_{s'}\left|\widehat{P_\cS^\theta}(s'|s) - \overline{P_\cS^\theta}(s'|s)\right| \ge\epsilon\right\}\\
    &\subseteq \mathop{\cup}_{s}\mathop{\cup}_{s'}\left\{\left|\widehat{P_\cS^\theta}(s'|s) - \overline{P_\cS^\theta}(s'|s)\right| \ge\frac{\epsilon}{|\cS|}\right\}
\end{align*}
Then according to Lemma \ref{lemma:bound-P-S-hat},
\begin{align*}
    \Pr\left(\|\left\|\widehat{P_\cS^\theta} - \overline{P_\cS^\theta}\right\|_\infty \ge\epsilon\right) &\le \sum_{(s',s)}\Pr\left(\left|\widehat{P_\cS^\theta}(s'|s) - \overline{P_\cS^\theta}(s'|s)\right| \ge\frac{\epsilon}{|\cS|}\right)\\
    &\le 4\tau|\cS|^2\exp\left(-\frac{\epsilon^2\sigma_S^2\lfloor\frac{T}{\tau}\rfloor}{32|\cS|^2}\right).
\end{align*}
which completes the proof of the corollary.
\end{proof}

\begin{proof}(Proof of Theorem \ref{thm:estimation-d-theta})
\begin{align*}
    \|\widehat{d_\theta}-d_\theta\|_1 &= (1-\gamma)\|\left(\left(I-\gamma\widehat{P_\cS^\theta}^\top\right)^{-1}-\left(I-\gamma\overline{P_\cS^\theta}^\top\right)^{-1}\right)\rho\|_1\\
    &\le (1-\gamma)\left\|\left(I-\gamma\widehat{P_\cS^\theta}^\top\right)^{-1}-\left(I-\gamma\overline{P_\cS^\theta}^\top\right)^{-1}\right\|_1\|\rho\|_1 \\
    &\le (1-\gamma)\left\|\left(I-\gamma\widehat{P_\cS^\theta}\right)^{-1}-\left(I-\gamma\overline{P_\cS^\theta}\right)^{-1}\right\|_\infty\|\rho\|_1 \\
    &=\gamma(1-\gamma)\left\|\left(I-\gamma\overline{P_\cS^\theta}\right)^{-1}\left(\widehat{P_\cS^\theta}-\overline{P_\cS^\theta}\right)\left(I-\gamma\widehat{P_\cS^\theta}\right)^{-1}\right\|_\infty\\
    &\le\frac{\gamma}{1-\gamma}\left\|\widehat{P_\cS^\theta}-\overline{P_\cS^\theta}\right\|_\infty
\end{align*}
Thus
\begin{align*}
    \Pr\left(\|\widehat{d_\theta}-d_\theta\|_1\ge\epsilon\right)\le\Pr\left(\left\|\widehat{P_\cS^\theta}-\overline{P_\cS^\theta}\right\|_\infty\ge\frac{1-\gamma}{\gamma}\epsilon\right)\\
    \le 4\tau |\cS|^2 \exp\left(-\frac{(1-\gamma)^2\epsilon^2\sigma_S^2\left\lfloor\frac{T}{\tau}\right\rfloor}{32\gamma^2|\cS|^2}\right)
\end{align*}
\end{proof}
\subsection{Proof of Theorem \ref{thm:gradient-estimation}}
\begin{proof}
Since the stochastic game satisfies $(\tau,\sigma_S)$-sufficient exploration on states, then for any $\theta\in\cX^\alpha$, we know that it satisfies $(\tau, \frac{\alpha\sigma_S}{\max_i|\cA_i|})$-sufficient exploration. Substitute this into Theorem \ref{thm:policy-evaluation}, we have that for 
\begin{equation}
    T_J \ge \frac{32\tau(1+\alpha)^2|\cS|^3\sum_i|\cA_i|\max_{i}|\cA_i|^2}{(1-\gamma)^6\epsilon_g^2\alpha^2\sigma_S^2}\log\left(\frac{16\tau T_G|\cS|^2\sum_i|\cA_i|}{\delta}\right) + 1,
\end{equation}
with probability at least $1-\frac{\delta}{2T_G}$,
\begin{equation*}
    \|\overline{Q^{\theta^{(k)}}} - \widehat{Q^{\theta^{(k)}}}\|_\infty\le\frac{(1-\gamma)\epsilon_g}{(1+\alpha)\sqrt{|\cS|\sum_i|\cA_i|}}.
\end{equation*}
Similarly, applying Theorem \ref{thm:estimation-d-theta}, we have that with probability at least $1-\frac{\delta}{2T_G}$,
\begin{equation*}
    \|d_{\theta^{(k)}} - \widehat{d_{\theta^{(k)}}}\|_1\le\frac{(1-\gamma)^2\epsilon_g\alpha}{(1+\alpha)\sqrt{|\cS|\sum_i|\cA_i|}}.
\end{equation*}
Since:
\begin{align*}
   \left| \left[\nabla\Phi(\theta)-\widehat{\nabla}\Phi(\theta) \right]_{(s,a_i)}\right| &= \left|\frac{1}{1-\gamma}d_\theta(s)\overline{Q_i^\theta}(s,a_i) -\frac{1}{1-\gamma}\widehat{d_\theta}(s)\widehat{Q_i^\theta}(s,a_i)\right|\\
    &=\left|\frac{1}{1-\gamma}d_\theta(s)(\overline{Q_i^\theta}(s,a_i)-\widehat{Q_i^\theta}(s,a_i))\right| + \left|\frac{1}{1-\gamma}\widehat{Q_i^\theta}(s,a_i)(d_\theta(s)-\widehat{d_\theta}(s))\right|\\
    &\le\frac{1}{1-\gamma}|\overline{Q_i^\theta}(s,a_i)-\widehat{Q_i^\theta}(s,a_i)| + \frac{1}{(1-\gamma)^2}|d_\theta(s)-\widehat{d_\theta}(s)|\\
    &\le \frac{\epsilon_g}{(1+\alpha)\sqrt{|\cS|\sum_j|\cA_j|}} + \frac{\epsilon_g\alpha}{(1+\alpha)\sqrt{|\cS|\sum_j|\cA_j|}} \\
    & = \frac{\epsilon_g}{\sqrt{|\cS|\sum_j|\cA_j|}}
\end{align*}
Thus, with probability $1-\delta$
\begin{align*}
    \|\nabla\Phi(\theta^{(k)}-\widehat{\nabla}\Phi(\theta^{(k)})\|_2^2 &= \sum_i \sum_s\sum_{a_i} \left| \left[\nabla\Phi(\theta)-\widehat{\nabla}\Phi(\theta) \right]_{(s,a_i)}\right|^2 \le \epsilon_g^2, ~~\forall~1\le k\le T_G
\end{align*}
\end{proof}

\section{Proof of Theorem \ref{thm:main}}\label{apdx:proof-main-sample-based}
\paragraph{Notations:} We define the following variables that will be useful in the analysis:
$$\widehat{G}^\eta(\theta):=\frac{1}{\eta}\left( Proj_{\cX}(\theta+\eta\widehat{\nabla}\Phi(\theta))-\theta\right)$$
$$\widehat{G}^{\eta,\alpha}(\theta):=\frac{1}{\eta}\left( Proj_{\cX^{\alpha}}(\theta+\eta\widehat{\nabla}\Phi(\theta))-\theta\right).$$
\subsection{Optimization Lemmas}
\begin{lemma}{(Sufficient ascent)}\label{lemma:sufficient-ascent}
Suppose $\Phi(\theta)$ is $\beta$-smooth. Let $\theta^+ = Proj_{\cX^{\alpha}}(\theta+\eta\widehat{\nabla}\Phi(\theta))$. Then for $\eta\le\frac{1}{2\beta}$,
\begin{equation*}
    \begin{split}
        \Phi(\theta^+) - \Phi(\theta)\ge \frac{\eta}{4}\|\widehat{G}^{\eta,\alpha}(\theta)\|^2 - \frac{\eta}{2}\left\|\nabla\Phi(\theta) - \widehat{\nabla}\Phi(\theta)\right\|^2
    \end{split}
\end{equation*}
\end{lemma}
\begin{proof}
From the smoothness property we have that:
\begin{equation*}
    \Phi(\theta^+)-\Phi(\theta) \ge \nabla_\theta\Phi(\theta)^\top(\theta^+-\theta) -\frac{\beta}{2}\|\theta^+-\theta\|^2
\end{equation*}
Since $\theta^+=Proj_{\cX^\alpha}(\theta + \eta\widehat{\nabla}\Phi(\theta))$, we have that:
\begin{equation*}
    (\theta + \eta\widehat{\nabla}\Phi(\theta) - \theta^+)^\top(\theta'-\theta^+) \le 0, ~~\forall~\theta'\in\cX^\alpha
\end{equation*}
take $\theta'=\theta$, we get:
\begin{equation*}
    \widehat{\nabla}\Phi(\theta)^\top(\theta^+-\theta) \ge \frac{1}{\eta}\|\theta^+-\theta\|^2. 
\end{equation*}
Thus:
\begin{align*}
    &\nabla\Phi(\theta)^\top(\theta^+-\theta) = \left(\nabla\Phi(\theta) - \widehat{\nabla}\Phi(\theta)\right)^\top(\theta^+-\theta) + \widehat{\nabla}\Phi(\theta)^\top(\theta^+-\theta) \\
    &\ge -\frac{\eta}{2}\left\|\nabla\Phi(\theta) - \widehat{\nabla}\Phi(\theta)\right\|^2 - \frac{1}{2\eta}\left\|\theta^+-\theta\right\|^2 + \widehat{\nabla}\Phi(\theta)^\top(\theta^+-\theta)\\
    &\ge -\frac{\eta}{2}\left\|\nabla\Phi(\theta) - \widehat{\nabla}\Phi(\theta)\right\|^2 - \frac{1}{2\eta}\left\|\theta^+-\theta\right\|^2 + \frac{1}{\eta}\|\theta^+-\theta\|^2\\
    &= \frac{1}{2\eta}\|\theta^+-\theta\|^2-\frac{\eta}{2}\left\|\nabla\Phi(\theta) - \widehat{\nabla}\Phi(\theta)\right\|^2
\end{align*}
Thus from \eqref{eq:smoothness-property}:
\begin{align*}
    \Phi(\theta^+)-\Phi(\theta) &\ge \left(\frac{1}{2\eta}-\frac{\beta}{2}\right)\|\theta^+-\theta\|^2-\frac{\eta}{2}\left\|\nabla\Phi(\theta) - \widehat{\nabla}\Phi(\theta)\right\|^2\\
    &\ge \frac{1}{4\eta}\|\theta^+-\theta\|^2-\frac{\eta}{2}\left\|\nabla\Phi(\theta) - \widehat{\nabla}\Phi(\theta)\right\|^2\\
    &=\frac{\eta}{4}\|\widehat{G}^{\eta,\alpha}(\theta)\|^2 - \frac{\eta}{2}\left\|\nabla\Phi(\theta) - \widehat{\nabla}\Phi(\theta)\right\|^2
\end{align*}
which completes the proof.
\end{proof}
Lemma \ref{lemma:sufficient-ascent} immediately results in the following corollary:
\begin{coro}{(of Lemma \ref{lemma:sufficient-ascent})}\label{coro:sufficient-ascent}
In Algorithm \ref{alg:sample-based learning}, suppose $\|\widehat{\nabla}\Phi(\theta^{(k)})-\nabla\Phi(\theta^{(k)})\|_\infty\le\epsilon_g$ holds for every $0\le k\le T_G-1$, then running algorithm \ref{alg:sample-based learning} will guarantee that:
\begin{equation*}
    \frac{1}{T_G}\sum_{k=0}^{T_G-1}\|\widehat{G}^{\eta,\alpha}(\theta^{(k)})\|^2\le \frac{4(\Phi_{\max}-\Phi_{\min})}{\eta T_G} + 2\epsilon_g^2
\end{equation*}
\end{coro}
\begin{proof}
From Lemma \ref{lemma:sufficient-ascent} we have that:
\begin{align*}
    \Phi(\theta^{(k+1)}) - \Phi(\theta^{(k)}) &\ge \frac{\eta}{4}\|\widehat{G}^{\eta,\alpha}(\theta^{(k)})\|^2 - \frac{\eta}{2}\left\|\nabla\Phi(\theta^{(k)}) - \widehat{\nabla}\Phi(\theta^{(k)})\right\|^2\\
    &\ge \frac{\eta}{4}\|\widehat{G}^{\eta,\alpha}(\theta^{(k)})\|^2 - \frac{\eta}{2}\epsilon_g^2.
\end{align*}
Thus 
\begin{align*}
    \frac{1}{T_G}\sum_{k=0}^{T_G-1}\|\widehat{G}^{\eta,\alpha}(\theta^{(k)})\|^2
    &\le \frac{4(\Phi(\theta^{(0)})-\Phi(\theta^{(T_G)}))}{\eta T_G} + 2\epsilon_g^2\\
    &\le \frac{4(\Phi_{\max}-\Phi_{\min})}{\eta T_G} + 2\epsilon_g^2
\end{align*}
\end{proof}
\begin{lemma}(First-order stationarity and $\|\widehat{G}^{\eta,\alpha}(\theta)\|$)\label{lemma:first-order-stationary-and-gradient-mapping}
Suppose $\Phi(\theta)$ is $\beta$-smooth. Let $\theta^+ = Proj_{\cX^{\alpha}}(\theta+\eta\widehat{\nabla}\Phi(\theta))$. Then:
\begin{equation}\label{eq:first-order-with-alpha}
    \nabla_{\theta}\Phi(\theta^+)^\top(\theta'-\theta^+)\le\left[(1+\eta\beta)\|\widehat{G}^{\eta,\alpha}(\theta)\| + \|\widehat{\nabla}\Phi(\theta) - \nabla\Phi(\theta)\|\right]\|\theta'-\theta^+\|, \quad \forall \theta'\in\cX^\alpha.
\end{equation}
Further:
\begin{equation}\label{eq:first-order-without-alpha}
    \max_{\overline{\theta}_i\in\cX_i}\nabla_{\theta_i}\Phi(\theta^+)^\top(\overline{\theta}_i-\theta_i^+)\le2\sqrt{|\cS|}\left[(1+\eta\beta)\|\widehat{G}^{\eta,\alpha}(\theta)\| + \|\widehat{\nabla}\Phi(\theta) - \nabla\Phi(\theta)\|\right] + \frac{2\alpha}{1-\gamma}
\end{equation}
\end{lemma}
\begin{proof}
Since $\theta^+ = Proj_{\cX^{\alpha}}(\theta+\eta\widehat{\nabla}\Phi(\theta))$, we have:
\begin{align*}
    (\theta+\eta\widehat{\nabla}\Phi(\theta)-\theta^+)^\top(\theta'-\theta^+)&\le0~~\forall~\theta'\in\cX^\alpha\\
    \Longrightarrow~ \eta\widehat{\nabla}\Phi(\theta)^\top(\theta'-\theta^+)&\le(\theta-\theta^+)^\top(\theta'-\theta^+)\\
    \Longrightarrow~ \eta\nabla\Phi(\theta)^\top(\theta'-\theta^+)&\le(\theta-\theta^+)^\top(\theta'-\theta^+) + \eta(\nabla\Phi(\theta)-\widehat{\nabla}\Phi(\theta))^\top(\theta'-\theta^+)\\
    \Longrightarrow~ \eta\nabla\Phi(\theta^+)^\top(\theta'-\theta^+)&\le(\theta-\theta^+)^\top(\theta'-\theta^+)+\eta(\nabla\Phi(\theta)-\widehat{\nabla}\Phi(\theta))^\top(\theta'-\theta^+)  \\&\quad + \eta(\nabla\Phi(\theta^+)-\nabla\Phi(\theta))^\top(\theta'-\theta^+)\\
    \Longrightarrow~ \eta\nabla\Phi(\theta^+)^\top(\theta'-\theta^+) &\le(\|\theta-\theta^+\|+\eta\|\nabla\Phi(\theta)-\widehat{\nabla}\Phi(\theta)\| +\eta\|\nabla\Phi(\theta^+)-\nabla\Phi(\theta)\|)\|\theta'-\theta^+\|\\
    &\le (\|\theta-\theta^+\|+\eta\|\nabla\Phi(\theta)-\widehat{\nabla}\Phi(\theta)\| +\eta\beta\|\theta^+-\theta\|)\|\theta'-\theta^+\|\\
    &=\left[(1+\eta\beta)\|\theta-\theta^+\|+\eta\|\nabla\Phi(\theta)-\widehat{\nabla}\Phi(\theta)\|\right]\|\theta'-\theta^+\|\\
    \Longrightarrow~\nabla\Phi(\theta^+)^\top(\theta'-\theta^+)&\le \left[(1+\eta\beta)\|\widehat{G}^{\eta,\alpha}(\theta)\|+\|\nabla\Phi(\theta)-\widehat{\nabla}\Phi(\theta)\|\right]\|\theta'-\theta^+\|,
\end{align*}
which proves \eqref{eq:first-order-with-alpha}.
We now prove \eqref{eq:first-order-without-alpha}. For any $\theta_{i,s}'\in\Delta(|\cA_i|)$, we know that $(1-\alpha)\theta_{i,s}' +\alpha U_{|\cA_i|}\in \Dalpha$. Let $U_i:=[\underbrace{U_{|\cA_i|},\dots, U_{|\cA_i|}}_{|\cS| \text{times}}]$, then for any $\theta_i'\in\cX_i$, $(1-\alpha)\theta_i'+\alpha U_i \in \cX_i^\alpha$.

Thus:
\begin{align*}
    \nabla_{\theta_i}\Phi(\theta^+)^\top(\theta_i'-\theta_i^+)&\le \nabla_{\theta_i}\Phi(\theta^+)^\top((1-\alpha)\theta_i'+\alpha U_i-\theta_i^+) 
    +\nabla_{\theta_i}\Phi(\theta^+)^\top(\theta_i'-(1-\alpha)\theta_i' -\alpha U_i)\\
    &\le\left[(1+\eta\beta)\|\widehat{G}^{\eta,\alpha}(\theta)\| + \|\widehat{\nabla}\Phi(\theta) - \nabla\Phi(\theta)\|\right]\|(1-\alpha)\theta_i'+\alpha U_i-\theta_i^+\| \\
    &\qquad+ \nabla_{\theta_i}\Phi(\theta^+)^\top(\theta_i'-(1-\alpha)\theta_i' -\alpha U_i)\\
   & \le 2\sqrt{|\cS|}\left[(1+\eta\beta)\|\widehat{G}^{\eta,\alpha}(\theta)\| + \|\widehat{\nabla}\Phi(\theta) - \nabla\Phi(\theta)\|\right]+ \alpha\nabla_{\theta_i}\Phi(\theta^+)^\top(\theta_i'- U_i)
\end{align*}
Since
\begin{align*}
    \nabla_{\theta_i}\Phi(\theta^+)^\top(\theta_i'- U_i) &= \sum_sd_\theta(s)\overline{Q_{i,s}^\theta}^\top(\theta_{i,s}'- U_{|\cA_i|})\\
    &\le\sum_sd_\theta(s)\|\overline{Q_{i,s}^\theta}\|_\infty\|\theta_{i,s}'- U_{|\cA_i|}\|_1\\
    & \le\sum_sd_\theta(s)\frac{2}{1-\gamma} \le \frac{2}{1-\gamma},
\end{align*} 
we have that:
\begin{align*}
    \nabla_{\theta_i}\Phi(\theta^+)^\top(\theta_i'-\theta_i)&\le 2\sqrt{|\cS|}\left[(1+\eta\beta)\|\widehat{G}^{\eta,\alpha}(\theta)\| + \|\widehat{\nabla}\Phi(\theta) - \nabla\Phi(\theta)\|\right] + \frac{2\alpha}{1-\gamma}
\end{align*}
\end{proof}
\subsection{Proof of Theorem \ref{thm:main}}
\begin{proof}
Recall that $\Phi$ is $\beta$-smooth with $\beta = \frac{2}{(1-\gamma)^3}\left(\sum_{i=1}^n|\cA_i| \right)$. The step size $\eta$ in Theorem \ref{thm:main} satisfies $\eta \le \frac{(1-\gamma)^3}{4\sum_{i=1}^n|\cA_i|} = \frac{1}{2\beta}$. 

Recall from gradient domination property:
\begin{align*}
    \NEgap_i(\theta^{(k+1)}) &= \max_{\theta_i'\in\cX_i}J_i(\theta_i',\theta_{-i}^{(k+1)})-J_i(\theta_i^{(k+1)},\theta_{-i}^{(k+1)})\\
    &\le M\max_{\theta_i'\in\cX_i}(\theta_i'-\theta_i^{(k+1)})^\top\nabla_{\theta_i}\Phi(\theta^{(k+1)})
\end{align*}

Suppose $\|\widehat{\nabla}\Phi(\theta^{(k)})-\nabla\Phi(\theta^{(k)})\|_\infty\le\epsilon_g, ~\forall 0\le k\le T_G-1$, recall from Lemma \ref{lemma:first-order-stationary-and-gradient-mapping},
\begin{align*}
    \NEgap(\theta^{(k+1)})&\le \max_i \NEgap_i(\theta^{(k+1)})\le M\max_i\max_{\theta_i'\in\cX_i}(\theta_i'-\theta_i^{(k+1)})^\top\nabla_{\theta_i}\Phi(\theta^{(k+1)})\\
   & \le 2M\sqrt{|\cS|}\left[(1+\eta\beta)\|\widehat{G}^{\eta,\alpha}(\theta^{(k)})\| + \epsilon_g\right] + \frac{2\alpha M}{1-\gamma}
\end{align*}
Thus,
\begin{align*}
    \frac{1}{T_G}\sum_{k=0}^{T_G-1}\NEgap(\theta^{(k+1)})^2&\le \frac{1}{T_G}\sum_{k=0}^{T_G-1} 3\times\left[4M^2|\cS|(1+\eta\beta)^2\|\widehat{G}^{\eta,\alpha}(\theta^{(k)})\|^2 +4M^2|\cS|\epsilon_g^2+\frac{4\alpha^2M^2}{(1-\gamma)^2}\right]\\
    &=12M^2|\cS|\epsilon_g^2+\frac{12\alpha^2M^2}{(1-\gamma)^2}+ 12M^2|\cS|(1+\eta\beta)^2\left(\frac{1}{T_G}\sum_{k=0}^{T_G-1}\|\widehat{G}^{\eta,\alpha}(\theta^{(k)})\|^2\right)
\end{align*}
From Corollary \ref{coro:sufficient-ascent}, we have that
\begin{align}
    \frac{1}{T_G}\sum_{k=0}^{T_G-1}\NEgap(\theta^{(k+1)})^2&\le 12M^2|\cS|\epsilon_g^2+\frac{12\alpha^2M^2}{(1-\gamma)^2}+ 12M^2|\cS|(1+\eta\beta)^2\left( \frac{4(\Phi_{\max}-\Phi_{\min})}{\eta T_G} + 2\epsilon_g^2\right)\notag\\
    &\le 66M^2|\cS|\epsilon_g^2+\frac{12\alpha^2M^2}{(1-\gamma)^2}+\frac{108M^2|\cS|(\Phi_{\max}-\Phi_{\min})}{\eta T_G}\label{eq:6}
\end{align}
Substitute 
\begin{align*}
    \alpha = \frac{(1-\gamma)\epsilon}{6M}, ~~\epsilon_g = \frac{\epsilon}{2\sqrt{33}M\sqrt{|\cS|}} \textup{ and }
    T_G\ge\frac{648M^2(\Phi_{\max}-\Phi_{\min})|\cS|}{\eta\epsilon^2}
\end{align*}
into the above inequality we get that:
\begin{equation*}
     \frac{1}{T_G}\sum_{k=0}^{T_G-1}\NEgap(\theta^{(k+1)})^2\le\frac{\epsilon^2}{2}+\frac{\epsilon^2}{3}+\frac{\epsilon^2}{6}=\epsilon^2
\end{equation*}
Substitute the value of $\alpha, \epsilon_g$ in \eqref{eq:6} into Theorem \ref{thm:gradient-estimation} will give us:
\begin{equation*}
    T_J\ge \frac{206976\tau nM^4|\cS|^3\max_i|\cA_i|^3}{(1-\gamma)^8\epsilon^4\sigma_S^2}\log\left(\frac{16\tau T_G|\cS|^2\sum_i|\cA_i|}{\delta}\right) + 1 
\end{equation*}
which completes the proof.
\end{proof}

\section{Smoothness}
\begin{lemma}{(Smoothness for Direct Distributed Parameterization)}\label{lemma:smoothness}
Assume that $0 \le r_i(s,a) \le 1, ~\forall s, a, ~ i=1,2,\dots,n$, then:
\begin{equation}
    \|g(\theta') - g(\theta)\| \le \frac{2}{(1-\gamma)^3}\left(\sum_{i=1}^n|\cA_i| \right)\|\theta' - \theta\|,
\end{equation}
where $g(\theta) = \{\nabla_{\theta_i}J_i(\theta)\}$.
\end{lemma}
The proof of Lemma \ref{lemma:smoothness} depends on the following lemma:
\begin{lemma}\label{lemma:smoothness-1}
\begin{equation}
    \|\nabla_{\theta_i}J_i(\theta') - \nabla_{\theta_i}J_i(\theta)\| \le \frac{2}{(1-\gamma)^3}\sqrt{|\cA_i|} \sum_{j=1}^n \sqrt{|\cA_j|}\|\theta_j' - \theta_j\|
\end{equation}
\end{lemma}
Lemma \ref{lemma:smoothness} is a simple corollary of Lemma \ref{lemma:smoothness-1}.
\begin{proof}{(Proof of Lemma \ref{lemma:smoothness})}
\begin{align*}
    \|g(\theta') - g(\theta)\|^2 &= \sum_{i=1}^n \|\nabla_{\theta_i}J_i(\theta') - \nabla_{\theta_i}J_i(\theta)\|^2\\
    &\le \left(\frac{2}{(1-\gamma)^3}\right)^2\sum_i|\cA_i|\left(\sum_{j=1}^n\sqrt{|\cA_j|}\|\theta_j' - \theta_j\|\right)^2\\
    &\le \left(\frac{2}{(1-\gamma)^3}\right)^2\sum_i|\cA_i|\left(\sum_{j=1}^n|\cA_j|\right)\left(\sum_{j=1}^n\|\theta_j' - \theta_j\|^2\right)\\
    &= \left(\frac{2}{(1-\gamma)^3}\right)^2\left(\sum_{i=1}^n|\cA_i|\right)^2 \|\theta' - \theta\|^2,
\end{align*}
which completes the proof.
\end{proof}
Lemma \ref{lemma:smoothness-1} is equivalent to the following lemma:
\begin{lemma}\label{lemma:smoothness-2}
\begin{equation}
   \left| \frac{\partial J_i(\theta_i' + \alpha u_i, \theta_{-i}') - \partial J_i(\theta_i + \alpha u_i, \theta_{-i})}{\partial \alpha} \Big |_{\alpha=0} \right| \le \frac{2}{(1-\gamma)^3}\sqrt{|\cA_i|} \sum_{j=1}^n \sqrt{|\cA_j|}\|\theta_j' - \theta_j\|, \quad \forall \|u\| = 1
\end{equation}
\end{lemma}
\begin{proof}{(Lemma \ref{lemma:smoothness-2})}
Define:
\begin{align*}
    \pi_{i,\alpha}(a_i|s)&:= \pi_{\theta_i+\alpha u_i}(a_i|s) = \theta_{s,a_i} + \alpha u_{a_i,s}\\
    \pi_{i,\alpha}'(a_i|s)&:= \pi_{\theta_i+\alpha u_i}'(a_i|s) = \theta_{s,a_i}' + \alpha u_{a_i,s}\\
    \pi_{\alpha}(a|s) &:= \pi_{\theta_i+\alpha u_i}(a_i|s)\pi_{\theta_{-i}}(a_{-i}|s)\\
    \pi_{\alpha}'(a|s) &:= \pi_{\theta_i'+\alpha u_i}(a_i|s)\pi_{\theta_{-i}}'(a_{-i}|s)\\
    Q_i^\alpha(s,a) &:= Q_{(\theta_i+\alpha u_i, \theta_{-i})}(s,a)\\
    d_\alpha'(s) &:= d_{(\theta_i'+\alpha u_i, \theta_{-i})}(s)
\end{align*}
According to cost difference lemma,
\begin{align*}
    &\quad \left| \frac{\partial J_i(\theta_i' + \alpha u_i, \theta_{-i}') - \partial J_i(\theta_i + \alpha u_i, \theta_{-i})}{\partial \alpha} \Big |_{\alpha=0} \right| \\&= \frac{1}{1-\gamma}\left| \frac{\partial \sum_{s,a}d_{\alpha'}(s)\pi'_{\alpha}(a|s)A_i^\alpha(s,a)}{\partial \alpha} \Big |_{\alpha=0} \right|\\
    &= \frac{1}{1-\gamma}\left| \frac{\partial \sum_{s,a}d_{\alpha'}(s)\left(\pi'_{\alpha}(a|s) - \pi_{\alpha}(a|s)\right)Q_i^\alpha(s,a)}{\partial \alpha} \Big |_{\alpha=0} \right|\\
    &\le\frac{1}{1-\gamma} \left(\underbrace{\left|\sum_{s,a} d_{\theta}'(s) \frac{\partial \pi'_{\alpha}(a|s) - \partial \pi_{\alpha}(a|s)}{\partial \alpha}\Big |_{\alpha=0}Q_i^\theta(s,a)\right|}_{\text{Part A}}\right.\\
    & + \underbrace{\left|\sum_{s,a} d_{\theta}'(s)(\pi'_{\theta}(a|s) - \pi_{\theta}(a|s)) \frac{\partial Q_i^\alpha(s,a)}{\partial \alpha}\Big |_{\alpha=0}\right|}_{\text{Part B}}\\
    & + \left.\underbrace{\left|\sum_{s,a} \frac{\partial d_{\alpha}'(s)}{\partial \alpha}\Big |_{\alpha=0}(\pi'_{\theta}(a|s) - \pi_{\theta}(a|s)) Q_i^\theta(s,a)\right|}_{\text{Part C}}\right)
\end{align*}
Thus:
\begin{align}
    \textup{Part A} &= \left|\sum_{s,a} d_{\theta}'(s) \frac{\partial \pi'_{\alpha}(a|s) - \partial  \pi_{\alpha}(a|s)}{\partial \alpha}\Big |_{\alpha=0}Q_i^\theta(s,a)\right|\notag \\
    &= \left|\sum_{s,a} d_{\theta}'(s) u_{a_i,s} (\pi_{\theta_{-i}'}(a_{-i}|s) - \pi_{\theta_{-i}}(a_{-i}|s))Q_i^\theta(s,a)\right|\label{eq:1}\\
    &\le\frac{1}{1-\gamma} \left|\sum_{s} d_{\theta}'(s) \sum_{a_i}|u_{a_i,s}|\sum_{a_{-i}} \left|\pi_{\theta_{-i}'}(a_{-i}|s) - \pi_{\theta_{-i}}(a_{-i}|s)\right|\right|\label{eq:2}\\
     &\le \frac{1}{1-\gamma} \left(\max_s \sum_{a_i}|u_{a_i,s}|\right) \sum_{s} d_{\theta}'(s) 2d_{\text{TV}}(\pi_{\theta_{-i}'}(\cdot|s)||\pi_{\theta_{-i}}(\cdot|s))\label{eq:3}\\
     &\le \frac{1}{1-\gamma} \left(\max_s \sum_{a_i}|u_{a_i,s}|\right) \sum_{s} d_{\theta}'(s) \sum_{j\neq i}2d_{\text{TV}}(\pi_{\theta_j'}(\cdot|s)||\pi_{\theta_j}(\cdot|s))\label{eq:4}\\
     &= \frac{1}{1-\gamma} \left(\max_s \sum_{a_i}|u_{a_i,s}|\right) \sum_{s} d_{\theta}'(s) \sum_{j\neq i}\|\theta'_{j,s} - \theta_{j,s}\|_1\label{eq:4-1}\\
     &\le \frac{1}{1-\gamma}\sqrt{|\cA_i|} \sum_{s} d_{\theta}'(s) \sum_{j\neq i}\sqrt{|\cA_j|}\|\theta_{j,s}' - \theta_{j,s}\|\label{eq:5}\\
     &\le \frac{1}{1-\gamma}\sqrt{|\cA_i|} \sum_{j\neq i}\sqrt{|\cA_j|}\sqrt{\sum_s d_{\theta'}(s)^2}\sqrt{\sum_s\|\theta_{j,s}' - \theta_{j,s}\|^2}\\
     &= \frac{1}{1-\gamma}\sqrt{|\cA_i|} \sum_{j\neq i}\sqrt{|\cA_j|}\sqrt{\sum_s d_{\theta'}(s)^2}\|\theta_j'-\theta_j\|\notag\\
     &\le \frac{1}{1-\gamma}\sqrt{|\cA_i|} \sum_{j\neq i}\sqrt{|\cA_j|}\|\theta_j'-\theta_j\|\notag\\
     &\le \frac{1}{1-\gamma}\sqrt{|\cA_i|} \sum_{j=1}^n\sqrt{|\cA_j|}\|\theta_j'-\theta_j\|,\notag
\end{align}
where \eqref{eq:1} to \eqref{eq:2} is derived from the fact that $|Q_i^\theta(s,a)| \le \frac{1}{1-\gamma}$. \eqref{eq:3} to \eqref{eq:4} relies on the property of total variation distance:
\begin{equation*}
    d_{\text{TV}}(\pi_{\theta_{-i}'}(\cdot|s)||\pi_{\theta_{-i}}(\cdot|s)) \le \sum_{j\neq i}d_{\text{TV}}(\pi_{\theta_j'}(\cdot|s)||\pi_{\theta_j}(\cdot|s))
\end{equation*}
\eqref{eq:4-1} to \eqref{eq:5} is derived from:
\begin{align*}
   & \max_s \sum_{a_i}|u_{a_i,s}| \le \sqrt{|\cA_i|}, ~~ \|u\| \le 1\\
&   \|\theta'_{j,s}-\theta_{j,s}\|_1 \le \sqrt{|\cA_j|}\|\theta_{j,s}' - \theta_{j,s} \|
\end{align*}
which can be immediately verified by applying Cauchy-Schwarz inequality.

Before looking into Part B, we first define $\widetilde P(\alpha)$ as the state-action under $\pi_\alpha$:
\begin{equation*}
    \left[\widetilde P(\alpha)\right]_{(s,a)\rightarrow(s',a')}=\pi_\alpha(a'|s')P(s'|s,a)
\end{equation*}
Then we have that:
\begin{equation*}
     \left[\frac{\partial \widetilde P(\alpha)}{\partial \alpha}\Big |_{\alpha=0}\right]_{(s,a)\rightarrow(s',a')} = u_{a_i',s'}\pi_{\theta_{-i}}(a_{-i}'|s')P(s'|s,a)
\end{equation*}
For an arbitrary vector $x$:
\begin{align*}
     \left[\frac{\partial \widetilde P(\alpha)}{\partial \alpha}\Big |_{\alpha=0} x\right]_{(s,a)} &= \sum_{s',a' }u_{a_i',s'}\pi_{\theta_{-i}}(a_{-i}'|s')P(s'|s,a)x_{s',a'}\\
     &\le \|x\|_{\infty}\sum_{s',a'}|u_{a_i',s'}|\pi_{\theta_{-i}}(a_{-i}'|s')P(s'|s,a)\\
     &= \|x\|_{\infty}\sum_{s'}P(s'|s,a)\sum_{a_i'}|u_{a_i',s'}|\sum_{a_{-i}'}\pi_{\theta_{-i}}(a_{-i}'|s')\\
     &\le \|x\|_{\infty}\sum_{s'}P(s'|s,a)\sqrt{|\cA_i|}\sum_{a_{-i}'}\pi_{\theta_{-i}}(a_{-i}'|s') \\
     &\le \sqrt{|\cA_i|} \|x\|_{\infty}
\end{align*}
Thus:
\begin{equation*}
    \left\|\left[\frac{\partial \widetilde P(\alpha)}{\partial \alpha}\Big |_{\alpha=0} x\right]_{(s,a)}\right\|_{\infty} \le \sqrt{|\cA_i|} \|x\|_{\infty}
\end{equation*}
Similarly we can define $\widetilde P(\alpha)'$ as the state-action under $\pi_\alpha'$, and can easily check that 
\begin{equation*}
    \left\|\left[\frac{\partial \widetilde P(\alpha)'}{\partial \alpha}\Big |_{\alpha=0} x\right]_{(s,a)}\right\|_{\infty} \le \sqrt{|\cA_i|} \|x\|_{\infty}
\end{equation*}
Define:
$$M(\alpha):= \left(I-\gamma \widetilde P(\alpha)\right)^{-1}, \quad M(\alpha)':= \left(I-\gamma \widetilde P(\alpha)'\right)^{-1}.$$
Because:
\begin{equation*}
    M(\alpha)= \left(I-\gamma \widetilde P(\alpha)\right)^{-1} = \sum_{n=0}^\infty \gamma^n\widetilde P(\alpha),
\end{equation*}
which implies that every entry of $M(\alpha)$ is nonnegative and $M(\alpha) \mathbf{1} = \frac{1}{1-\gamma}\mathbf{1}$, this implies:
\begin{equation*}
    \|M(\alpha) x\|_{\infty} \le \frac{1}{1-\gamma}\|x\|_{\infty},
\end{equation*}
and similarly
\begin{equation*}
    \|M(\alpha)' x\|_{\infty} \le \frac{1}{1-\gamma}\|x\|_{\infty}.
\end{equation*}
Now we are ready to bound Part B. Because:
\begin{align*}
    Q_i^\alpha(s,a) &= e_{(s,a)}^\top M(\alpha) r_i\\
    \Longrightarrow~~ \frac{\partial Q_i^\alpha(s,a)}{\partial\alpha} &= e_{(s,a)}^\top \frac{\partial M(\alpha)}{\partial \alpha} r_i = \gamma e_{(s,a)}^\top M(\alpha)\frac{\partial \widetilde P(\alpha)}{\partial \alpha} M(\alpha)r_i\\
     \Longrightarrow~~ \left|\frac{\partial Q_i^\alpha(s,a)}{\partial\alpha}\right| &\le \gamma \left\| M(\alpha)\frac{\partial \widetilde P(\alpha)}{\partial \alpha} M(\alpha)r_i \right\|_{\infty}\\
     &\le \frac{\gamma}{(1-\gamma)^2}\sqrt{|\cA_i|}
\end{align*}
Thus,
\begin{align*}
    \textup{Part B} &=\left|\sum_{s,a} d_{\theta}'(s)(\pi'_{\theta}(a|s) - \pi_{\theta}(a|s)) \frac{\partial Q_i^\alpha(s,a)}{\partial \alpha}\Big |_{\alpha=0}\right|\\
    &\le \sum_{s,a}d_{\theta}'(s)\left|\pi'_{\theta}(a|s) - \pi_{\theta}(a|s)\right|\left|\frac{\partial Q_i^\alpha(s,a)}{\partial \alpha}\Big |_{\alpha=0}\right|\\
    &\le \frac{\gamma}{(1-\gamma)^2}\sqrt{|\cA_i|}\sum_{s}d_{\theta}'(s)2d_{\textup{TV}}(\pi_{\theta'}(\cdot|s)||\pi_{\theta}(\cdot|s))\\
    &\le \frac{\gamma}{(1-\gamma)^2}\sqrt{|\cA_i|}\sum_{s}d_{\theta}'(s)\sum_{j}2d_{\textup{TV}}(\pi_{\theta'_j}(\cdot|s)||\pi_{\theta_j}(\cdot|s))\\
    &=\frac{\gamma}{(1-\gamma)^2}\sqrt{|\cA_i|}\sum_{s}d_{\theta}'(s)\sum_{j}\|\theta_{j,s}' - \theta_{j,s}\|_1\\
     &\le \frac{\gamma}{(1-\gamma)^2}\sqrt{|\cA_i|}\sum_{s}d_{\theta}'(s)\sum_{j=1}^n\sqrt{|\cA_j|}\|\theta_{j,s}' - \theta_{j,s}\|\\
     &\le \frac{\gamma}{(1-\gamma)^2}\sqrt{|\cA_i|}\sum_{j=1}^n\sqrt{|\cA_j|}\sqrt{\sum_sd_{\theta}'(s)^2}\sqrt{\sum_s \|\theta_{j,s}' - \theta_{j,s}\|^2}\\
     &\le \frac{\gamma}{(1-\gamma)^2}\sqrt{|\cA_i|}\sum_{j=1}^n\sqrt{|\cA_j|} \|\theta_{j}' - \theta_{j}\|
\end{align*}

Now let's look at Part C:
\begin{align*}
    d_\alpha'(s)&= (1-\gamma) \sum_{s'}\rho(s')\sum_{a'}\pi_\alpha'(a'|s')e_{(s',a')}^\top M(\alpha)'\sum_{a''}e_{(s,a'')}\\
    \Longrightarrow ~~ \frac{\partial  d_\alpha'(s)}{\partial \alpha} &= (1-\gamma)\left( \underbrace{\sum_{s'}\rho(s')\sum_{a'}\frac{\partial \pi_\alpha'(a'|s')}{\partial \alpha}e_{(s',a')}^\top}_{v_1^\top} M(\alpha)'\sum_{a''}e_{(s,a'')}\right.\\
    & \left.+ \underbrace{\sum_{s'}\rho(s')\sum_{a'}\pi_\alpha'(a'|s')e_{(s',a')}^\top}_{v_2^\top} \frac{\partial M(\alpha)'}{\partial \alpha}\sum_{a''}e_{(s,a'')}\right)\\
    &= (1-\gamma) \left(v_1^\top M(\alpha)' + \gamma v_2^\top M(\alpha)'\frac{\partial \widetilde P(\alpha)}{\partial \alpha}M(\alpha)'\right)\sum_{a''}e_{(s,a'')}
\end{align*}
Note that $v_1, v_2$ are constant vectors that are independent of the choice of $s$. Additionally:
\begin{align*}
    \|v_1\|_1 &= \left\|\sum_{s}\rho(s)\sum_{a}\frac{\partial \pi_\alpha'(a|s)}{\partial \alpha}e_{(s,a)}\right\|_1\\
    &= \sum_s \rho(s)\sum_{a}\left|\frac{\partial \pi_\alpha'(a|s)}{\partial \alpha}\right|\\
    &= \sum_s \rho(s)\sum_a\left|u_{a_i,s}\right| \pi_{\theta_{-i}'}(a_{-i}|s)\\
    &\leq \sum_s \rho(s)\sum_a\left|u_{a_i,s}\right| \le \sqrt{|\cA_i|}\\
    \|v_2\|_1 &= \|\sum_{s}\rho(s)\sum_{a}\pi_\alpha'(a|s)e_{(s,a)}\|_1\\
    &= \sum_{s}\rho(s)\sum_{a}\pi_\alpha'(a|s) =1
\end{align*}
Thus:
\begin{align*}
    \text{Part C} &= \left|\sum_{s,a} \frac{\partial d_{\alpha}'(s)}{\partial \alpha}\Big |_{\alpha=0}(\pi'_{\theta}(a|s) - \pi_{\theta}(a|s)) Q_i^\theta(s,a)\right|\\
    &=(1-\gamma) \left|\left(v_1^\top M(0)' + \gamma v_2^\top M(0)'\frac{\partial \widetilde P(\alpha)}{\partial \alpha}\Big|_{\alpha=0}M(0)'\right)\underbrace{\sum_{s,a}\sum_{a'}e_{(s,a')} (\pi'_{\theta}(a|s) - \pi_{\theta}(a|s)) Q_i^\theta(s,a)}_{v_3}\right|\\
    &\le (1-\gamma) \left(\frac{1}{1-\gamma}\|v_1\|_1\|v_3\|_\infty + \frac{\gamma}{(1-\gamma)^2}\sqrt{|\cA_i|}\|v_2\|_1\|v_3\|_\infty\right)\\
    &\le \frac{\sqrt{|\cA_i|}}{1-\gamma}\|v_3\|_\infty
\end{align*}
Additionally:
\begin{align*}
    \left|[v_3]_{(s_0,a_0)}\right| &= \left|\sum_a(\pi_{\theta'}(a|s_0) - \pi_{\theta}(a|s_0)) Q_i^\theta(s_0,a)\right|\\
    &\le \frac{1}{1-\gamma} \sum_a|\pi_{\theta'}(a|s_0) - \pi_{\theta}(a|s_0)|\\
    &= \frac{1}{1-\gamma} 2d_{\textup{TV}}(\pi_{\theta'}(\cdot|s_0)||\pi_{\theta}(\cdot|s_0))\\
    &\le\frac{1}{1-\gamma} \sum_{j=1}^n 2d_{\textup{TV}}(\pi_{\theta_j'}(\cdot|s_0)||\pi_{\theta_j}(\cdot|s_0))\\
    &=\frac{1}{1-\gamma} \sum_{j=1}^n\|\theta_{j,s}' - \theta_{j,s}\|_1\\
    &\le \frac{1}{1-\gamma} \sum_{j=1}^n \sqrt{|\cA_j|} \|\theta_{j,s}' - \theta_{j,s}\|\\
    &\le \frac{1}{1-\gamma} \sum_{j=1}^n \sqrt{|\cA_j|} \|\theta_{j}' - \theta_{j}\|
\end{align*}
Combining the above inequalities we get:
\begin{equation*}
    \text{Part C} \le \frac{\sqrt{|\cA_i|}}{1-\gamma}\|v_3\|_\infty \le  \frac{\sqrt{|\cA_i|}}{(1-\gamma)^2}\sum_{j=1}^n \sqrt{|\cA_j|} \|\theta_{j}' - \theta_{j}\|
\end{equation*}
Sum up Part A-C we get:
\begin{align*}
    \left| \frac{\partial J_i(\theta_i' + \alpha u_i, \theta_{-i}') - \partial J_i(\theta_i + \alpha u_i, \theta_{-i})}{\partial \alpha} \Big |_{\alpha=0} \right| &\le \frac{1}{1-\gamma} (\text{Part A} + \text{Part B} + \text{Part C})\\
    & \le \frac{2}{(1-\gamma)^3}\sqrt{|\cA_i|}\sum_{j=1}^n \sqrt{|\cA_j|} \|\theta_{j}' - \theta_{j}\|,
\end{align*}
which completes the proof.
\end{proof}
\section{Auxiliary}\label{apdx:auxiliary}
We recall \Cref{lemma:auxiliary}.
\auxiliary*
\begin{proof}
Let $y = \theta + g$, without loss of generality, assume that $i^* = 1$ and that:
$$y_1 > y_2\ge y_3\ge\cdots\ge y_n.$$
Using KKT condition, one can derive an efficient algorithm for solving $Proj_\cX(y)$ \cite{Wang13}, which consists of the following steps:
\begin{enumerate}
    \item Find $\rho:= \max\{1\le j\le n: y_j + \frac{1}{j}\left(1-\sum_{i=1}^j y_i\right) > 0\}$;
    \item Set  $~~\lambda:= \frac{1}{\rho} \left(1-\sum_{i=1}^\rho y_i\right)$;
    \item Set  $~~\theta'_i = \max\{y_i+\lambda, 0\}$.
\end{enumerate}
From the algorithm, we have that:
\begin{align*}
    \lambda &= \frac{1}{\rho} \left(1-\sum_{i=1}^\rho y_i\right) = \frac{1}{\rho} \left(1-\sum_{i=1}^\rho (\theta_i+g_i)\right)\\
    &=\frac{1}{\rho} \left(1-\sum_{i=1}^\rho \theta_i\right) - \frac{1}{\rho}\sum_{i=1}^\rho g_i\\
    &\ge -\frac{1}{\rho}\sum_{i=1}^\rho g_i.
\end{align*}
If $\rho \ge 2$,
\begin{align*}
    \theta'_1 &= \max\{y_1+\lambda, 0\} \ge y_1+\lambda \ge \theta_1 + g_1 - \frac{1}{\rho}\sum_{i=1}^\rho g_i\\
    &\ge \theta_1+(1-\frac{1}{\rho})g_1 - \frac{1}{\rho}\sum_{i=2}^\rho (g_1-\Delta) = \theta_1 + \frac{\rho-1}{\rho}\Delta \ge \theta_1 + \frac{\Delta}{2}.
\end{align*}
If $\rho = 1$,
$$\theta'_1 = y_1 + \lambda = y_1 + (1-y_1) = 1.$$
Thus:
$$\theta'_1 \ge \min\{1, \theta_1 + \frac{\Delta}{2}\},$$ which completes the proof.
\end{proof}

\section{Numerical Simulation Details}\label{apdx:numerics-detail}
\paragraph{Verification of the fully mixed NE in Game 2}
We now verify that joining network 1 with probability $\frac{1-3\epsilon}{3(1-2\epsilon)}$,i.e.:
\begin{equation*}
    \pi_{\theta_i}(a_i=1|s) = \frac{1-3\epsilon}{3(1-2\epsilon)},~~\forall s\in \cS, ~~ i=1,2,
\end{equation*}
is indeed a NE. First, observe that
\begin{align*}
    \Prtheta(s_{i,t+1}=1) &= \left(\frac{1-3\epsilon}{3(1-2\epsilon)}\right)P(s_{i,t+1}=1|a_{i,t}=1) + \left(1-\frac{1-3\epsilon}{3(1-2\epsilon)}\right)P(s_{i,t+1}=1|a_{i,t}=2)\\
    &=\left(\frac{1-3\epsilon}{3(1-2\epsilon)}\right)(1-\epsilon) + \left(1-\frac{1-3\epsilon}{3(1-2\epsilon)}\right)\epsilon = \frac{1}{3}.
\end{align*}
Thus,
\begin{align*}
    V(s) &= r(s) + \sum_{t=1}^\infty\bE_{s_t} \gamma^t r(s_t) = r(s) + \frac{2\gamma}{3(1-\gamma)},\\
    \overline{Q^{\theta}_i}(s,a_i) &= r(s) + \gamma \sum_{s',a_{-i}}P(s'|a_i, a_{-i})\pi_{\theta_{-i}}(a_{-i}|s)V(s')\\ &= r(s) + \gamma\sum_{s_i'\in\{1,2\}}(P(s_{i}'|a_i)\Prtheta(s_{-i}=1)r(s_i,s_{-i}=1) + P(s_{i}'|a_i)\Prtheta(s_{-i}=2)r(s_i,s_{-i}=2)) + \frac{2\gamma^2}{3(1-\gamma)}\\
    &= r(s) + \gamma P(s_i'=1|a_i)\left(\frac{1}{3}r(s_i'=1,s_{-i}=1) + \frac{2}{3}r(s_i'=1,s_{-i}=2)\right) \\&\quad +\gamma P(s_i'=2|a_i) \left(\frac{1}{3}(s_i'=2,s_{-i}=1) + \frac{2}{3}r(s_i'=2,s_{-i}=2)\right)+ \frac{2\gamma^2}{3(1-\gamma)}\\
    &= r(s) + \frac{2}{3}\gamma + \frac{2\gamma^2}{3(1-\gamma)} = r(s) + \frac{2\gamma}{3(1-\gamma)} = V(s),
\end{align*}
which implies that:
$$(\theta_i' -\theta_i)^\top \nabla_{\theta_i}J_i(\theta)= 0, \quad \forall \theta_i' \in \cX_i,\quad i = 1,2,$$
i.e. $\theta$ satisfies first-order stationarity. Since $d_{\theta}(s) > 0$ holds for any valid $\theta$, by Theorem \ref{thm:equivalence-stationary-NE}, $\theta$ is a NE.
\paragraph{Computation of strict NEs in Game 2}
The computation of strict NEs is done numerically, using the criterion in Lemma \ref{lemma:strict-NE-gap}. We enumerate over all $2^{8}$ possible deterministic policies and check whether the conditions in Lemma \ref{lemma:strict-NE-gap} hold. For $\epsilon = 0.1, \gamma = 0.95,$ and an initial distribution set as:
$$\rho(s_1=i,s_2=j) = 1/4, ~~i,j\in\{1,2\},$$
the numerical calculation shows there exist $13$ different strict NEs.

\end{document}